\documentclass[sn-basic,iicol]{sn-jnl}

\usepackage{bm}
\usepackage{natbib}
\usepackage{titletoc}
\usepackage{tikz}
\usepackage[framemethod=tikz]{mdframed}
\usepackage{thm-restate}
\usepackage{hyperref}
\usepackage{subfigure} 
\usepackage{ragged2e}
\usepackage{soul}
\usepackage{tcolorbox}

\usepackage{graphicx}%
\usepackage{multirow}%
\usepackage{amsmath,amssymb,amsfonts}%
\usepackage{amsthm}%
\usepackage{mathrsfs}%
\usepackage[title]{appendix}%
\usepackage{xcolor}%
\usepackage{textcomp}%
\usepackage{manyfoot}%
\usepackage{booktabs}%
\usepackage{algorithm}%
\usepackage{algorithmicx}%
\usepackage{algpseudocode}%
\usepackage{listings}%

\hypersetup{
	colorlinks=true,
	linkcolor=blue,
	urlcolor=orange,
	citecolor=orange,
}

\theoremstyle{thmstyleone}%
\newtheorem{proposition}{Proposition}
\newtheorem{lemma}{Lemma}
\newtheorem{corollary}{Corollary}

\theoremstyle{thmstyletwo}%
\newtheorem{remark}{Remark}%

\newtheorem*{sketch}{Proof Sketch}

\theoremstyle{thmstylethree}%
\newtheorem{definition}{Definition}%
\newtheorem{assumption}{Assumption}

\newcommand{\size}[1]{\left\lvert #1 \right\rvert}
\newcommand{\I}[1] {\mathbf{1} \left[ #1 \right]}
\newcommand{\pp}[1] {\mathbb{P} \left[ #1 \right]}
\newcommand{\E}[2] {\mathop{\mathbb{E}}\limits_{#1} \left[ #2 \right]}
\newcommand{\pr}[2] {\mathop{\mathbb{P}}\limits_{#1} \left[ #2 \right]}

\definecolor{Top2}{RGB}{102, 171, 221}
\definecolor{Top1}{RGB}{245, 137, 112}

\definecolor{gt}{RGB}{136, 197, 75}
\definecolor{am}{RGB}{117, 170, 255}

\tcbuselibrary{skins}
\newenvironment{qbox}
	{\begin{tcolorbox}[enhanced jigsaw, drop shadow=black!50!white,colback=white, width=0.95\linewidth, center, left=2pt,right=2pt,top=1pt,bottom=1pt,halign=center]}
	{\end{tcolorbox}}

\newenvironment{qbox-tight}
	{\begin{tcolorbox}[enhanced jigsaw, drop shadow=black!50!white,colback=white, width=\linewidth, center, left=2pt,right=2pt,top=1pt,bottom=1pt,halign=center]}
	{\end{tcolorbox}}

\usepackage{color}
\definecolor{dkgreen}{rgb}{0,0.6,0}
\definecolor{gray}{rgb}{0.5,0.5,0.5}
\definecolor{mauve}{rgb}{0.58,0,0.82}
\begin{document}

\title[]{Top-K Pairwise Ranking: Bridging the Gap Among Ranking-Based Measures for Multi-Label Classification}


\author[1,2]{\fnm{Zitai} \sur{Wang}}\email{wangzitai@iie.ac.cn}
\author*[3]{\fnm{Qianqian} \sur{Xu}}\email{xuqianqian@ict.ac.cn}
\author[4]{\fnm{Zhiyong} \sur{Yang}}\email{yangzhiyong21@ucas.ac.cn}
\author[3,4]{\fnm{Peisong} \sur{Wen}}\email{wenpeisong20z@ict.ac.cn}
\author[5]{\fnm{Yuan} \sur{He}}\email{heyuan.hy@alibaba-inc.com}
\author[6]{\fnm{Xiaochun} \sur{Cao}}\email{caoxiaochun@mail.sysu.edu.cn}
\author*[4,3,7]{\fnm{Qingming} \sur{Huang}}\email{qmhuang@ucas.ac.cn}

\affil[1]{\orgdiv{}\orgname{IIE, Chinese Academy of Sciences}, \city{Beijing}, \postcode{100093}, \state{Beijing}, \country{China}}
\affil[2]{\orgdiv{SCS}, \orgname{University of Chinese Academy of Sciences}, \city{Beijing}, \postcode{100049}, \state{Beijing}, \country{China}}
\affil[3]{\orgdiv{IIP}, \orgname{ICT, Chinese Academy of Sciences}, \city{Beijing}, \postcode{100190}, \state{Beijing}, \country{China}}
\affil[4]{\orgdiv{SCST}, \orgname{University of Chinese Academy of Sciences}, \city{Beijing}, \postcode{100049}, \state{Beijing}, \country{China}}
\affil[5]{\orgname{Alibaba Group}, \city{Hangzhou}, \postcode{311121}, \state{Zhejiang}, \country{China}}
\affil[6]{\orgdiv{SCST}, \orgname{Shenzhen Campus of Sun Yat-sen University}, \city{Shenzhen}, \postcode{518107}, \state{Guangdong}, \country{China}}
\affil[7]{\orgdiv{BDKM}, \orgname{University of Chinese Academy of Sciences}, \city{Beijing}, \postcode{101408}, \state{Beijing}, \country{China}}

\abstract{Multi-label ranking, which returns multiple top-ranked labels for each instance, has a wide range of applications for visual tasks. Due to its complicated setting, prior arts have proposed various measures to evaluate model performances. However, both theoretical analysis and empirical observations show that a model might perform inconsistently on different measures. To bridge this gap, this paper proposes a novel measure named Top-K Pairwise Ranking (TKPR), and a series of analyses show that TKPR is compatible with existing ranking-based measures. In light of this, we further establish an empirical surrogate risk minimization framework for TKPR. On one hand, the proposed framework enjoys convex surrogate losses with the theoretical support of Fisher consistency. On the other hand, we establish a sharp generalization bound for the proposed framework based on a novel technique named data-dependent contraction. Finally, empirical results on benchmark datasets validate the effectiveness of the proposed framework.}


\maketitle

\section{Introduction}
\label{sec:introduction}
In many real-world visual tasks, the instances are intrinsically multi-labeled \cite{DBLP:journals/pr/BoutellLSB04,DBLP:journals/pami/CarneiroCMV07,DBLP:conf/cvpr/WangYMHHX16,DBLP:conf/cvpr/ChenWWG19}. For example, a photo taken on a coast might consist of multiple objects such as the sea, beach, sky, and cloud. Hence, Multi-Label Classification (MLC) has attracted rising attention in recent years \cite{DBLP:journals/pami/LiuWST22,DBLP:journals/corr/abs-2210-03968,DBLP:conf/cvpr/DingWCZLBYH23,DBLP:conf/icml/WuLY23,DBLP:conf/cvpr/KimKJSA023,DBLP:conf/icml/Liu0LG23,DBLP:journals/ijcv/ChenZL23}.

\begin{figure*}[t]
        \centering
        \includegraphics[width=0.98\linewidth]{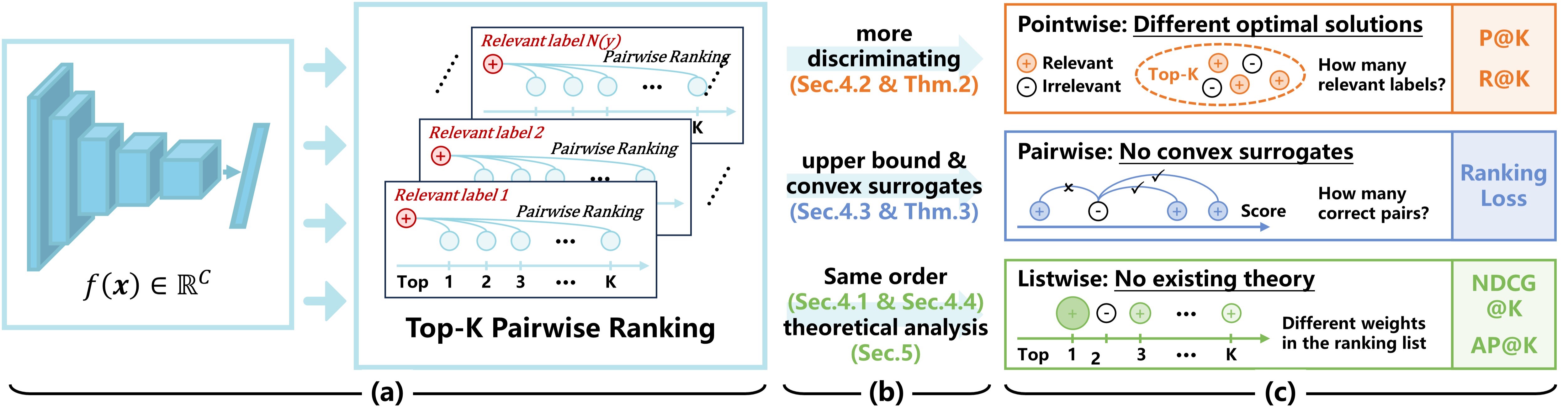}
        \caption{Overview of measure comparison: \textbf{(a)} the definition of the proposed TKPR measure, \textbf{(b)} the advantages of TKPR over existing ranking-based measures, and \textbf{(c)} representative ranking-based measures and their limitations.}
      \label{fig:overview}
\end{figure*}

Due to the complicated setting, a number of measures are proposed to evaluate the model performance \cite{DBLP:journals/tkde/ZhangZ14,DBLP:conf/icml/WuZ17,DBLP:journals/pami/LiuWST22}. Generally speaking, existing MLC measures can be divided into two groups: the threshold-based ones and the ranking-based ones. The former ones, such as the Hamming loss \cite{DBLP:journals/ml/DembczynskiWCH12,DBLP:conf/nips/WuZ20}, the subset accuracy \cite{DBLP:conf/nips/WuZ20,DBLP:conf/nips/GerychHBAR21}, and the F-measure \cite{DBLP:journals/tkde/TsoumakasKV11,DBLP:conf/icml/NanCLC12,DBLP:journals/jmlr/WaegemanDJCH14}, require predefined thresholds to decide whether each label is relevant. Although these measures are widely used due to its intuitive nature, the optimal thresholds might vary according to the decision conditions in the test phase, and thus predefined thresholds could produce systematic biases. By contrast, the ranking-based ones only check whether the relevant labels are top-ranked, which is insensitive to the selection of thresholds. In view of this, this paper focuses on the ranking-based measures and their optimization, which is also known as Multi-Label Ranking (MLR) \cite{DBLP:conf/icml/DembczynskiKH12,DBLP:conf/cvpr/LiSL17,DBLP:conf/nips/WuLXZ21}.

    As shown in Fig.\ref{fig:overview}\textbf{(c)}, existing ranking-based measures fall into the following three categories, according to the taxonomy in Learning-to-Rank \cite{DBLP:journals/ftir/Liu09}: (1) Pointwise approaches, such as precision@K and recall@K \cite{DBLP:conf/nips/WydmuchJKBD18,DBLP:conf/nips/X19}, reduce MLC to multiple single-label problems and check whether each relevant label is ranked higher than K. (2) Pairwise approaches, such as the ranking loss \cite{DBLP:conf/icml/DembczynskiKH12,DBLP:journals/ai/GaoZ13,DBLP:conf/nips/WuZ20,DBLP:conf/nips/WuLXZ21} and AUC \cite{DBLP:conf/icml/WuZ17}, transforms MLC into a series of pairwise problems that checks whether the relevant label is ranked higher than the irrelevant one. (3) Listwise approaches, such as Average Precision (AP) \cite{DBLP:conf/eccv/WuH0WL20,DBLP:conf/iccv/RidnikBZNFPZ21} and Normalized Discounted Cumulative Gain (NDCG) \cite{DBLP:conf/kdd/PrabhuV14,kalina2018bayes}, do not have such transformation.

Faced with these measures, prior arts have provided a series of theoretical and empirical insights \cite{DBLP:journals/ai/GaoZ13,DBLP:conf/icml/WuZ17,kalina2018bayes,DBLP:conf/nips/X19,DBLP:conf/nips/WuZ20,DBLP:conf/nips/WuLXZ21}. However, two important problems remain open. On one hand, the Bayes optimal solutions to different measures might be different. Consequently, \textbf{optimizing a specific measure does not necessarily induce better model performances in terms of the others}. For example, \cite{DBLP:conf/nips/X19} theoretically shows that no multiclass reduction can be optimal for both precision@K and recall@K. Empirically, as shown in Fig.\ref{fig:tendency_in_content}(a), the model performances on different measures show different trends when optimizing the ranking loss. One potential remedy is to optimize the more discriminating listwise measures. However, the connection among these measures has not been well studied, though some literature has compared the measures within the same category \cite{kalina2018bayes,DBLP:conf/nips/X19,DBLP:conf/nips/WuZ20}. On the other hand, \textbf{efficient optimization on these complicated measures is still challenging}. For instance, \cite{DBLP:journals/ai/GaoZ13} shows that convex surrogate losses are inconsistent with the ranking loss. Although \cite{DBLP:conf/icml/DembczynskiKH12} presents a consistent surrogate, its generalization property and empirical performance are not satisfactory \cite{DBLP:conf/nips/WuLXZ21}. Hence, a natural question arises: 

\begin{qbox}
    Whether there exists a measure that is \textbf{(A)} compatible with other measures and also \textbf{(B)} easy to optimize?
\end{qbox}

\begin{figure}[t]
        \centering
        \subfigure[Optimize the ranking loss]{\includegraphics[width=0.48\linewidth]{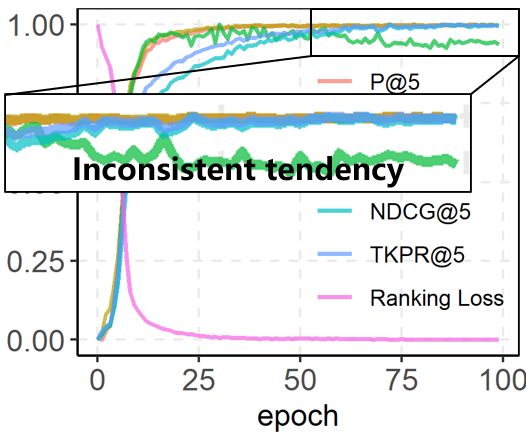}}
        \subfigure[Optimize TKPR]{\includegraphics[width=0.48\linewidth]{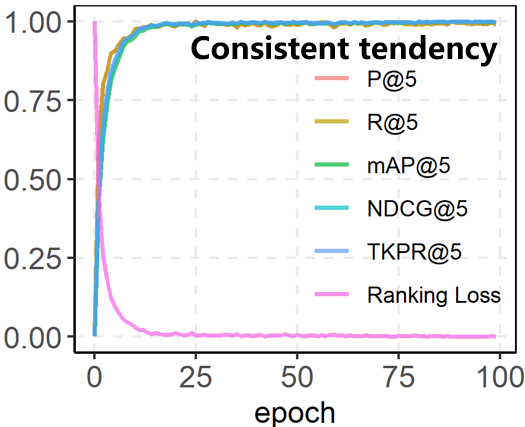}}
        \caption{Normalized ranking-based measures \textit{w.r.t.} the training epoch on the Pascal VOC 2007 dataset in the MLC setting. (a) When optimizing the ranking loss, the changes in different measures are inconsistent. (b) By contrast, when optimizing TKPR, the changes are highly consistent.}
        \label{fig:tendency_in_content}
\end{figure}

    To answer this question, as shown in Fig.\ref{fig:overview}\textbf{(a)}, we construct a novel measure named Top-K Pairwise Ranking (TKPR) by integrating the intelligence of existing ranking-based measures. As shown in Sec.\ref{subsec:def_tkpr}, TKPR has three equivalent formulations that exactly correspond to the aforementioned three categories of measures. On top of the pointwise formulation, the analysis in Sec.\ref{subsec:vs_p_r} shows that TKPR is more discriminating than precision@K and recall@K. Then, Sec.\ref{subsec:vs_ranking_loss} shows that the pairwise formulation of TKPR is the upper bound of the ranking loss. Finally, based on the listwise formulation, Sec.\ref{subsec:vs_ap} shows that TKPR has the same order as a cut-off version of AP. To sum up, as shown in Fig.\ref{fig:tendency_in_content}(b), \textbf{TKPR is compatible with existing ranking-based measures}, which answers \textbf{(A)}.

Considering these advantages, it is appealing to construct an Empirical Risk Minimization (ERM) framework for TKPR, as an answer to \textbf{(B)}. As shown in Sec.\ref{sec:roadmap}, the original objective has the following abstract formulation:
\begin{equation}
    \label{eq:op_abs}
    \min_{f} \E{(\boldsymbol{x}, \boldsymbol{y})}{\sum_{i} \sum_{k} \ell_{i, k}(f; \boldsymbol{x})},
\end{equation}
    where the loss $\ell_{i, k}$ takes the model prediction $f(\boldsymbol{x})$, a relevant label $i$, and the top-$k$ label as inputs.

First, in Sec.\ref{sec:consistency}, we replace the discrete loss $\ell_{i, k}$ with differentiable surrogate losses. One basic requirement for such surrogates is Fisher Consistency, \textit{i.e.}, optimizing the surrogate objective should recover the solution to Eq.(\ref{eq:op_abs}). In this direction, prior arts \cite{DBLP:conf/nips/WydmuchJKBD18,DBLP:conf/nips/X19,DBLP:journals/pami/WangXYHCH23} assume that no ties exist in the conditional distribution $\pp{y_i = 1 \mid \boldsymbol{x}}$. By contrast, we adopt a more practical with-tie assumption since some labels could be highly correlated in the multi-label setting. Under this assumption, we first present a necessary and sufficient condition for the Bayes optimal solution to Eq.(\ref{eq:op_abs}) in Sec.\ref{subsec:bayes}. On top of this, a sufficient condition is established for TKPR Fisher consistency in Sec.\ref{subsec:consistency}, which indicates that \textbf{common convex losses are all reasonable surrogates}.

Then, we turn to minimize the unbiased estimation of the surrogate objective over the training data sampled from the distribution. But, whether the performance on training data can generalize well to unseen data? First, in Sec.\ref{subsec:gen_traditional}, we show that, based on the traditional techniques \cite{DBLP:conf/nips/WuZ20,DBLP:conf/nips/WuLXZ21}, the proposed framework can achieve a generalization bound proportional to $\sqrt{K}$, where $K$ is a hyperparameter determined by the largest ranking position of interest. For scenarios requiring a large $K$, this bound is rather unfavorable. To fix this issue, in Sec.\ref{subsec:sharper_bound}, we extend the traditional definition of Lipschitz continuity and propose a novel technique named data-dependent contraction. On top of this, we can \textbf{obtain informative generalization bounds} that not only eliminate the dependence on $K$, but also become sharper under an imbalanced label distribution. What's more, the proposed framework also achieves sharper bounds on the pairwise ranking loss. Eventually, in Sec.\ref{subsec:prac_bound}, we show that the novel technique is applicable to both kernel-based models and convolutional neural networks by presenting the corresponding practical results.

    Finally, the empirical results presented in Sec.\ref{sec:experiment} validate the effectiveness of the learning algorithm induced by the ERM framework and the theoretical results.

To sum up, the contribution of this paper is four-fold:
    \begin{itemize}
        \item \textbf{New measure with detailed comparison}: We propose a novel measure named TKPR that is compatible with existing ranking-based MLC measures such as precision@K, recall@K, the ranking loss, NDCG@K, and AP@K.
        \item \textbf{ERM framework with consistent convex surrogates}: An ERM framework for TKPR optimization is established with convex surrogate losses, supported by Fisher consistency. 
        \item \textbf{Technique for generalization analysis}: A novel technique named data-dependent contraction provides sharp and informative generalization bounds.
        \item \textbf{Induced Algorithm and Empirical Validation}: The empirical results not only speak to the effectiveness of the learning algorithm induced by the ERM framework, but also validate the theoretical results.
    \end{itemize}

\section{Related Work}
\subsection{Multi-Label Classification}
In this part, we outline existing methods and some common settings for multi-label classification.

According to the taxonomy presented in \cite{DBLP:journals/tkde/ZhangZ14}, traditional methods for MLC fall into two categories. The first category, named \textbf{problem transformation}, transforms MLC into other well-studied problems. For example, one-vs-all approaches treat the prediction on each label as a binary classification problem \cite{DBLP:journals/pr/BoutellLSB04,DBLP:conf/icml/DembczynskiCH10,DBLP:journals/ml/DembczynskiWCH12}. Ranking-based approaches optimize the pairwise ranking between relevant and irrelevant labels \cite{DBLP:journals/ml/FurnkranzHMB08,DBLP:journals/ai/GaoZ13}. And \cite{DBLP:journals/tkde/TsoumakasKV11,DBLP:journals/pr/BoutellLSB04,DBLP:conf/icml/JerniteCS17,DBLP:conf/nips/X19} reduce MLC to multiple multiclass problems, where each problem consists of a relevant label. The second category, named \textbf{algorithm adaptation}, utilizes traditional learning techniques to model multi-label data such as decision tree \cite{DBLP:conf/pkdd/ClareK01} and SVM \cite{DBLP:conf/nips/ElisseeffW01}.

In the era of deep learning, algorithm adaptation methods utilize the high learning capability of neural networks to model the correlation between classes. For example, \cite{DBLP:conf/cvpr/WangYMHHX16,DBLP:conf/cvpr/ChenWWG19,DBLP:conf/aaai/Tang0XPWC20} utilize Recurrent Neural Networks (RNNs) or Graph Convolutional Networks (GCNs) to exploit the higher-order dependencies among labels. \cite{DBLP:conf/iccv/WangCLXL17,DBLP:conf/aaai/YouGCLBW20,DBLP:conf/eccv/YeHPWQ20} capture the attentional regions of the image by exploiting the information hidden in the label dependencies. As an orthogonal direction, \textbf{loss-oriented methods}, which follow the inspiration of problem transformation, aims to boost the learning process with well-designed loss functions. For example, \cite{DBLP:conf/eccv/WuH0WL20} proposes a weighted binary loss to handle the imbalanced label distribution. \cite{DBLP:conf/iccv/RidnikBZNFPZ21} designs an adaptive weight scheme for the binary loss to pay more attention to hard-negative labels. Recently, \cite{DBLP:conf/cvpr/LiuLLHYY22} employs causal inference to eliminate the contextual bias induced by co-occurrence but out-of-interest objects in images.

In anthor direction, an emerging trend is to learn with partial annotations since fully-annotated multi-label datasets are generally expensive. However, the definition of partial annotation still lacks consensus. For example, MLML \cite{DBLP:conf/aaai/SunZZ10,DBLP:conf/cvpr/ColeALPMJ21,DBLP:conf/eccv/ZhouCWCH22,DBLP:journals/ijcv/ChenZL23} assumes that only a part of relevant labels, even only a single relevant label, is available. This setting is correlated to extreme multi-label learning, where an extremely large number of candidate labels is of interest \cite{DBLP:conf/sigir/LiuCWY17,DBLP:conf/nips/WydmuchJKBD18,kalina2018bayes,DBLP:journals/corr/abs-2210-03968}. As a comparison, the other MLML setting assumes that a part of irrelevant labels is also available \cite{DBLP:conf/icpr/WuLWHJ14,DBLP:journals/ijcv/WuJLGL18,DBLP:conf/cvpr/BaruchRFBZNZ22}. And PML \cite{DBLP:conf/aaai/XuLG19,DBLP:journals/pami/XieH22} requires all relevant labels and assumes that extra false positive labels exist. 

In this paper, the proposed framework belongs to both problem transformation and loss-oriented methods. Besides, this framework is naturally applicable to MLML since the objective does not require irrelevant labels as inputs. Thus, we also provide the empirical results under the MLML setting in Sec.\ref{sec:experiment}.

\subsection{Optimization on Ranking-Based Measures}
In this part, we briefly review the optimization methods on ranking-based measures. Note that \textbf{the scenario is not limited to multi-label classification}. In fact, a significant part of methods in this direction focuses on the binary case, where the top-ranked instances are of interest, rather than labels.

\subsubsection{Top-K Optimization}
The top-K measure, which checks whether the ground-truth label is ranked higher than K, is a popular metric for multiclass classification when semantic ambiguity exists among the classes \cite{DBLP:journals/tmm/XuXHY14,DBLP:journals/ijcv/RussakovskyDSKS15,DBLP:conf/cvpr/HeZRS16}. The key issue of top-K optimization is its Bayes optimal solution. To be specific, the early work \cite{DBLP:conf/nips/LapinHS15} simply relaxes the multiclass hinge loss. However, \cite{DBLP:conf/cvpr/Lapin0S16,DBLP:journals/pami/LapinHS18} point out that the induced objective is inconsistent with the top-K measure. In other words, optimizing the relaxed objective does not necessarily recover the solution to the original one. To fix this issue, \cite{DBLP:conf/icml/YangK20} establishes a sufficient and necessary condition for top-K consistency, that is, the top-K preserving property. 

Faced with multi-label data, \cite{DBLP:conf/nips/X19} shows that the pointwise measures, precision@K and recall@K, is equivalent to the top-K measure averaged on relevant labels. In Sec.\ref{sec:tkpr}, we show that TKPR has a similar but more discriminating pointwise formulation.

\subsubsection{Multi-Label AUC Optimization}
As a representative pairwise ranking measure, AUC measures the probability that the positive instance is ranked higher than the negative one in each pair \cite{DBLP:conf/ijcai/LingHZ03}. In multi-label learning, AUC has three formulations that average correctly-ranked pairs on each label, each instance, and prediction matrix, named macro-AUC, instance-AUC, and micro-AUC, respectively \cite{DBLP:conf/icml/WuZ17}. Formally, the ranking loss is equivalent to instance-AUC. However, the theoretical analyses cannot simply adapt from the single-labeled AUC \cite{DBLP:journals/pami/YangXBCH22,DBLP:conf/nips/WangX00CH22,DBLP:journals/csur/YangY23,DBLP:journals/pami/YangXBHCH23,DBLP:conf/nips/DaiX0CH23} to the ranking loss. To be specific, \cite{DBLP:journals/ai/GaoZ13} points out that no convex surrogate loss is consistent with the ranking loss. To address this issue, \cite{DBLP:conf/icml/DembczynskiKH12,DBLP:conf/nips/WuLXZ21} propose pointwise convex surrogates for the ranking loss. Unfortunately, neither the theoretical generalization bound nor the empirical results of these surrogates is satisfactory. 

As shown in Sec.\ref{subsec:pairwise_measures} and Sec.\ref{sec:generalization}, the proposed framework provides a promising solution to minimizing the ranking loss. And the empirical analysis in Sec.\ref{sec:experiment} also validates its effectiveness.


\subsubsection{AP \& NDCG Optimization}
\label{subsec:related_ap_ndcg}
AP and NDCG are two popular performance measures for ranking systems, particularly in the fields of information retrieval \cite{DBLP:conf/sigir/XuL07,DBLP:conf/sigir/RadlinskiC10} and recommender systems \cite{DBLP:conf/www/XuBCC12,DBLP:conf/aaai/HuangWCMG15}. AP is adaptable to largely skewed datasets since it is an asymptotically-unbiased estimation of AUPRC, which is insensitive to the data distribution \cite{DBLP:conf/icml/DavisG06}. On the other hand, NDCG considers the graded relevance values, which allows it performs beyond the relevant/irrelevant scenario. 

    To the best of our knowledge, existing work on AP \& NDCG optimization generally focuses on the setting of retrieval, where the ranking of instances is of interest \cite{DBLP:conf/cvpr/MohapatraRJKK18,DBLP:conf/nips/RamziTRAB21,DBLP:conf/eccv/BrownXKZ20,wen2022exploring,DBLP:conf/nips/SwezeyGCE21,DBLP:conf/icml/QiuHZZY22}. By contrast, multi-label learning focuses on the ranking of labels \cite{DBLP:conf/eccv/WuH0WL20,DBLP:conf/iccv/RidnikBZNFPZ21,DBLP:conf/kdd/PrabhuV14,kalina2018bayes}. In this direction, although random forest algorithms can directly optimize NDCG-based losses \cite{DBLP:conf/kdd/PrabhuV14}, it is not applicable to neural networks. 

In Sec.\ref{sec:tkpr}, we show that optimizing TKPR is also favorable to the model performance on both NDCG@K and AP@K.

\begin{table}[t]
    \renewcommand{\arraystretch}{1.2}
    \caption{Some important Notations used in this paper.}
    \label{table:notation}
    \centering
    \footnotesize 
    \begin{tabular}{p{0.17\linewidth}p{0.78\linewidth}}
        \toprule
        Notation & Description \\
        \midrule
        $\mathcal{X}, \mathcal{Y}$ & the input space and output space \\
        $C$ & the number of classes \\
        $\mathcal{L}$ & the set of possible labels \\
        $K$ & the hyperparameter of TKPR \\
        $\mathcal{D}$ & the joint distribution defined on $\mathcal{Z} := \mathcal{X} \times \mathcal{Y}$ \\
        $(\boldsymbol{x}, \boldsymbol{y})$ & a sample belonging to $\mathcal{Z}$  \\
        $\mathcal{S}$ & the dataset \textit{i.i.d.} sampled from $\mathcal{D}$ \\
        $\eta(\boldsymbol{x})_i$ & the conditional probability $\pp{y_i = 1 | \boldsymbol{x}}$ \\
        $\mathcal{P}(\boldsymbol{y}), \mathcal{N}(\boldsymbol{y})$ & the relevant/irrelevant label set of $\boldsymbol{y}$ \\
        $N(\boldsymbol{y}), N_{-}(\boldsymbol{y})$ & the number of relevant/irrelevant labels of $\boldsymbol{y}$ \\
        \midrule
        $\pi_{\boldsymbol{s}}(i)$ & the index of $s_i$ when sorting $\boldsymbol{s} \in \mathbb{R}^C$ descendingly \\
        $\sigma(\boldsymbol{s}, k)$ & the index of the top-$k$ entry in $\boldsymbol{s} \in \mathbb{R}^C$ \\
        $s_{[k]}$ & the top-$k$ entry in $\boldsymbol{s} \in \mathbb{R}^C$ \\
        $\texttt{Tie}_k(\boldsymbol{s})$ & the indices of the entries that equals to $s_{[k]}$ \\
        $\texttt{Top}_k(\boldsymbol{s})$ & the indices of the top-1 to top-$k$ entries  in $\boldsymbol{s} \in \mathbb{R}^C$ \\
        \midrule
        $f$ & a score function mapping $\mathcal{X}$ to $\mathbb{R}^{C}$ \\
        $\mathcal{F}$ & the set of score functions \\
        $m(f)$ & the measure defined on $\mathcal{D}$, \textit{i.e.}, $\E{\boldsymbol{z} \sim \mathcal{D}}{m(f(\boldsymbol{x}), \boldsymbol{y})}$ \\
        $m(f \mid \boldsymbol{x})$ & the conditional measure, \textit{i.e.}, $\E{\boldsymbol{y} \mid \boldsymbol{x}}{m(f(\boldsymbol{x}), \boldsymbol{y})}$ \\
        $m(f, \boldsymbol{y})$ & the measure defined on a sample, \textit{i.e.}, $m(f(\boldsymbol{x}), \boldsymbol{y})$ \\
        $\hat{m}(f, \mathcal{S})$ & the empirical measure defined on $\mathcal{S}$ \\
        $m_\text{MC}, m_\text{ML}$ & Abstract multiclass/multi-label measures \\
        $\alpha$ & the weighting terms for pointwise measures \\
        $\alpha_1, \alpha_2, \alpha_3$ & weighting terms for TKPR (refer to Prop.\ref{prop:tkprtondcg})\\
        $\text{reg}(f; m)$ & the regret of $f$ \textit{w.r.t.} the measure $m$ \\
        $\ell_{0-1}, \ell$ & the 0-1 loss and its surrogate loss \\
        $\hat{\mathfrak{C}}_\mathcal{S}(\mathcal{F})$ & the empirical complexity of $\mathcal{F}$ \\
        \bottomrule
    \end{tabular}
\end{table}

\section{Preliminaries}
\label{sec:pre}
    In this section, we first present basic notations and common settings for MLC, as summarized in Tab.\ref{table:notation}. Then, as shown in Fig.\ref{fig:overview}, we review existing ranking-based measures and point out their limitations.

\subsection{Notations and Settings}
Let $\mathcal{X}$ be the input space, and $\mathcal{L} := \{1, 2, \cdots, C\}$ denotes the set of possible labels. In MLC, each input $\boldsymbol{x} \in \mathcal{X}$ is associated with a label vector $\boldsymbol{y} \in \mathcal{Y} := \{0, 1\}^{C}$. Traditional MLC interprets $y_i = 1 \text{ or } 0$ as the relevance/irrelevance of label $i$ to the instance $\boldsymbol{x}$, respectively. Since fully-annotated data are generally expensive, MLC with Missing Labels (MLML) has gained rising attention, where $y_i = 0$ means that the annotation of the label $y_i$ is missing \cite{DBLP:conf/aaai/SunZZ10,DBLP:conf/mir/IbrahimEPR20,DBLP:conf/cvpr/ColeALPMJ21,DBLP:journals/corr/abs-2210-03968,DBLP:journals/pami/LiuWST22}. In other words, let $\mathcal{P}(\boldsymbol{y}) := \{i \in \mathcal{L} \mid y_i = 1\}$ denote the relevant label set of $\boldsymbol{y}$. Then, in traditional MLC, the number of relevant labels can be directly obtained by $N(\boldsymbol{y}) := \size{\mathcal{P}(\boldsymbol{y})} = \sum_{i \in \mathcal{L}} y_i$. While in MLML, we only know that the number of relevant labels is no less than $N(\boldsymbol{y})$. 

Let $\mathcal{D}$ be the joint distribution defined on $\mathcal{Z} := \mathcal{X} \times \mathcal{Y}$. Given a sample $(\boldsymbol{x}, \boldsymbol{y}) \in \mathcal{Z}$, 
\begin{equation}
    \eta(\boldsymbol{x})_i := \pp{y_i = 1 \mid \boldsymbol{x}} = \sum_{\boldsymbol{y}: y_i = 1} \pp{\boldsymbol{y} \mid \boldsymbol{x}}
\end{equation}
represents the conditional probability of the class $i$ according to $\mathcal{D}$. Some prior arts assume that no ties exist among the conditional probabilities \cite{DBLP:conf/nips/WydmuchJKBD18,DBLP:conf/nips/X19,DBLP:journals/pami/WangXYHCH23}. That is, given an input $\boldsymbol{x} \in \mathcal{X}$, for any $i \neq j \in \mathcal{L}, \eta(\boldsymbol{x})_i \neq \eta(\boldsymbol{x})_j$. However, this no-tie assumption might be a little strong since some labels might be highly correlated in the multi-label setting. In the following discussion, we will adopt the following with-tie assumption:

\begin{assumption}
    \label{asm:tiesexist}
    Given an input $\boldsymbol{x} \in \mathcal{X}$, we assume that ties might exist in the conditional distribution $\eta(\boldsymbol{x})$. That is, 
    \begin{equation}
        \forall \boldsymbol{x} \in \mathcal{X}, i \neq j \in \mathcal{L}, \pr{}{\eta(\boldsymbol{x})_i = \eta(\boldsymbol{x})_j} > 0.
    \end{equation}
\end{assumption}

Given the training data sampled from $\mathcal{D}$, our goal is to learn a score function $f: \mathcal{X} \to \mathbb{R}^C$ such that the relevant labels can be ranked higher than the irrelevant ones. Then, one can adopt the top-ranked labels as the final decision results. Following the prior arts \cite{DBLP:conf/icml/YangK20,DBLP:journals/pami/WangXYHCH23}, we adopt the worst case assumption when ties exist between the scores:
\begin{assumption}
    \label{asm:wrongly_break_ties}
    Ties might exist in $f(\boldsymbol{x})$, and all the ties will be wrongly broken. In other words, given $i, j \in \mathcal{L}$ such that $\eta(\boldsymbol{x})_{i} > \eta(\boldsymbol{x})_{j}$, if $f(\boldsymbol{x})_{i} = f(\boldsymbol{x})_{j}$, then we have $\pi_{f}(i) > \pi_{f}(j)$, where $\pi_{f}(i) \in \mathcal{L}$ denotes the index of $f(\boldsymbol{x})_{i}$ when $f(\boldsymbol{x}) \in \mathbb{R}^C$ is sorted in descending order.
\end{assumption}
\begin{corollary}
    \label{coll:wrongly_break_ties}
    Given a relevant label $i$ and an irrelevant label $j$ such that $f(\boldsymbol{x})_{i} = f(\boldsymbol{x})_{j}$, we have $\pi_{f}(i) > \pi_{f}(j)$.
\end{corollary}
\begin{remark}
    Asm.\ref{asm:wrongly_break_ties} will be used when the conditional distribution $\eta(\boldsymbol{x})$ is available, such as the proofs of Prop.\ref{prop:bayes} and Thm.\ref{thm:condition_for_consistency}. And Cor.\ref{coll:wrongly_break_ties} applies to the case where the relevant and irrelevant labels are given, such as the proof of Prop.\ref{prop:reformulation}.
\end{remark}

\subsection{Ranking-based MLC Measures}
\label{sec:ranking_measures}
To evaluate model performances in multi-label setting, prior arts have proposed various measures. According to the taxonomy in Learning-to-Rank \cite{DBLP:journals/ftir/Liu09}, we classify existing ranking-based measures into three categories: the pointwise ones, the pairwise ones, and the listwise ones.

\subsubsection{Pointwise Measures}
Essentially, pointwise measures reduce the multi-label problem to $N(\boldsymbol{y})$ multiclass classification problems. Abstractly, given a sample $(\boldsymbol{x}, \boldsymbol{y}) \in \mathcal{Z}$, we have 
\begin{equation}
    \label{eq:pointwise}
    m_\text{ML}(f, \boldsymbol{y}) = \frac{1}{\alpha} \sum_{y \in \mathcal{P}(\boldsymbol{y})} m_\text{MC}(f, y),
\end{equation}
where $m_\text{MC}: \mathbb{R}^C \times \mathcal{L} \to \mathbb{R}_+$ denotes a measure for multiclass classification, and $\alpha$ is a weighting term designed for different scenarios. As shown in \cite{DBLP:conf/nips/X19}, when selecting the top-K measure \cite{DBLP:conf/cvpr/Lapin0S16,DBLP:journals/pami/LapinHS18,DBLP:conf/icml/YangK20} as $m_\text{MC}$, $m_\text{ML}$ will be exactly equivalent to two pointwise measures: precision@K and recall@K.

\begin{proposition} 
    \label{prop:consistency_PAL_PALW}
    Let $\I{\cdot}$ denote the indicator function. Given the top-K measure $m_\text{K}(f, y) = \I{\pi_{f}(y) \le K}$. When $\alpha = K$, $m_\text{ML}$ is equivalent to
    \begin{equation}
        \texttt{P@K}(f, \boldsymbol{y}):= \frac{1}{K} \sum_{k=1}^{K} y_{\sigma(f, k)};
    \end{equation}
    when $\alpha = N(\boldsymbol{y})$, $m_\text{ML}$ is equivalent to
    \begin{equation}
        \texttt{R@K}(f, \boldsymbol{y}) := \frac{1}{N(\boldsymbol{y})} \cdot \sum_{k=1}^{K} y_{\sigma(f, k)},
    \end{equation}
    where $\sigma(f, k) \in \mathcal{L}$ denotes the top-$k$ class in $f(\boldsymbol{x})$. 
\end{proposition}

Cor.4 of \cite{DBLP:conf/nips/X19} shows that, under the no-tie assumption, the Bayes optimal solutions to optimizing \texttt{P@K} and \texttt{R@K} are generally inconsistent. In other words, optimizing one measure cannot guarantee the performance on the other one. We will update this corollary under Asm.\ref{asm:tiesexist} in Appendix \ref{app:bayesofpr}. In Sec.\ref{subsec:vs_ap}, we show that optimizing our proposed measure, \textit{i.e.}, TKPR can boost both precision and recall performance.

\subsubsection{Pairwise Measures}
\label{subsec:pairwise_measures}
    Pairwise measures check whether in each label pair, the relevant label is ranked higher than the irrelevant one. Similar inspiration exists in many paradigms such as contrastive learning \cite{DBLP:conf/nips/KhoslaTWSTIMLK20, DBLP:conf/aaai/AljundiPSCR23}. For multi-label learning, the ranking loss, which is essentially equivalent to instance-AUC \cite{DBLP:conf/icml/WuZ17}, calculates the fraction of mis-ranked pairs for each sample \cite{DBLP:journals/ai/GaoZ13}:
\begin{equation}
    \label{eq:ranking_loss}
    \begin{aligned}
        & L_\text{rank}(f, \boldsymbol{y}) \\
        & := \frac{1}{N(\boldsymbol{y}) N_{-}(\boldsymbol{y})} \sum_{i \in \mathcal{P}(\boldsymbol{y})}\sum_{j \in \mathcal{N}(\boldsymbol{y})} \ell_{0-1}(s_i - s_j),
    \end{aligned}
\end{equation}
where $\boldsymbol{s} := f(\boldsymbol{x})$ denotes the output of the score function, $\ell_{0-1}(t) := \I{t \le 0}$ is the 0-1 loss, $\mathcal{N}(\boldsymbol{y}) := \{j \in \mathcal{L} \mid y_j = 0\}$ is the irrelevant label set, and $N_{-}(\boldsymbol{y}) := \size{\mathcal{N}(\boldsymbol{y})}$ is the number of irrelevant labels.

To optimize the ranking loss, one should replace the non-differentiable 0-1 loss with a differentiable surrogate loss $\ell$, which induces the following surrogate objective:
\begin{equation}
    \label{eq:ranking_loss_surrogate}
    \begin{aligned}
    & L_\text{rank}^{\ell}(f, \boldsymbol{y}) := \\
    & \frac{1}{N(\boldsymbol{y}) N_{-}(\boldsymbol{y})} \sum_{i \in \mathcal{P}(\boldsymbol{y})}\sum_{j \in \mathcal{N}(\boldsymbol{y})} \ell(s_i - s_j).
    \end{aligned}
\end{equation}
However, as pointed out in \cite{DBLP:conf/icml/WuZ17}, under Asm.\ref{asm:tiesexist}, $L_\text{rank}^{\ell}(f, \boldsymbol{y})$ induced by any convex surrogate loss is inconsistent with $L_\text{rank}(f, \boldsymbol{y})$, and we have to select a non-convex one to optimize. Although \cite{DBLP:conf/icml/DembczynskiKH12} proposes a consistent surrogate, its generalization property is not satisfactory \cite{DBLP:conf/nips/WuLXZ21}, which is summarized later in Tab.\ref{tab:loss_comparison}.

To address this issue, in Sec.\ref{subsec:vs_ranking_loss}, we show that the TKPR loss is the upper bound of the ranking loss. Furthermore, Sec.\ref{sec:generalization} shows that the Empirical Risk Minimization framework for TKPR also enjoys a sharp generalization on the ranking loss.

\subsubsection{Listwise Measures}
Listwise measures assign different weights to labels according to their positions in the ranking list. In this way, these measures pay more attention to the top-ranked labels. For example, Normalized Discounted Cumulative Gain at K (NDCG@K) \cite{DBLP:journals/corr/abs-2210-03968}, weighs the importance of different positions with a specified decreasing discount functions:
\begin{equation*}
    \begin{aligned}
        \texttt{DCG@K}(f, \boldsymbol{y}) & := \sum_{k=1}^{K} D(k) y_{\sigma(f, k)}, \\
        \texttt{IDCG@K}(\boldsymbol{y}) & := \max_{f} \texttt{DCG@K}(f, \boldsymbol{y}) = \sum_{k=1}^{N_K(\boldsymbol{y})} D(k), \\
        \texttt{NDCG@K}(f, \boldsymbol{y}) & := \frac{\texttt{DCG@K}(f, \boldsymbol{y})}{\texttt{IDCG@K}(\boldsymbol{y})}, \\
    \end{aligned}
\end{equation*} 
where $N_K(\boldsymbol{y}) := \min\{K, N(\boldsymbol{y})\}$, and $D(k)$ is the discount function. Here, we consider two common choices \cite{DBLP:conf/colt/WangWLHL13}:
\begin{equation}
    \label{eq:discount}
    \begin{aligned}
        D_{log}(k) & :=  1 / \log_2{(k + 1)}, \\
        D_l(k) & := K + 1 - k. \\
    \end{aligned}
\end{equation}
As another example, Average Precision (AP) \cite{DBLP:conf/icml/WuZ17,DBLP:conf/eccv/WuH0WL20,DBLP:conf/iccv/RidnikBZNFPZ21,DBLP:conf/nips/WenX00H22,DBLP:journals/pami/WenX00H24} averages the precision performance at different recall performances:
\begin{equation}
    \texttt{AP}(f, \boldsymbol{y}) := \frac{1}{C} \sum_{k=1}^{C} y_{\sigma(f, k)} \cdot \texttt{P@k}(f, \boldsymbol{y}).
\end{equation}
In this paper, we consider its cut-off version \cite{DBLP:conf/cikm/LiGX17}:
\begin{equation}
    \texttt{AP@K}(f, \boldsymbol{y}) := \frac{1}{N_K(\boldsymbol{y})} \sum_{k=1}^{K} y_{\sigma(f, k)} \cdot \texttt{P@k}(f, \boldsymbol{y}).
\end{equation}

As shown in Sec.\ref{subsec:related_ap_ndcg}, existing work on NDCG \& AP optimization generally focuses on the ranking of instances, rather than that of labels. In Sec.\ref{subsec:def_tkpr} and Sec.\ref{subsec:vs_ap}, we will show that optimizing TKPR can also boost the model performance on the two measures.

\section{TKPR and its Advantages}
\label{sec:tkpr}
In this section, we begin by the definition of TKPR. Then, detailed analyses illustrate how this measure is compatible with existing ranking-based measures, whose outline is shown in Fig.\ref{fig:overview}.

\subsection{Three formulations of TKPR}
\label{subsec:def_tkpr}
To bridge the gap among existing ranking-based measures, we average the ranking results between relevant labels and the top-ranked ones:
\begin{equation}
    \label{eq:tkpr}
    \texttt{TKPR}(f, \boldsymbol{y}) := \frac{1}{\alpha K} \sum_{y \in \mathcal{P}(\boldsymbol{y})} \sum_{k \le K} \I{s_y > s_{[k]}},
\end{equation}
where $s_{[k]}$ denotes the $k$-largest entry in $\mathbf{s}$, and $\alpha$ denotes the weighting terms.

At the first glance, Eq.(\ref{eq:tkpr}) has a pairwise formulation. But if we review Eq.(\ref{eq:pointwise}), it is not difficult to find that TKPR also enjoys a pointwise formulation:
\begin{equation}
    \label{eq:tkpr_point}
    \texttt{TKPR}(f, \boldsymbol{y}) = \frac{1}{\alpha} \sum_{y \in \mathcal{P}(\boldsymbol{y})} \left[ \frac{1}{K} \sum_{k \le K}  m_k(f, y) \right],
\end{equation}
Besides, TKPR also exhibit a listwise formulation, whose proof can be found in Appendix \ref{app:tkprtondcg}.
\begin{restatable}{proposition}{tkprtondcg}
    \label{prop:tkprtondcg}
    Given a score function $f$ and $(\boldsymbol{x}, \boldsymbol{y}) \in \mathcal{Z}$,
    \renewcommand{\arraystretch}{1.5}
    $$
        \texttt{TKPR}^{\alpha}(f, \boldsymbol{y}) = \left\{\begin{array}{ll}
            \frac{1}{K} \cdot \texttt{DCG-l@K}(f, \boldsymbol{y}), \ & \alpha = \alpha_1, \\
            \frac{1}{K} \cdot \texttt{DCG-ln@K}(f, \boldsymbol{y}), \ & \alpha = \alpha_2, \\
            \frac{1}{K} \cdot \texttt{NDCG-l@K}(f, \boldsymbol{y}), \ & \alpha = \alpha_3, \\
        \end{array}\right.
    $$
    where $\texttt{DCG-l@K}, \texttt{NDCG-l@K}$ represent the NDCG measures equipped with the linear discount function $D_l$ defined in Eq.(\ref{eq:discount}), respectively; 
    \begin{equation}
        \texttt{DCG-ln@K}(f, \boldsymbol{y}) := \frac{1}{N_K(\boldsymbol{y})} \cdot \texttt{DCG-l@K}(f, \boldsymbol{y}).
    \end{equation}
    denotes the linear DCG@K with a linear weighting term; \textbf{$\alpha_1 = 1$, $\alpha_2 := N_{K}(\boldsymbol{y})$, $\alpha_3 := N_{K}(\boldsymbol{y}) \tilde{N}_K(\boldsymbol{y})$}, and $$\tilde{N}_K(\boldsymbol{y}) := \left[ 2K + 1 - N_K(\boldsymbol{y}) \right] / 2.$$
\end{restatable}
Hence, optimizing TKPR is favorable to the model performance on the NDCG measures. Next, on top of these formulations, we will show how TKPR is compatible with other ranking-based measures.

\subsection{TKPR v.s. P@K and R@K}
\label{subsec:vs_p_r}
Intuitively, TKPR is more discriminating than \texttt{P@K} and \texttt{R@K}. That is, TKPR can distinguish finer-grained differences on model performances. For example, given $K=4$ and a sample with 
$$\boldsymbol{y} = \begin{pmatrix}
    1 & 1 & 0 & 0 & 0 & 0
\end{pmatrix},$$
let $\boldsymbol{s}_1 = f_1(\boldsymbol{x})$ and $\boldsymbol{s}_2 = f_2(\boldsymbol{x})$ be two predictions such that
$$\begin{aligned}
    \boldsymbol{s}_1 & = \begin{pmatrix}
        0.8 & 0.8 & 0.9 & 0.9 & 0.2 & 0.2
    \end{pmatrix}, \\
    \boldsymbol{s}_2 & = \begin{pmatrix}
        0.9 & 0.9 & 0.8 & 0.8 & 0.2 & 0.2
    \end{pmatrix}. \\
\end{aligned}$$
It is clear that $f_2$ performs better than $f_1$ since the relevant labels are ranked higher. However, \texttt{P@K} fails to distinguish this difference:
$$\begin{aligned}
    \texttt{P@K}(f_1, \boldsymbol{y}) = \frac{2}{4} & = \frac{2}{4} = \texttt{P@K}(f_2, \boldsymbol{y}), \\
    \texttt{TKPR}^{\alpha_1}(f_1, \boldsymbol{y}) = \frac{2 + 1}{4} & <  \frac{4 + 3}{4} = \texttt{TKPR}^{\alpha_1}(f_2, \boldsymbol{y}). \\
\end{aligned}$$

Although this example aligns with our intuitive understanding, there are evident instances of counterexamples. Given another prediction $\boldsymbol{s}_3 = f_3(\boldsymbol{x})$ such that
$$
    \boldsymbol{s}_3 = \begin{pmatrix}
        0.8 & 0.1 & 0.9 & 0.2 & 0.2 & 0.2
    \end{pmatrix},
$$
TKPR fails to distinguish the difference:
$$\begin{aligned}
    \texttt{P@K}(f_1, \boldsymbol{y}) = \frac{2}{4} & > \frac{1}{4} = \texttt{P@K}(f_3, \boldsymbol{y}), \\
    \texttt{TKPR}^{\alpha_1}(f_1, \boldsymbol{y}) = \frac{2 + 1}{4} & = \frac{3}{4} = \texttt{TKPR}^{\alpha_1}(f_3, \boldsymbol{y}). \\
\end{aligned}$$

In view of this, we present a precise definition of statistical discriminancy, which enables us to conduct a comprehensive comparison between measures.

\begin{definition}[Statistical discriminancy \cite{DBLP:conf/ijcai/LingHZ03}]
    Given two measures $m_1$ and $m_2$ and two predictions $\boldsymbol{s}, \boldsymbol{s}'$, let 
    $$\begin{aligned}
        P & := \{ (\boldsymbol{s}, \boldsymbol{s}') | m_1(\boldsymbol{s}) > m_1(\boldsymbol{s}'), m_2(\boldsymbol{s}) = m_2(\boldsymbol{s}') \}, \\
        S & := \{ (\boldsymbol{s}, \boldsymbol{s}') | m_1(\boldsymbol{s}) = m_1(\boldsymbol{s}'), m_2(\boldsymbol{s}) > m_2(\boldsymbol{s}') \}. \\
    \end{aligned}$$
    Then, the degree of discriminancy between $m_1$ and $m_2$ is defined by $\texttt{Dis}(m_1, m_2) := \size{P} / \size{S}$. We say $m_1$ is statistically more discriminating if and only if $\texttt{Dis}(m_1, m_2) > 1$. In this case, $m_1$ could discover more discrepancy that $m_2$ fails to distinguish. 
\end{definition}

Then, the following theorem validates our conjunction, whose details can be found in Appendix \ref{app:conanddis}.
\begin{restatable}{theorem}{conanddis}
    \label{thm:conanddis}
    Given $K > 1$, TKPR is statistically more discriminating than \texttt{P@K} and \texttt{R@K}.
\end{restatable}
\begin{sketch}
    The proof is based on the concept of partition number in combinatorial mathematics \cite{andrews_1984}. To be concise, we first define a partition number $pn_K(a, b)$ that exactly equals to the number of predictions with $\texttt{P@K}(f, \boldsymbol{y}) = b / K, \texttt{TKPR}^{\alpha_1}(f, \boldsymbol{y}) = a / K$. Then, $\size{P} - \size{S}$ can be denoted as the function of $pn_K(a, b)$. We further construct a recurrence formula for $pn_K(a, b)$, which helps complete the proof.
\end{sketch}

\subsection{TKPR v.s. the ranking loss}
\label{subsec:vs_ranking_loss}
To compare TKPR with the ranking loss, we have the following equivalent formulation of Eq.(\ref{eq:tkpr_point}), where the indicator function is replaced by the 0-1 loss. The corresponding proof can be found in Appendix \ref{app:reformulation}. 
\begin{restatable}{proposition}{reformulation}
    \label{prop:reformulation}
    Under Asm.\ref{asm:wrongly_break_ties}, maximizing TKPR is equivalent to minimizing the TKPR loss
    \begin{equation}
        \label{eq:op1_}
        L_{K}^\alpha(f, \boldsymbol{y}) := \frac{1}{\alpha K} \sum_{y \in \mathcal{P}(\boldsymbol{y})} \sum_{k \le K + 1} \ell_{0-1} \left(s_y - s_{[k]}\right),
    \end{equation}
    Furthermore, the following inequality holds:
    \begin{equation}
        L_\text{rank}(f, \boldsymbol{y}) \le \frac{\alpha}{N(\boldsymbol{y})} L_{K}^\alpha(f, \boldsymbol{y}).
    \end{equation}
\end{restatable}
Since $L_K^\alpha$ is the upper bound of $L_\text{rank}$, minimizing $L_K^\alpha$ will boost the performance on $L_\text{rank}$. Besides, the TKPR loss has the following advantages:
\begin{itemize}
    \item \textbf{(Scenario).} $L_\text{rank}$ can only be applied to traditional MLC since the irrelevant labels are explicitly required. Whereas, $L_K^\alpha$ loss is applicable to both traditional MLC and MLML.
    \item \textbf{(Complexity).} $L_\text{rank}$ suffers from a high computational burden $\mathcal{O}(C^2)$ \cite{DBLP:conf/nips/WuLXZ21}. By contrast, $L_K^\alpha$ enjoys a complexity of $\mathcal{O}(C K)$.
    \item \textbf{(Theoretical results).} Further analysis in Sec.\ref{sec:consistency} and Sec.\ref{sec:generalization} shows that $L_K^\alpha$ has some other superior properties, which is summarized in Tab.\ref{tab:loss_comparison}.
\end{itemize}

\subsection{TKPR v.s. AP@K}
\label{subsec:vs_ap}
The following theorem shows that benefiting from the linear discount function, TKPR has the same order as \texttt{AP@K}. In other words, optimizing TKPR can help improve the model performance on \texttt{AP@K}, whose proof can be found in Appendix \ref{app:ndcgboundmap}.
\begin{restatable}{theorem}{ndcgboundmap}
    \label{thm:ndcgboundmap}
    Given a score function $f$ and $(\boldsymbol{x}, \boldsymbol{y}) \in \mathcal{Z}$, there exists a constant $\rho > 0$ such that
    \begin{equation}
        \begin{aligned}
            \rho \cdot \texttt{TKPR}^{\alpha_2}(f, \boldsymbol{y}) & \le  \texttt{AP@K}(f, \boldsymbol{y}) \\
            & \le K \cdot \texttt{TKPR}^{\alpha_1}(f, \boldsymbol{y}), \\
        \end{aligned}
    \end{equation}
    where the upper bound of $\rho$ is bounded in  
    $$
        \left[ 1 / (K + 1), K \ln(K + 1) \right].
    $$
\end{restatable}
\begin{remark}
    \label{rem:bound_rho}
    Abstractly, the upper bound of $\rho$ has two parts:
    \begin{equation}
        \rho \le  U_1 (f, K) \cdot U_2 (f, K), 
    \end{equation}
    where $U_1$ is increasing \textit{w.r.t.} $\texttt{P@K}(f, \boldsymbol{y})$ but is not monotonic \textit{w.r.t.} the ranking performance of the model; $U_2$ is also increasing \textit{w.r.t.} $\texttt{P@K}(f, \boldsymbol{y})$. We present the concrete formulations and the corresponding analysis in Appendix \ref{app:ndcgboundmap}.
\end{remark}

The training process generally improves both \texttt{P@K} and the ranking performance. According to Rem.\ref{rem:bound_rho}, TKPR will bound \texttt{AP@K} more tightly when $\texttt{P@K}$ increases. While this bound might alternate between becoming tighter and looser when the learning process only improves the ranking performance. Fortunately, the empirical results in Fig.\ref{fig:tendency_in_content} show that optimizing TKPR can consistently boost \texttt{AP@K}, and more results can be found in Sec.\ref{sec:experiment}.

\section{ERM Framework for TKPR}
\label{sec:roadmap}
So far, we have known that optimizing TKPR will be compatible with existing ranking-based measures. Thus, it becomes appealing to construct an Empirical Risk Minimization (ERM) framework for TKPR. To this end, we first present the following risk minimization problem  based on Eq.(\ref{eq:op1_}):
\begin{equation}
    \label{eq:op1}
    \begin{aligned}
        & (OP_1)\phantom{|} \min_{f} \mathcal{R}_K^\alpha(f) := \E{\boldsymbol{z} \sim \mathcal{D}} {L_{K}^\alpha(f, \boldsymbol{y})}, \\
        & = \E{\boldsymbol{z} \sim \mathcal{D}} {\frac{1}{\alpha K} \sum_{y \in \mathcal{P}(\boldsymbol{y})} \sum_{k \le K + 1} \ell_{0-1} \left(s_y - s_{[k]}\right)},
    \end{aligned}
\end{equation}
According to Eq.(\ref{eq:op1}), the main challenges for directly minimizing $\mathcal{R}_K^\alpha(f)$ are two-fold: 
\begin{itemize}
    \item[\textbf{(C1)}] The loss function $\ell_{0-1}$ is not differentiable, making graident-based methods infeasible;
    \item[\textbf{(C2)}] The data distribution $\mathcal{D}$ is unavailable, making it impossible to calculate the expectation.
\end{itemize}
Next, we will tackle the challenges \textbf{(C1)} and \textbf{(C2)} in Sec.\ref{sec:consistency} and Sec.\ref{sec:generalization}, respectively.

\subsection{Consistency Analysis of the ERM Framework}
\label{sec:consistency}
To tackle \textbf{(C1)}, one common strategy is to replace $\ell_{0-1}$ with a differentiable surrogate loss $\ell: \mathbb{R} \to \mathbb{R}_+$. Let
\begin{equation}
    L_{K}^{\alpha, \ell}(f, \boldsymbol{y}) := \frac{1}{\alpha K} \sum_{y \in \mathcal{P}(\boldsymbol{y})} \sum_{k \le K + 1} \ell \left(s_y - s_{[k]}\right)
\end{equation}
Then, we have the following surrogate objective: 
\begin{equation}
    \label{eq:op2}
    (OP_2)\phantom{|} \min_{f} \ \mathcal{R}_K^{\alpha, \ell}(f) := \E{(\boldsymbol{x}, \boldsymbol{y}) \sim \mathcal{D}} {L_{K}^{\alpha, \ell}(f, \boldsymbol{y})}.
\end{equation}
    As mentioned in Sec.\ref{subsec:pairwise_measures}, convex surrogate losses are inconsistent with the ranking loss $L_\text{rank}^{\ell}$. In view of this, a question naturally arises: whether a convex surrogate objective $(OP_2)$ is consistent with the original one? In other words, 
    \begin{qbox-tight}
        Given a convex $\ell$, can optimizing $(OP_2)$ recover the solution to $(OP_1)$?
    \end{qbox-tight}
To answer this question, we first the present the definition of TKPR consistency:
\begin{definition}[TKPR Fisher consistency]
    The surrogate loss $\ell: \mathbb{R} \to \mathbb{R}_{+}$ is Fisher consistent with TKPR if for any sequence $\{ f^{(t)} \}_{t=1}^{\infty}$,
    \begin{equation}
        \text{reg}(f^{(t)}; \mathcal{R}_{K}^{\alpha, \ell}) \to 0 \implies \text{reg}(f^{(t)}; \mathcal{R}_{K}^{\alpha}) \to 0,
    \end{equation}
    where 
    $$
        \text{reg}(f; m) := \E{(\boldsymbol{x}, \boldsymbol{y})}{m(f(\boldsymbol{x}), \boldsymbol{y})} - \inf_{g} \E{(\boldsymbol{x}, \boldsymbol{y})}{m(g(\boldsymbol{x}), \boldsymbol{y})}
    $$ 
    represents the regret of $f$ \textit{w.r.t.} the measure $m$.
\end{definition}
    \noindent Next, we present the Bayes optimal solution to $(OP_1)$ in Sec.\ref{subsec:bayes}. On top of this, a sufficient condition for TKPR consistency, which consists of convexity, is established in Sec.\ref{subsec:consistency}.

\subsubsection{TKPR Bayes Optimality}
\label{subsec:bayes}
We first define the Bayes optimal solution to TKPR optimization:
\begin{definition}[TKPR Bayes optimal]
    Given the joint distribution $\mathcal{D}$, the score function $f^*: \mathcal{X} \to \mathbb{R}^{C}$ is TKPR Bayes optimal if 
    \begin{equation}
        \label{eq:bayes_optimal}
        f^* \in \arg \inf_f \mathcal{R}_K^\alpha(f).
    \end{equation}
\end{definition}

In other words, our goal is to find the solution to Eq.(\ref{eq:bayes_optimal}). The following property is necessary for further analysis.
\begin{definition}[Top-$K$ ranking-preserving property with ties] 
    \label{def:rp}
    Given $\boldsymbol{a}, \boldsymbol{b} \in \mathbb{R}^{C}$, we say that $\boldsymbol{b}$ is \textit{top-$K$ ranking-preserving with ties} \textit{w.r.t.} $\boldsymbol{a}$, denoted as $\mathsf{RPT}_K(\boldsymbol{b}, \boldsymbol{a})$, if for any $k \le K - 1$,
    $$
        \texttt{Tie}_k(\boldsymbol{b}) = \texttt{Tie}_k(\boldsymbol{a}),
    $$ 
    and
    $$
        \texttt{Tie}_K(\boldsymbol{b}) \subset \texttt{Tie}_K(\boldsymbol{a}),
    $$
    where $\texttt{Tie}_k(\boldsymbol{a}) := \{i \in \mathcal{L} \mid a_i = a_{[k]}\}$ returns the labels having ties with $a_{[k]}$. 
\end{definition} 

Then, the following proposition reveals the sufficient and necessary condition for TKPR optimization, whose proof can be found in Appendix \ref{app:byestkpr}.
\begin{restatable}[Bayes optimality of TKPR]{proposition}{bayestkpr}
    \label{prop:bayes}
    The score function $f: \mathcal{X} \to \mathbb{R}^{C}$ is TKPR Bayes optimal if and only if for an input $\boldsymbol{x}$, the prediction $f(\boldsymbol{x})$ is top-$K$ ranking-preserving \textit{w.r.t.} $\Delta(\boldsymbol{x}) \in \mathbb{R}^C$, where
    \begin{equation}
        \Delta(\boldsymbol{x})_i := \sum_{\boldsymbol{y}: y_i = 1}\frac{\pp{\boldsymbol{y} \mid \boldsymbol{x}}}{\alpha}.
    \end{equation}
\end{restatable}
\begin{corollary}
    Given $\alpha = \alpha_1$, the Bayes optimal solution of TKPR is top-$K$ ranking preserving \textit{w.r.t.} $\eta(\boldsymbol{x})$. Given $\alpha = \alpha_2$, if the hyperparameter $K$ is large enough, the Bayes optimal solution of TKPR is top-$K$ ranking preserving \textit{w.r.t.} $\eta'(\boldsymbol{x})$, which is defined in Appendix \ref{app:bayesofpr}.
\end{corollary}
Benefiting from the additional consideration on the ranking among the top-K labels, $\mathsf{RPT}$ is stricter than the Bayes optimalities of \texttt{P@K} and \texttt{R@K} described in Appendix \ref{app:bayesofpr}. Thus, this corollary again validates that TKPR is more discriminating than \texttt{P@K} and \texttt{R@K}.

\subsubsection{Consistency of the Surrogate Objective}
\label{subsec:consistency}
So far, we have known the Bayes optimal solution to TKPR optimization. On top of this, we can further present the following sufficient condition for TKPR consistency, which is much easier to check than the top-K ranking-preserving property. Please refer to Appendix \ref{app:suffcondicons} for the details.

\begin{restatable}{theorem}{conditionforconsistency}
    \label{thm:condition_for_consistency}
    The surrogate loss $\ell(t)$ is TKPR Fisher consistent if it is bounded, differentiable, strictly decreasing, and convex.
\end{restatable}

\begin{sketch}
    The key point is to show that if $\lnot \mathsf{RPT}(\boldsymbol{s}, \Delta)$, $\boldsymbol{s}$ will not be an optimal solution to $(OP_2)$. Given a prediction $\boldsymbol{s}$ and $i, j \in \mathcal{L}$, $\lnot \mathsf{RPT}(\boldsymbol{s}, \Delta)$ consists of three cases: (1) $\Delta_i = \Delta_j$ but $s_i \neq s_j$; (2) $\Delta_i \neq \Delta_j$ but $s_i = s_j$; (3) $\Delta_i < \Delta_j$ but $s_i > s_j$. We obtain the result in each case by a contradiction.
\end{sketch}

In Thm.\ref{thm:condition_for_consistency}, we have discussed the consistency \textit{w.r.t.} all  measurable functions. However, common surrogate losses are not bounded. To this end, we next restrict the functions within a special function set $\mathcal{F}$, which induces the concept of $\mathcal{F}$-consistency:
\begin{definition}[TKPR $\mathcal{F}$-consistency]
    The surrogate loss $\ell: \mathbb{R} \to \mathbb{R}_{+}$ is $\mathcal{F}$-consistent with TKPR if for any sequence $\{ f^{(t)} \}_{n=1}^{\infty}, f \in \mathcal{F}$,
    \begin{equation}
        \text{reg}(f^{(t)}; \mathcal{R}_{K}^{\alpha, \ell}) \to 0 \implies \text{reg}(f^{(t)}; \mathcal{R}_{K}^{\alpha}) \to 0,
    \end{equation}
\end{definition}
Then, we can find that common convex surrogate losses are all $\mathcal{F}$-consistent with TKPR:
\begin{corollary}
    \label{coll:consistent_loss}
    Let $\mathcal{F}$ denote the set of functions whose outputs are bounded in $[0, 1]$. Then, the surrogate loss $\ell(t)$ is $\mathcal{F}$-consistent with TKPR if it is differentiable, strictly decreasing, and convex in $[0, 1]$. Thus, we can conclude that the square loss $\ell_{sq}(t) = (1 - t)^2$, the exponential loss $\ell_{exp}(t) = \exp(-t)$, and the logit loss $\ell_{logit}(t) = \log (1 + \exp(-t))$ are all $\mathcal{F}$-consistent with TKPR.
\end{corollary}

As pointed out in \cite{DBLP:journals/ai/GaoZ13}, any convex surrogate losses are inconsistent with the ranking loss $L_\text{rank}$. Although \cite{DBLP:conf/icml/DembczynskiKH12} proposes a consistent surrogate, its generalization bound is not satisfactory \cite{DBLP:conf/nips/WuLXZ21}, as shown in Tab.\ref{tab:loss_comparison}. In the next part, we show that the ERM framework for multi-label also enjoys a sharp generalization bound on the ranking loss.

\subsection{Generalization Analysis of the ERM Framework}
\label{sec:generalization}

To tackle \textbf{(C2)}, we turn to optimize its empirical estimation based on the given dataset $\mathcal{S}=\{\boldsymbol{z}^{(n)}\}_{n=1}^{N}$ sampled \textit{i.i.d.} from $\mathcal{D}$. For the sake of convenience, let $s^{(n)}_{y}$ denote the score of the class $y \in \mathcal{L}$ on the $n$-th sample. Then, we have the following empirical optimization problem:
\begin{equation}
    \label{eq:op3}
    \begin{aligned}
        & (OP_3)\phantom{|} \min_{f \in \mathcal{F}} \ \hat{\mathcal{R}}_K^{\alpha, \ell}(f) := \frac{1}{N}\sum_{n \le N} L_{K}^{\alpha, \ell}(f(\boldsymbol{x}^{(n)}), \boldsymbol{y}^{(n)}) \\
        & = \frac{1}{NK}\sum_{n \le N} \sum_{y \in \mathcal{P}(\boldsymbol{y}^{(n)})} \sum_{k \le K+1} \frac{1}{\alpha} \cdot \ell \left( s_y^{(n)} - s_{[k]}^{(n)} \right).
    \end{aligned}
\end{equation}
In Sec.\ref{sec:consistency}, we know that optimizing $(OP_2)$ can recover the solution to $(OP_1)$. Then, 
\begin{qbox-tight}
    Can optimizing $(OP_3)$ approx. the solution to $(OP_2)$?
\end{qbox-tight} 
In other words, it requires that the model performance on $\mathcal{S}$ can generalize well to unknown data. In Sec.\ref{subsec:gen_traditional}, we first follow the techniques used in prior arts \cite{DBLP:conf/nips/WuZ20,DBLP:conf/nips/WuLXZ21} and present a coarse-grained result. To obtain a more fine-grained result, in Sec.\ref{subsec:sharper_bound}, we extend the definition of Lipschitz continuity and propose a novel contraction technique that relies on the data distribution. On top of this, the proposed ERM framework enjoys a sharper generalization bound under mild conditions. Finally, in Sec.\ref{subsec:prac_bound}, we present some practical results for kernel-based models and convolutional neural networks.

\subsubsection{Generalization Bounds with Traditional Techniques}
\label{subsec:gen_traditional}
In this part, our analysis is based on the traditional Lipschitz continuity property and the following assumption: 
\begin{definition}[Lipschitz continuity]
    \label{def:tra_lip}
    We say the loss function $L(f, \boldsymbol{x})$ is $\mu$-Lipschitz continuous, if $\forall f, f' \in \mathcal{F}, | L(f, \boldsymbol{y}) - L(f', \boldsymbol{y}) | \le \mu \Vert f(\boldsymbol{x}) - f'(\boldsymbol{x})\Vert$, where $\Vert \cdot \Vert$ denotes the 2-norm.
\end{definition}
\begin{assumption}
    \label{asm:gen}
    We assume that (1) the surrogate loss $\ell$ is $\mu_{\ell}$-Lipschitz continuous, has an upper bound $M_{\ell}$, and satisfies the conditions in Thm.\ref{thm:condition_for_consistency}; (2) the hyperparameter $K \ge \max_{\boldsymbol{z} \sim \mathcal{D}} N(\boldsymbol{y})$.
\end{assumption}
\begin{remark}
    According to Cor.\ref{coll:consistent_loss}, we should normalize the outputs of the surrogate loss $\ell$ with a bounded function such as $\mathsf{Softmax}$. Note that $\mathsf{Softmax}$ is $1 / \sqrt{2}$-Lipschitz continuous \cite{DBLP:journals/pami/YangXBCH22}, which will not affect the order of the generalization bound. Thus, we will omit $\mathsf{Softmax}$ for the sake of conciseness.
\end{remark}

Following the techniques in prior arts \cite{DBLP:conf/nips/WuZ20,DBLP:conf/nips/WuLXZ21}, we have the following lemma:

\begin{lemma}[Basic lemma for generalization analysis \cite{10.5555/2371238}]
    \label{lem:lem_gen_lin_1}
    Given the function set $\mathcal{F}$ and a loss function $L: \mathbb{R}^C \times \mathcal{Y} \to [0, M]$, let $\mathcal{G} = \{L \circ f: f \in \mathcal{F}\}$. Then, for any $\delta \in (0, 1)$, with probability at least $1 - \delta$ over the training set $\mathcal{S}$, the following generalization bound holds for all the $g \in \mathcal{G}$:
    \begin{equation}
        \E{\boldsymbol{z} \sim \mathcal{D}}{g(\boldsymbol{z})} \precsim \Phi(L, \delta) + \hat{\mathfrak{C}}_\mathcal{S}(\mathcal{G}), 
    \end{equation}
    where 
    \begin{equation}
        \Phi(L, \delta) := \frac{1}{N} \sum_{n=1}^{N} g(\boldsymbol{z}^{(n)}) + 3 M \sqrt{\frac{\log 2 / \delta}{2N}},
    \end{equation}
    consists of the empirical risk and a $\delta$-dependent term, 
    \begin{equation}
        \hat{\mathfrak{C}}_\mathcal{S}(\mathcal{G}) := \E{\boldsymbol{\xi}}{\sup_{f \in \mathcal{F}} \frac{1}{N}\sum_{n=1}^{N} \xi^{(n)} g(\boldsymbol{z}^{(n)}) },
    \end{equation}
    denotes an empirical complexity measure for the function set $\mathcal{G}$, $\boldsymbol{\xi} := ( \xi^{(1)}, \xi^{(2)}, \cdots, \xi^{(N)} )$ are the independent random variables for the complexity measure, and $\precsim$ is the asymptotic notation helps omit constants and undominated terms:
    $$
        f(t) \precsim g(t) \Longleftrightarrow  \exists \text{ a constant } C, f(t) \le C \cdot g(t).
    $$ 
\end{lemma}
\begin{remark}
    Different random variables $\boldsymbol{\xi}$ will induce different complexity measures. For example, given uniform random variables taking values from $\{-1, +1\}$, it becomes Rademacher complexity, denoted as $\hat{\mathfrak{R}}_\mathcal{S}(\mathcal{G})$. Given the standard normal distribution, it turns to Gaussian complexity, denoted as $\hat{\mathfrak{G}}_\mathcal{S}(\mathcal{G})$.
\end{remark}

According to Lem.\ref{lem:lem_gen_lin_1}, our task is to bound $\hat{\mathfrak{C}}_\mathcal{S}(\mathcal{G})$. To this end, the following contraction lemma can help us obtain the result directly:

\begin{lemma}[Vector Contraction Inequality \cite{DBLP:conf/alt/Maurer16}]
    \label{lem:vector_contraction}
    Assume that the loss function $L(f, \boldsymbol{x})$ is $\mu$-Lipschitz continuous. Then, the following inequality holds:
    \begin{equation}
        \hat{\mathfrak{C}}_\mathcal{S}(\mathcal{G}) \le \sqrt{2} \mu \hat{\mathfrak{C}}_\mathcal{S}(\mathcal{F}).
    \end{equation}
\end{lemma}

According to Lem.\ref{lem:vector_contraction}, we present the Lipschitz constant of the TKPR loss $L_{K}^{\alpha, \ell}$, whose proof can be found in Appendix \ref{app:r_tra_lipschitz}.
\begin{restatable}{proposition}{Rtralipschitz}
    \label{prop:r_tra_lipschitz}
    Under Asm.\ref{asm:gen}, the TKPR surrogate loss $L_{K}^{\alpha, \ell}$ is $\mu_\ell \mu_K$-Lipschitz continuous and bounded by $M_K$, where
    \begin{itemize}
        \item  $\mu_K = \frac{K + 1}{\sqrt{K}} + \sqrt{K + 1}, M_K = (K + 1) M_{\ell}$ when $\alpha = \alpha_1$;
        \item $\mu_K =  \frac{K + 1}{\sqrt{K}} + \frac{\sqrt{K + 1}}{K}, M_K = (K + 1) M_{\ell}$ when $\alpha = \alpha_2$;
        \item $\mu_K = \frac{K + 1}{K^2} + \frac{2}{K \sqrt{K + 1}}, M_K = \frac{K + 1}{K} M_{\ell}$ when $\alpha = \alpha_3$.
    \end{itemize}    
\end{restatable}

Finally, combining Lem.\ref{lem:lem_gen_lin_1}-\ref{lem:vector_contraction} and Prop.\ref{prop:r_tra_lipschitz}, we obtain the generalization bound of the proposed ERM framework, whose proof is presented in Appendix \ref{app:gen_existing}.
\begin{restatable}{proposition}{genexisting}
    \label{prop:gen_existing}
    Under Asm.\ref{asm:gen}, for any $\delta \in (0, 1)$, with probability at least $1 - \delta$ over the training set $\mathcal{S}$, the following generalization bound holds for all the $f \in \mathcal{F}$:
    $$\begin{aligned}
        {\color{blue} \mathcal{R}_{K}^{\alpha, \ell}}(f) & \precsim \Phi( {\color{blue} L_{K}^{\alpha, \ell}}, \delta) + \\
        & \left\{\begin{array}{ll}
            \mathcal{O}( \sqrt{K} ) \cdot \hat{\mathfrak{C}}_\mathcal{S}(\mathcal{F}), \ & \alpha \in \{ \alpha_1,  \alpha_2\}, \\
            \mathcal{O}( 1 / K ) \cdot \hat{\mathfrak{C}}_\mathcal{S}(\mathcal{F}), \ & \alpha = \alpha_3. \\
        \end{array}\right.
    \end{aligned}$$
\end{restatable}

Furthermore, combining Prop.\ref{prop:reformulation}, we show that optimizing TKPR can provide a sharp generalization bound for the ranking loss, whose proof can be found in Appendix \ref{app:sharp_bound}.
\begin{restatable}{proposition}{sharpbound}
    \label{prop:sharp_bound}
    Under Asm.\ref{asm:gen}, let 
    \begin{equation}
        \mathcal{R}_\text{rank}^{\ell}(f) :=  \E{(\boldsymbol{x}, \boldsymbol{y}) \sim \mathcal{D}} {L_\text{rank}^{\ell}(f, \boldsymbol{y})}
    \end{equation}
    denote the generalization error of the traditional ranking loss. Then, for any $\delta \in (0, 1)$, with probability at least $1 - \delta$ over the training set $\mathcal{S}$, the following generalization bound holds for all the $f \in \mathcal{F}$:
    $$\begin{aligned}
        & {\color{orange} \mathcal{R}_\text{rank}^{\ell}}(f) \\
        & \precsim \left\{\begin{array}{ll}
            \Phi({\color{blue} L_{K}^{\alpha, \ell}}, \delta) + \mathcal{O}( \sqrt{K} ) \cdot \hat{\mathfrak{C}}_\mathcal{S}(\mathcal{F}), \ & \alpha \in \{ \alpha_1,  \alpha_2\}, \\
            K \cdot \Phi({\color{blue} L_{K}^{\alpha, \ell}}, \delta) + \mathcal{O}( 1 ) \cdot \hat{\mathfrak{C}}_\mathcal{S}(\mathcal{F}), \ & \alpha = \alpha_3. \\
        \end{array}\right.
    \end{aligned}$$
\end{restatable}

\begin{table*}[t]
    \renewcommand{\arraystretch}{2}
    \caption{Systematic comparison between the TKPR loss and the ranking loss, as well as its pointwise surrogates. For \textit{Generalization}, we assume that $\pi_i \propto e^{- \lambda i}$, and more details can be found in Prop.\ref{prop:final_bound_ranking}. For \textit{Consist.}, $\surd$ and $\times$ mean that a convex surrogate loss can be consistent with the original objective or not, respectively. \textit{Complexity} represents the time complexity for each sample. And for \textit{MLML}, $\surd$ and $\times$ mean that the loss is applicable to MLML or not, respectively.}
    \label{tab:loss_comparison}
    \centering
    \begin{tabular}{ccccc}
        \toprule
        Loss & Generalization & Consist. & Complexity & MLML \\
        \midrule
        $L_{\text{rank}}$ \cite{DBLP:journals/ai/GaoZ13} & $\mathcal{O}( \sqrt{\frac{C}{N}} )$ & $\times$ & $\mathcal{O}(C^2)$ & $\times$ \\
        $L_{u_1}$ \cite{DBLP:conf/icml/DembczynskiKH12} & $\mathcal{O}(\sqrt{\frac{C^2}{N}})$ & $\times$ & $\mathcal{O}(C)$ & $\times$ \\
        $L_{u_2}$ \cite{DBLP:conf/icml/DembczynskiKH12} & $\mathcal{O}(\sqrt{\frac{C^2}{N}})$ & $\surd $ & $\mathcal{O}(C)$ & $\times$ \\
        $L_{u_3}$ \cite{DBLP:conf/nips/WuLXZ21} & $\mathcal{O}(\sqrt{\frac{C}{N}})$ & $\times$ & $\mathcal{O}(C)$ & $\times$ \\
        $L_{u_4}$ \cite{DBLP:conf/nips/WuLXZ21} & $\mathcal{O}(\sqrt{\frac{C^2}{N}})$ & $\times$ & $\mathcal{O}(C)$ & $\times$ \\
        \midrule
        $L_K^{\alpha, \ell}$ (Ours) & $\mathcal{O}(\sqrt{\frac{C}{N e^{\lambda}}})$ & $\surd$ & $\mathcal{O}(C K)$ & $\surd$ \\
        \bottomrule
    \end{tabular}
\end{table*}

\subsubsection{Sharper Bounds with Data-dependent Contraction}
\label{subsec:sharper_bound}
Although Prop.\ref{prop:gen_existing} and Prop.\ref{prop:sharp_bound} have provided sharp generalization bounds than prior arts \cite{DBLP:conf/icml/DembczynskiKH12,DBLP:conf/nips/WuLXZ21}, the results under $\alpha_1$ and $\alpha_2$ suffer an order of $\sqrt{K}$, which is unfavorable in the scenarios requiring a large $K$. After rethinking the proofs, we find that the root cause lies in the simple relaxation of the term $\frac{1}{N(\boldsymbol{y})}$ to $1$. Note that the distribution of relevant labels is generally imbalanced, with only a small subset of instances having a large number of labels, and the majority of instances only having a few labels. This insight motivates us to  extend the traditional Lipschitz continuity property:

\begin{definition}[Local Lipschitz continuity]
    Let $\{\mathcal{S}_q\}_{q = 1}^Q$ be a partition of $\mathcal{S}$. We say the loss function $L(f, \boldsymbol{y})$ is local Lipschitz continuous with the partition $\{\mathcal{S}_q\}_{q = 1}^Q$ and constants $\{\mu_q\}_{q = 1}^Q$ if for any $f, f' \in \mathcal{F}, q \in \{1, 2, \cdots, Q\}$, 
    \begin{equation}
        | L(f, \boldsymbol{y}) - L(f', \boldsymbol{y}) | \le \mu_q \cdot \Vert f(\boldsymbol{x}) - f'(\boldsymbol{x})\Vert, (\boldsymbol{x}, \boldsymbol{y}) \in \mathcal{S}_q.
    \end{equation}
\end{definition}
\noindent Then, the following data-dependent contraction inequality helps us obtain a sharper bound under the following assumption, whose proof can be found in Appendix \ref{app:data_contraction}.

\begin{assumption}
    \label{asm:rademacher}
    Next, we assume that $\hat{\mathfrak{C}}_\mathcal{S}(\mathcal{F}) \sim \mathcal{O}( 1 / \sqrt{N} )$. Note that this result holds for kernel-based models with traditional techniques \cite{DBLP:conf/nips/WuZ20,DBLP:conf/nips/WuLXZ21} and neural networks with lastest techniques \cite{DBLP:conf/colt/GolowichRS18,DBLP:conf/iclr/LongS20}.
\end{assumption}

\begin{restatable}[Data-dependent contraction inequality]{proposition}{labelcontraction}
    \label{prop:data_contraction}
    Under Asm.\ref{asm:rademacher}, if the loss function $L(f, \boldsymbol{y})$ is local Lipschitz continuous with a partition $\mathcal{S}_Q$ and constants $\{\mu_q\}_{q = 1}^Q$. Let $\pi_q := \frac{N_q}{N}$ be the ratio of $\mathcal{S}_q$ in $\mathcal{S}$, where $N_q = \size{\mathcal{S}_q}$. Then, 
    \begin{equation}
        \label{eq:data_c1}
        \hat{\mathfrak{C}}_\mathcal{S}(\mathcal{G}) \le \hat{\mathfrak{C}}_{\mathcal{S}}(\mathcal{F}) \sum_{q=1}^{Q} \sqrt{\pi_q} \mu_q.
    \end{equation}  
\end{restatable}
\begin{remark}
    Eq.(\ref{eq:data_c1}) is favorable when $\pi_q$ is decreasing \textit{w.r.t.} $\mu_q$. In this case, a sharper generalization bound might be available. However, if local Lipschitz continuity degenerates to Def.\ref{def:tra_lip}, this inequality becomes a little loose since $\sum_{q=1}^Q \sqrt{\pi_q} > 1$.
\end{remark}

Similar to Prop.\ref{prop:r_tra_lipschitz}, we partition the dataset $\mathcal{S}$ and calculate the Lipschitz constants for $L_{K}^{\alpha, \ell}(f, \boldsymbol{y})$ as follows, whose proof can be found in Appendix \ref{app:local_continuous}.
\begin{restatable}{proposition}{Localcontinuous}
    \label{prop:local_continuous}
    Let $\mathcal{S}_q := \{\boldsymbol{z} \in \mathcal{S}: N(\boldsymbol{y}) = q \}$. That is, all the samples in $\mathcal{S}_q$ have $q$ relevant labels. Then, under Asm.\ref{asm:gen}, $L_{K}^{\alpha, \ell}(f, \boldsymbol{y}$) is local Lipschitz continuous with constants $\{\mu_q\}_{q = 1}^Q$ such that 
    \begin{equation}
        \mu_q = \frac{\mu_{\ell} \left[ (K + 1) \sqrt{q} + q \sqrt{K + 1} \right]}{\alpha(q) K},
    \end{equation}
    where $\alpha(q) \in \{1, q, q (2K + 1 - q) / 2 \}$ and $q \le K$.
\end{restatable}

Combining Lem.\ref{lem:lem_gen_lin_1} and Prop.\ref{prop:local_continuous}, we obtain the following generalization bound of TKPR optimization, whose proof can be found in Appendix \ref{app:final_bound}.

\begin{table}[t]
    \renewcommand{\arraystretch}{1.4}
    \caption{The concrete formulations of $\mathfrak{g}(K)$ under an exponential distribution and a multinomial distribution, parameterized by $\lambda$.}
    \label{table:fin_bound}
    \centering
    \begin{tabular}{c|c|cc}
        \toprule
        $\alpha(q)$ & $\pi_q \propto e^{- \lambda q}$ & $\lambda$ & $\pi_q \propto q^{- \lambda}$ \\
        \midrule
        \multirow{3}*{1} & \multirow{3}*{$ \frac{1}{\lambda^2} $} & $(0, 3)$ & $ K^{(3-\lambda) / 2} $ \\
        ~ & ~ & $[3, 5)$ & $ \ln K $ \\
        ~ & ~ & $[5, \infty]$ & $ 1 $ \\
        \midrule
        \multirow{3}*{$q$} & \multirow{3}*{$ \frac{1}{e^{\lambda / 2}} $} & $(0, 1)$ & $  K^{(1-\lambda) / 2} $ \\
        ~ & ~ & $[1, 3)$ & $ \ln K $ \\
        ~ & ~ & $[3, \infty]$ & $ 1 $ \\
        \midrule
        \multirow{3}*{$\frac{q (2K + 1 - q)}{2}$} & \multirow{3}*{$ \frac{1}{K e^{\lambda / 2}} $} & $(0, 1)$ & $  K^{- (1 + \lambda) / 2} $ \\
        ~ & ~ & $[1, 3)$ & $ K^{-1} \ln K $ \\
        ~ & ~ & $[3, \infty]$ & $ K^{-1} $ \\
        \midrule
    \end{tabular}
\end{table}

\begin{restatable}{theorem}{finalbound}
    \label{thm:final_bound}
    Under Asm.\ref{asm:gen} and Asm.\ref{asm:rademacher}, for any $\delta \in (0, 1)$, with probability at least $1 - \delta$ over the training set $\mathcal{S}$, the following generalization bound holds for all $f \in \mathcal{F}$:
    \begin{equation}
        \label{eq:final_bound_abstract}
        {\color{blue} \mathcal{R}_{K}^{\alpha, \ell}}(f) \precsim \Phi( {\color{blue} L_{K}^{\alpha, \ell}}, \delta) + \hat{\mathfrak{C}}_{\mathcal{S}}(\mathcal{F}) \mathcal{O}\left( \mathfrak{g}(K) \right),
    \end{equation}
    where $\mathfrak{g}(K)$ relies on the distribution of $\pi_q$. In Tab.\ref{table:fin_bound}, we present the results under an exponential distribution and a multinomial distribution, parameterized by $\lambda$.
\end{restatable}

\begin{sketch}
    When $\pi_q$ follows an exponential distribution, we relax the bound with the definite integral from $1$ to $K$. When $\pi_q$ follows a multinomial distribution, we relax the bound with Riemann zeta function \cite{titchmarsh1986theory}. Note that the order of Riemann zeta function is out of the scope of this paper. Thus, we only provide a coarse-grained result, and more fine-grained results can be found in \cite{fokas2022asymptotics}. 
\end{sketch}

Compared with the results in Prop.\ref{prop:gen_existing}, we have the following observations:
\begin{itemize}
    \item \ul{When $\pi_q \propto e^{- \lambda q}$}, we generally obtain a sharper bound, where $\sqrt{K}$ is replaced with $\lambda^{-2}$ and $e^{-\lambda / 2}$. This property is appealing since $\lambda$ is independent of the selection of the hyperparameter $K$, and the bound will become sharper as the distribution becomes more imbalanced.
    \item \ul{When $\pi_q \propto q^{- \lambda}$}, if $\alpha(q) = 1$, a sharper bound is available when $\lambda > 2$. If $\alpha(q) = q$, a sharper bound is consistently available under any $\lambda > 0$. As $\lambda$ increases, \textit{i.e.}, the distribution becomes more imbalanced, that is, the generalization bound will become sharper.
    \item \ul{For $\alpha(q) = q (2K + 1 - q) / 2$}, the data-dependent contraction technique fails to provide a sharper bound \textit{w.r.t.} $K$. However, when $\pi_q \propto e^{- \lambda q}$, the result becomes more informative due to the additional term $e^{-\lambda / 2}$. 
\end{itemize}

Similarly, we can obtain a sharper generalization bound for the ranking loss, whose proof can be found in Appendix \ref{app:final_bound_ranking}.

\begin{restatable}{proposition}{finalboundranking}
    \label{prop:final_bound_ranking}
    Let $\tilde{L}_{K}^{\ell}(f, \boldsymbol{y}) := \frac{\alpha}{N(\boldsymbol{y})} \cdot L_{K}^{\alpha, \ell}(f, \boldsymbol{y})$. Then, under Asm.\ref{asm:gen} and Asm.\ref{asm:rademacher}, for any $\delta \in (0, 1)$, with probability at least $1 - \delta$ over the training set $\mathcal{S}$, the following generalization bound holds for all the $f \in \mathcal{F}$:
    \begin{itemize}
        \item When $\pi_q \propto e^{- \lambda q}$, 
        $$
            {\color{orange} \mathcal{R}_\text{rank}^{\ell}(f)} \precsim \Phi( {\color{blue} \tilde{L}_{K}^{\ell}}, \delta) + \mathcal{O}( e^{-\lambda / 2} ) \cdot \hat{\mathfrak{C}}_\mathcal{S}(\mathcal{F}).
        $$
        \item When $\pi_q \propto q^{- \lambda}$,
        \renewcommand{\arraystretch}{1.5}
        $$\begin{aligned}
            {\color{orange} \mathcal{R}_\text{rank}^{\ell}(f)} & \precsim \Phi( {\color{blue} \tilde{L}_{K}^{\ell}}, \delta) \\
            & + \left\{\begin{array}{ll}
                \mathcal{O}( K^{(1 - \lambda) / 2} ) \cdot \hat{\mathfrak{C}}_\mathcal{S}(\mathcal{F}), \ & \lambda \in (0, 1), \\
                \mathcal{O}( \ln K) \cdot \hat{\mathfrak{C}}_\mathcal{S}(\mathcal{F}), \ & \lambda \in [1, 3), \\
                \mathcal{O}( 1 ) \cdot \hat{\mathfrak{C}}_\mathcal{S}(\mathcal{F}), \ & \lambda \in [3, \infty). \\
            \end{array}\right.
        \end{aligned}$$
    \end{itemize}
\end{restatable}

Compared with the results in Prop.\ref{prop:sharp_bound}, we have the following observations:
\begin{itemize}
    \item \ul{When $\pi_q \propto e^{- \lambda q}$}, a sharper bound is consistently available for any choice of $\alpha$, and the bound will become sharper as the distribution becomes more imbalanced.
    \item \ul{When $\pi_q \propto q^{- \lambda}$}, a sharper bound is available for $\alpha \in \{\alpha_1, \alpha_2\}$. However, it does not hold for $\alpha = \alpha_3$.
    \item \ul{For $\alpha(q) = q (2K + 1 - q) / 2$}, $\Phi(\tilde{L}_{K}^{\ell}, \delta)$ is smaller than $K \cdot \Phi(L_{K}^{\alpha, \ell}, \delta)$, which is also favorable.
\end{itemize}

\subsubsection{Practical Bounds for Common Models}
\label{subsec:prac_bound}
Next, we show that the proposed contraction technique is applicable to common models, such as kernel-based models and neural networks. Note that we omit the cases where $\pi_q$ follows the multinomial distribution for the sake of conciseness.

\textbf{Practical Bounds for Kernel-Based Models.} Let $\mathbb{H}$ be a reproducing kernel Hilbert space (RKHS) with the kernel function $\kappa$, where $\kappa: \mathcal{X} \times \mathcal{X} \to \mathbb{R}$ is a Positive Definite Symmetric (PSD) kernel. The set of kernel-based models can be defined as:
\begin{equation}
    \mathcal{F}_{\mathbb{H}} := \left\{ \boldsymbol{x} \to \mathbf{W}^T \phi(\boldsymbol{x}): \Vert \mathbf{W} \Vert_{\mathbb{H}, 2} \le \Lambda \right\},
\end{equation}
where $\phi: \mathcal{X} \to \mathbb{H}$ is a feature mapping associated with $\kappa$, $\mathbf{W} = (\boldsymbol{w}_1, \boldsymbol{w}_2, \cdots, \boldsymbol{w}_C)^T$ represents the model parameters, and $\Vert \mathbf{W} \Vert_{\mathbb{H}, 2} := (\sum_{j=1}^C \Vert \boldsymbol{w}_j \Vert_{\mathbb{H}}^2)^{1 / 2}$. Assume that $\exists r > 0$ such that $\kappa(\boldsymbol{x}, \boldsymbol{x}) \le r^2$ for all $\boldsymbol{x} \in \mathcal{X}$. Then, we have the following propositions, whose proof can be found in Appendix \ref{app:bounds_kernel}:
\begin{restatable}{proposition}{genkernel}
    \label{prop:gen_kernel}
    Under Asm.\ref{asm:gen} and Asm.\ref{asm:rademacher}, for any $\delta \in (0, 1)$, with probability at least $1 - \delta$ over the training set $\mathcal{S}$, the following generalization bound holds for all the $f \in \mathcal{F}_{\mathbb{H}}$:
    \renewcommand{\arraystretch}{2}
    $$
        {\color{blue} \mathcal{R}_{K}^{\alpha, \ell}}(f) \precsim \Phi( {\color{blue} L_{K}^{\alpha, \ell}}, \delta) + \left\{\begin{array}{ll}
            \mathcal{O}( \sqrt{\frac{C \Lambda^2 r^2}{N \lambda^4}} ) , \ & \alpha = \alpha_1 \\
            \mathcal{O}( \sqrt{\frac{C \Lambda^2 r^2}{N e^\lambda}} ) , \ & \alpha = \alpha_2 \\
            \mathcal{O}( \sqrt{\frac{C \Lambda^2 r^2}{N K^2 e^\lambda}} ), \ & \alpha = \alpha_3. \\
        \end{array}\right.
    $$
\end{restatable}

\begin{restatable}{proposition}{genkernelranking}
    \label{prop:gen_kernel_ranking}
    Under Asm.\ref{asm:gen} and Asm.\ref{asm:rademacher}, for any $\delta \in (0, 1)$, with probability at least $1 - \delta$ over the training set $\mathcal{S}$, the following generalization bound holds for all the $f \in \mathcal{F}_{\mathbb{H}}$:
    $$
        {\color{orange} \mathcal{R}_\text{rank}^{\ell}(f)} \precsim \Phi( {\color{blue} \tilde{L}_{K}^{\ell}}, \delta) + \mathcal{O}( \sqrt{\frac{C \Lambda^2 r^2}{N e^\lambda}} ).
    $$
\end{restatable}

In Tab.\ref{tab:loss_comparison}, we compare the TKPR loss with the ranking loss systematically, as well as its pointwise surrogates. The results show that only the TKPR loss has both convex consistent surrogate losses and a sharp generalization bound on the ranking loss. Furthermore, the TKPR loss also enjoys a comparable computational complexity and more wider application scenarios. We will validate these observations in Sec.\ref{sec:experiment}.

\textbf{Practical Bounds for CNNs.} Next, we consider a family of neural networks, which consists of $N_c$ convolutional layers and $N_f$ fully-connected layers. To be specific, in each convolutional layer, a convolution operation is followed by an activation function and an optional pooling operation. All the convolutions utilize zero-padding \cite{deeplearning} with the kernel ${KN}^{(l)}$ for layer $l \in \{1, \cdots, N_c\}$. In each fully-connected layer, a fully-connected operation, parameterized by $V^{(l)}$, is followed by an activation function. All the activation functions and pooling operations are $1$-Lipschitz continuous. Finally, Let $\Theta = \{{KN}^{(1)}, \cdots, {KN}^{(N_c)}, V^{(1)}, \cdots, V^{(N_f)} \}$ represent the parameter set of the networks.

Meanwhile, we assume that the inputs and parameters are all regularized. Concretely, the input $\boldsymbol{x} \in \mathbb{R}^{d\times d\times c}$ satisfies $||\mathsf{vec}(\boldsymbol{x})|| \le \chi$, where $\mathsf{vec}(\cdot)$ denotes the vectorization operation defined on $\mathcal{X}$. The initial parameters, denoted as $\Theta_0$, satisfy
$$\begin{aligned}
    ||\mathsf{mt}({KN}^{(l)}_0)||_2 \le 1 + \nu, l = 1, \cdots, N_c, \\
    ||V^{(l)}_0||_2 \le 1 + \nu, l = 1, \cdots, N_f,  \\
\end{aligned}$$
where $\mathsf{mt}(\cdot)$ denotes the operator matrix of the given kernel;. And the distance from $\Theta_0$ to the current parameters $\Theta$ is bounded:
$$\begin{aligned}
    & \beta \ge ||\Theta - \Theta_0||  := \\
    & \sum_{l=1}^{N_c} || \mathsf{mt}({KN}^{(l)}) - \mathsf{mt}({KN}_0^{(l)}) ||_2 + \sum_{l=1}^{N_f} || V^{(l)} - V_0^{(l)} ||_2 \\
\end{aligned}$$

Let $\mathcal{F}_{\beta, \nu, \chi}$ denote the set of convolutional neural networks described above. Then, we have the following propositions, whose proof can be found in Appendix \ref{app:bounds_cnn}.

\begin{restatable}{proposition}{gencnn}
    \label{prop:gen_cnn}
    Under Asm.\ref{asm:gen} and Asm.\ref{asm:rademacher}, for any $\delta \in (0, 1)$, with probability at least $1 - \delta$ over the training set $\mathcal{S}$, the following generalization bound holds for all the $f \in \mathcal{F}_{\beta, \nu, \chi}$:
    \renewcommand{\arraystretch}{2}
    $$\begin{aligned}
        {\color{blue} \mathcal{R}_{K}^{\alpha, \ell}}(f) \precsim \Phi( {\color{blue} L_{K}^{\alpha, \ell}}, \delta) + \left\{\begin{array}{ll}
            \mathcal{O}(  \frac{ d \log \left( B_{\beta, \nu, \chi}N \right) }{\sqrt{N} \lambda^2} ) , \ & \alpha = \alpha_1, \\
            \mathcal{O}(  \frac{ d \log \left( B_{\beta, \nu, \chi}N \right) }{\sqrt{N e^\lambda}} ) , \ & \alpha = \alpha_2, \\
            \mathcal{O}(  \frac{ d \log \left( B_{\beta, \nu, \chi}N \right) }{\sqrt{N e^\lambda} K} ), \ & \alpha = \alpha_3, \\
        \end{array}\right.
    \end{aligned}$$
    where $B_{\beta, \nu, \chi} := \chi \beta (1 + \nu + \beta / N_a)^{N_a}$, $N_a := N_c + N_f$.
\end{restatable}

\section{Experiment}
\label{sec:experiment}
  In this section, we perform a series of experiments on benchmark datasets to validate the effectiveness of the proposed framework and the theoretical results. The induced learning algorithm, which contains a warm-up strategy, is summarized in Appendix \ref{app:more_imp_details}.

\begin{table*}[htbp] \small
    \centering
    \caption{Empirical results of time complexity, where the seconds are averaged over 200 trials.}
    \begin{tabular}{l|cccc|cccc}
    \toprule
    \multicolumn{1}{c}{\multirow{2}[4]{*}{}} & Forward & Loss  & Backward & \multicolumn{1}{c}{Total} &  Forward & Loss  & Backward & Total \\
    \cmidrule{2-9}    \multicolumn{1}{c}{} & \multicolumn{4}{c}{ResNet50, C=100} & \multicolumn{4}{c}{Swin-transformer, C=100} \\
    \midrule
    torch.ml & 0.0170  & 0.0001  & 0.0334  & 0.0504  & 0.1506  & 0.0003  & 0.3077  & 0.4586  \\
    $L_{\text{rank}}$ & 0.0166  & 0.0066  & 0.0377  & 0.0609  & 0.1515  & 0.0067  & 0.3136  & 0.4718  \\
    $L_{u_1}$ & 0.0170  & 0.0002  & 0.0335  & 0.0508  & 0.1510  & 0.0004  & 0.3087  & 0.4601  \\
    $L_{u_2}$ & 0.0170  & 0.0002  & 0.0336  & 0.0508  & 0.1507  & 0.0004  & 0.3078  & 0.4589  \\
    $L_{u_3}$ & 0.0170  & 0.0003  & 0.0336  & 0.0510  & 0.1512  & 0.0005  & 0.3092  & 0.4609  \\
    $L_{u_4}$ & 0.0171  & 0.0002  & 0.0336  & 0.0509  & 0.1519  & 0.0013  & 0.3099  & 0.4631  \\
    \midrule
    $L_K^{\alpha, \ell}$ (Ours) & 0.0167  & 0.0062  & 0.0369  & 0.0598  & 0.1512  & 0.0063  & 0.3125  & 0.4700  \\
    \midrule
    \multicolumn{1}{r}{} & \multicolumn{4}{c}{ResNet50, C=1,000} & \multicolumn{4}{c}{Swin-transformer, C=1,000} \\
    \midrule
    torch.ml & 0.0169  & 0.0013  & 0.0351  & 0.0533  & 0.1506  & 0.0014  & 0.3092  & 0.4611  \\
    $L_{\text{rank}}$ & 0.0166  & 0.0067  & 0.0410  & 0.0643  & 0.1516  & 0.0068  & 0.3157  & 0.4741  \\
    $L_{u_1}$ & 0.0170  & 0.0002  & 0.0336  & 0.0508  & 0.1507  & 0.0004  & 0.3081  & 0.4592  \\
    $L_{u_2}$ & 0.0171  & 0.0002  & 0.0336  & 0.0510  & 0.1510  & 0.0005  & 0.3088  & 0.4602  \\
    $L_{u_3}$ & 0.0171  & 0.0003  & 0.0337  & 0.0511  & 0.1516  & 0.0005  & 0.3099  & 0.4619  \\
    $L_{u_4}$ & 0.0171  & 0.0002  & 0.0336  & 0.0510  & 0.1509  & 0.0005  & 0.3086  & 0.4600  \\
    \midrule
    $L_K^{\alpha, \ell}$ (Ours) & 0.0167  & 0.0062  & 0.0369  & 0.0598  & 0.1505  & 0.0063  & 0.3102  & 0.4671  \\
    \midrule
    \multicolumn{1}{r}{} & \multicolumn{4}{c}{ResNet50, C=5,000} & \multicolumn{4}{c}{Swin-transformer, C=5,000} \\
    \midrule
    torch.ml & 0.0165  & 0.0256  & 0.0706  & 0.1127  & 0.1502  & 0.0270  & 0.3408  & 0.5180  \\
    $L_{\text{rank}}$ & 0.0165  & 0.0103  & 0.0894  & 0.1162  & 0.1509  & 0.0109  & 0.3607  & 0.5226  \\
    $L_{u_1}$ & 0.0170  & 0.0002  & 0.0336  & 0.0509  & 0.1513  & 0.0004  & 0.3083  & 0.4600  \\
    $L_{u_2}$ & 0.0171  & 0.0002  & 0.0337  & 0.0511  & 0.1514  & 0.0004  & 0.3088  & 0.4607  \\
    $L_{u_3}$ & 0.0171  & 0.0003  & 0.0338  & 0.0512  & 0.1517  & 0.0005  & 0.3093  & 0.4615  \\
    $L_{u_4}$ & 0.0171  & 0.0002  & 0.0338  & 0.0512  & 0.1518  & 0.0004  & 0.3095  & 0.4617  \\
    \midrule
    $L_K^{\alpha, \ell}$ (Ours) & 0.0167  & 0.0064  & 0.0378  & 0.0609  & 0.1515  & 0.0067  & 0.3122  & 0.4704  \\
    \midrule
    \multicolumn{1}{r}{} & \multicolumn{4}{c}{ResNet50, C=10,000} & \multicolumn{4}{c}{Swin-transformer, C=10,000} \\
    \midrule
    torch.ml & 0.0164  & 0.0966  & 0.1777  & 0.2907  & 0.1496  & 0.0989  & 0.4388  & 0.6873  \\
    $L_{\text{rank}}$ & 0.0165  & 0.0254  & 0.1887  & 0.2306  & 0.1505  & 0.0261  & 0.4623  & 0.6389  \\
    $L_{u_1}$ & 0.0171  & 0.0002  & 0.0337  & 0.0510  & 0.1507  & 0.0004  & 0.3073  & 0.4584  \\
    $L_{u_2}$ & 0.0171  & 0.0002  & 0.0338  & 0.0512  & 0.1514  & 0.0005  & 0.3090  & 0.4608  \\
    $L_{u_3}$ & 0.0172  & 0.0003  & 0.0340  & 0.0514  & 0.1518  & 0.0005  & 0.3098  & 0.4621  \\
    $L_{u_4}$ & 0.0172  & 0.0002  & 0.0339  & 0.0513  & 0.1511  & 0.0005  & 0.3083  & 0.4598  \\
    \midrule
    $L_K^{\alpha, \ell}$ (Ours) & 0.0168  & 0.0070  & 0.0394  & 0.0632  & 0.1508  & 0.0074  & 0.3130  & 0.4712  \\
    \bottomrule
    \end{tabular}%
  \label{tab:time_test}%
\end{table*}%

  \subsection{Efficiency Validation}
  \label{subsec:experiment_efficiency}
  In this part, we aim to validate this argument, we conduct an additional experiment, and such a superiority is also observed. Specifically, we use ResNet50 and swin-transformer as the backbone. For ResNet50, we set the input size as $448 \times 448$, the batch size as $16$, and $K$ as $15$. For swin-transformer, we adjust the input size to $384 \times 384$. The number of classes $C$ is set as $\{100, 1000, 5000, 10000\}$. To exclude the impact of implementation, we also report the results of the official pytorch loss \texttt{torch.nn.MultiLabelMarginLoss}, whose complexity is also $\mathcal{O}(C^2)$, denoted by \texttt{torch.ml}. We run 300 trials and report the results averaged over the latter 200 ones. All the experiments are conduct on an Nvidia(R) A100 GPU with a fixed random seed and a synchronized setup. Tab.\ref{tab:time_test} presents the time of forward-propagation, loss computation, and back-propagation, from which we have the following observations: 
  \begin{itemize}
    \item The complexity has little effect on the time of forward-propagation but has a significant impact on the loss computation and back-propagation. 
    \item For small models such as ResNet50, the effect of complexity is significant even with small $C$.
    \item For large models such as swin-transformer, the effect is not so significant until $C$ is large enough. 
    \item The computation time of pointwise surrogates, \textit{i.e.}, $L_{u_1}, L_{u_2}, L_{u_3}, L_{u_4}$ is insensitive to $C$. The time with $C = 10,000$ is almost the same as that with $C = 100$.
    \item The proposed TKPR loss $L_K^{\alpha, \ell}$, which enjoys a complexity of $\mathcal{O}(CK)$, achieves comparable results with the pointwise surrogates.
    \item The ranking loss $L_{\text{rank}}$ and the official implementation \texttt{torch.ml} suffer from the high complexity. Their time of loss computation and back-propagation scales fast as $C$ increases, becoming a significant part of the overall training time. 
    \item Our implementation is efficient since the ranking loss $L_{\text{rank}}$ has comparable or better results than the official implementation \texttt{torch.ml}.
  \end{itemize}

\renewcommand{\thefootnote}{\fnsymbol{footnote}}
\begin{table*}[t]
  \centering
    \caption{The empirical results of the ranking-bases losses and TKPR on MS-COCO, where the backbone is ResNet101. The best and runner-up results on each metric are marked with {\color{Top1}red} and {\color{Top2}blue}, respectively. The best competitor on each measure is marked with \underline{underline}.}
    \label{tab:ranking_loss_coco_MLC}%
\renewcommand\arraystretch{1.5}
  \tiny 
  \newcommand{\tabincell}[2]{\begin{tabular}{@{}#1@{}}#2\end{tabular}}
  \begin{tabular}{c|c|cc|cc|cc|cc|cc|cc|cc|c}
    \multicolumn{1}{c|}{\multirow{2}[4]{*}{Type}} & Metrics & \multicolumn{2}{c|}{P@K} & \multicolumn{2}{c|}{R@K} & \multicolumn{2}{c|}{mAP@K} & \multicolumn{2}{c|}{NDCG@K} & \multicolumn{2}{c|}{$\text{TKPR}^{\alpha_1}$} & \multicolumn{2}{c|}{$\text{TKPR}^{\alpha_2}$} & \multicolumn{2}{c|}{$\text{TKPR}^{\alpha_3}$} & \multicolumn{1}{c}{\multirow{2}[4]{*}{\tabincell{c}{Ranking \\ Loss}}} \\
\cmidrule{2-16}          & K & 3     & 5     & 3     & 5     & 3     & 5     & 3     & 5     & 3     & 5     & 3     & 5     & 3     & 5     &  \\
    \toprule
    \multicolumn{1}{c|}{\multirow{7}[2]{*}{\tabincell{c}{Ranking \\ Loss}}} & $L_{\text{rank}}$ & .540 & .424 & .673 & .821 & .473 & .514	 & .687 & .734 & 1.143 & 1.492 & .533 & .618 & .226 & .148 & .034 \\
          & $L_{u_1}$ & .571  & .424  & .705  & .817  & .609  & .621  & .706  & .730  & 1.204 & 1.538 & .549  & .626  & .234  & .150  & .037 \\
          & $L_{u_2}$ & .535  & .419  & .668  & .814  & .473  & .512  & .680  & .726  & 1.135 & 1.478 & .527  & .611  & .223  & .146  & .035 \\
          & $L_{u_3}$ & .536 & .414 & .663 & .804 & .511 & .542 & .676 & .717 & 1.138 & 1.472 & .524 & .605 & .223 & .145 & .033\\
          & $L_{u_4}$ & \underline{\textcolor[rgb]{ .886,  .42,  .039}{.615}} & \underline{\textcolor[rgb]{ .886,  .42,  .039}{.443}} & \underline{\textcolor[rgb]{ .886,  .42,  .039}{.753}} & \underline{\textcolor[rgb]{ .886,  .42,  .039}{.846}} & \underline{.663}  & \underline{.665}  & \underline{.753}  & \underline{.765}  & \underline{1.288} & \underline{1.631} & \underline{.591}  & \underline{.664}  & \underline{.251}  & \underline{.159}  & \underline{.032} \\
          & $L_{\text{LSEP}}$ & .522  & .420  & .648  & .816  & .419  & .469  & .650  & .711  & 1.078 & 1.444 & .498  & .594  & .211  & .142  & .045 \\
          & $L_{\text{TKML}}$ & .522  &	.393  &	.654  &	.775  &	.540  &	.554  &	.667  &	.697  &	1.135  &	1.435  &	.521  &	.594  &	.222  &	.141  &	.045  \\

    \toprule
    \multicolumn{1}{c|}{\multirow{3}[2]{*}{\tabincell{c}{TKPR \\ (Ours)}}} & $\alpha_1$   & .578  &	.423  &	.712  &	.816  &	{\color{Top2} .730}  &	{\color{Top2} .733}  &	{\color{Top2} .795}  &	{\color{Top2} .811}  &	{\color{Top2} 1.319}  &	{\color{Top2} 1.605}  &	{\color{Top2} .629}  &	{\color{Top2} .672}  &	{\color{Top2} .264}  &	{\color{Top2} .160}  &	.029 \\
          & $\alpha_2$   & {\color{Top2} .587}  & {\color{Top2} .432}  & {\color{Top2} .724}  & {\color{Top2} .832}  & \textcolor[rgb]{ .886,  .42,  .039}{.752} & \textcolor[rgb]{ .886,  .42,  .039}{.758} & \textcolor[rgb]{ .886,  .42,  .039}{.813} & \textcolor[rgb]{ .886,  .42,  .039}{.831} & \textcolor[rgb]{ .886,  .42,  .039}{1.356} & \textcolor[rgb]{ .886,  .42,  .039}{1.646} & \textcolor[rgb]{ .886,  .42,  .039}{.645} & \textcolor[rgb]{ .886,  .42,  .039}{.688} & \textcolor[rgb]{ .886,  .42,  .039}{.271} & \textcolor[rgb]{ .886,  .42,  .039}{.163} & \textcolor[rgb]{ .886,  .42,  .039}{.024} \\
          & $\alpha_3$   & .575  & .428  & .710  & .825  & .716  & .726  & .782  & .805  & 1.310 & 1.608 & .619  & .670  & .260  & .159  & {\color{Top2} .025} \\
    \end{tabular}%
\end{table*}%

\begin{table*}[t]
  \centering
    \caption{The empirical results of state-of-the-art MLC methods and TKPR on MS-COCO, where the backbone is ResNet101. The best and runner-up results on each metric are marked with {\color{Top1}red} and {\color{Top2}blue}, respectively. The best competitor on each measure is marked with \underline{underline}.}
    \label{tab:sota_coco_MLC}%
  \renewcommand\arraystretch{1.5}
  \tiny 
  \newcommand{\tabincell}[2]{\begin{tabular}{@{}#1@{}}#2\end{tabular}}
    \begin{tabular}{c|c|cc|cc|cc|cc|cc|cc|cc|c}
    \multicolumn{1}{c|}{\multirow{2}[4]{*}{Type}} & Metrics & \multicolumn{2}{c|}{P@K} & \multicolumn{2}{c|}{R@K} & \multicolumn{2}{c|}{mAP@K} & \multicolumn{2}{c|}{NDCG@K} & \multicolumn{2}{c|}{$\text{TKPR}^{\alpha_1}$} & \multicolumn{2}{c|}{$\text{TKPR}^{\alpha_2}$} & \multicolumn{2}{c|}{$\text{TKPR}^{\alpha_3}$} & \multicolumn{1}{c}{\multirow{2}[4]{*}{\tabincell{c}{Ranking \\ Loss}}} \\
    \cmidrule{2-16}          & K     & 3     & 5     & 3     & 5     & 3     & 5     & 3     & 5     & 3     & 5     & 3     & 5     & 3     & 5     &  \\
    \toprule
    \multicolumn{1}{c|}{\multirow{5}[2]{*}{\tabincell{c}{Loss \\ Oriented}}} & ASL\footnotemark[2]   & .668  & .474  & .800  & .885  & .879  & .868  & .910  & .910  & 1.536 & 1.841 & .722  & .754  & .305  & .181  & .015 \\
    & DB-Loss & \underline{.676}  & \underline{.475}  & \underline{.807}  & \underline{.886}  & \underline{.892}  & \underline{.877}  & \underline{.919} & \underline{.915} & \underline{1.554} & \underline{1.856} & \underline{.730}  & \underline{.759} & \underline{.308} & \underline{.182} & \underline{.015} \\
    & CCD\footnotemark[2]   & .654  & .463  & .783  & .868  & .860  & .848  & .894  & .893  & 1.510 & 1.803 & .709  & .740  & .299  & .177  & .018 \\
    & \multicolumn{1}{c|}{Hill\footnotemark[2]} & .643  & .462  & .775  & .868  & .829  & .826  & .874  & .881  & 1.467 & 1.774 & .692  & .731  & .292  & .175  & .019 \\
    & \multicolumn{1}{c|}{SPLC\footnotemark[2]} & .619  & .457  & .757  & .866  & .754  & .768  & .835  & .855  & 1.389 & 1.715 & .660  & .711  & .277  & .170  & .020 \\
    \toprule
    \multicolumn{1}{c|}{\multirow{3}[2]{*}{\tabincell{c}{TKPR \\ (Ours)}}} & $\alpha_1$ & {.678} & {.476} & \textcolor[rgb]{ .886,  .42,  .039}{.810} & {.889} & \textcolor[rgb]{ .886,  .42,  .039}{.895} & {.880} & {.889} & {.918} & {1.558} & \textcolor[rgb]{ .886,  .42,  .039}{1.862} & {.732} & {.762} & {.309} & {.183} & {.014} \\
    & $\alpha_2$ & {\color{Top2} .678} & {\color{Top2} .477} & .810 & {\color{Top2} .889} & .894 & {\color{Top2} .880} & {\color{Top2} .922} & {\color{Top2} .918} & {\color{Top2} 1.558} & 1.860 & \textcolor[rgb]{ .886,  .42,  .039}{.733} & {\color{Top2} .762} & {\color{Top2} .309} & {\color{Top2} .183} & {\color{Top2} .015} \\
    & $\alpha_3$ & \textcolor[rgb]{ .886,  .42,  .039}{.678} & \textcolor[rgb]{ .886,  .42,  .039}{.477} & {\color{Top2} .810} & \textcolor[rgb]{ .886,  .42,  .039}{.890} & \textcolor[rgb]{ .886,  .42,  .039}{.895} & \textcolor[rgb]{ .886,  .42,  .039}{.881} & \textcolor[rgb]{ .886,  .42,  .039}{.922} & \textcolor[rgb]{ .886,  .42,  .039}{.918} & \textcolor[rgb]{ .886,  .42,  .039}{1.558} & {\color{Top2} 1.862} & {\color{Top2} .732} & \textcolor[rgb]{ .886,  .42,  .039}{.762} & \textcolor[rgb]{ .886,  .42,  .039}{.309} & \textcolor[rgb]{ .886,  .42,  .039}{.183} & \textcolor[rgb]{ .886,  .42,  .039}{.014} \\
    \end{tabular}%
\end{table*}%

\subsection{Multi-Label Classification}
\label{subsec:experiment_mlc}
\subsubsection{Protocols}
  \noindent \textbf{Datasets.} We conduct the MLC experiments on three benchmark datasets:
  \begin{itemize}
      \item Pascal VOC 2007 \cite{DBLP:journals/ijcv/EveringhamGWWZ10} is a widely-used multi-label dataset for computer vision tasks. This dataset consists of 10K images coming from 20 different categories. There are 5,011 images and 4,952 images in the training set and the test set, respectively. Each image in the training set contains an average of 1.4 labels, with a maximum of 6 labels.
      \item MS-COCO \cite{DBLP:conf/eccv/LinMBHPRDZ14} is another popular dataset for multi-label recognization tasks, which consists of 122,218 images and 80 object categories. The dataset is split into a training set, consisting of 82,081 images, and a test set, consisting of 40,137 images. On average, each image in the training set contains 2.9 labels, with a maximum of 13 labels.
      \item NUSWDIE \cite{DBLP:conf/civr/ChuaTHLLZ09} is a large-scale multi-label dataset containing 269,648 Flickr images and 81 object categories. On average, each image in the training set contains 2.4 labels, with a maximum of 11 labels.
  \end{itemize}

  \noindent \textbf{Backbone and Optimization Method.} For CNN backbone, we utilize ResNet101 \cite{DBLP:conf/cvpr/HeZRS16} pre-trained on ImageNet \cite{DBLP:conf/cvpr/DengDSLL009} as the backbone, as used in \cite{DBLP:conf/eccv/WuH0WL20,DBLP:conf/iccv/RidnikBZNFPZ21}. The model is optimized by Stochastic Gradient Descent (SGD) with Nesterov momentum of 0.9 and a weight decay value of 1e-4 \cite{DBLP:conf/icml/SutskeverMDH13}. And an 1-cycle learning rate policy is utilized with the max learning rate searched in \{0.1, 0.01, 0.001, 0.0001\}. All input images are rescaled to $448 \times 448$, and the batch size is searched in \{32, 64, 128\}. For transformer backbone, we select swin-transformer \cite{DBLP:conf/iccv/LiuL00W0LG21} pre-trained on ImageNet-22k as the backbone. As suggested by \cite{DBLP:conf/cvpr/LiuLLHYY22}, we use Adam as the optimizer with a weight decay of 1e-4, a batch size of 32, a learning rate searched in \{5e-5, 1e-6\} with an 1-cycle policy, and the input images are rescaled to $384 \times 384$. More details can be found in Appendix \ref{app:more_imp_details}.

  \textbf{Evaluation Metric.} We evaluate the model performances on $\texttt{TKPR}^{\alpha}$, where $\alpha \in \{\alpha_1, \alpha_2, \alpha_3\}$. Meanwhile, we also report the results on \texttt{P@K}, \texttt{R@K}, \texttt{mAP@K}, \texttt{NDCG@K}, and the ranking loss, where $K \in \{3, 5\}$. Note that \texttt{NDCG@K} adopts a logarithmic discount function, and \texttt{mAP@K} summarizes the \texttt{AP@K} performance on different samples.

  \noindent \textbf{Competitors.} On one hand, we focus on ranking-based losses. Note that the pointwise surrogates of the ranking loss are also selected to validate Prop.\ref{prop:reformulation} and the generalization analyses in Sec.\ref{sec:generalization}:

\begin{itemize}
    \item $L_{\text{rank}}$ \cite{DBLP:journals/ai/GaoZ13} is exactly the loss defined in Eq.(\ref{eq:ranking_loss}).
    \item $L_{u_1}$ \cite{DBLP:conf/icml/DembczynskiKH12} has an abstract formulation $\frac{1}{C}\sum_{i} \ell_i$, where $\ell_i$ denotes the loss on the $i$-th class. While primarily designed for optimizing the Hamming loss, this loss also serves as a traditional surrogate for $L_{\text{rank}}$.
    \item $L_{u_2}$ \cite{DBLP:conf/icml/DembczynskiKH12} is a consistent surrogate for $L_{\text{rank}}$, which can be formulated as $\frac{1}{N(\boldsymbol{y}) N_{-}(\boldsymbol{y})} \sum_{i} \ell_i$.
    \item $L_{u_3}, L_{u_4}$ \cite{DBLP:conf/nips/WuLXZ21} are two reweighted surrogate pointwise losses, which can be formulated as $\frac{\sum_{i} \ell_i^+}{N(\boldsymbol{y})} + \frac{\sum_{i} \ell_i^-}{N_{-}(\boldsymbol{y})}$ and $\frac{\sum_{i} \ell_i}{\min \{N(\boldsymbol{y}), N_{-}(\boldsymbol{y})\}}$, respectively, where $\ell_i^-$ and $\ell_i^+$ denote the loss on relevant and irrelevant labels, respectively.
    \item LSEP \cite{DBLP:conf/cvpr/LiSL17} employs a smooth approximation of pairwise ranking to make the objective easier to optimize. Additionally, negative sampling techniques are used to reduce the computational complexity.
    \item TKML \cite{DBLP:conf/nips/HuY0L20} aims to maximize the score difference between the top-($k+1$) score of all the labels and the lowest prediction score of all the relevant labels.
\end{itemize}
On the other hand, we also compare several state-of-the-art loss-oriented MLC methods:
\begin{itemize}
    \item DB-Loss \cite{DBLP:conf/eccv/WuH0WL20} modifies the binary cross-entropy loss to tackle the imbalanced label distribution, which is achieved by rebalancing the weights of the co-occurrence of labels and restraining the dominance of negative labels.
    \item The ASymmetric Loss (ASL) \cite{DBLP:conf/iccv/RidnikBZNFPZ21} dynamically adjusts the weights to focus on hard-negative samples while keeping attention on positive samples.
    \item Hill, SPLC \cite{DBLP:journals/corr/abs-2112-07368} are two simple loss functions designed for both MLC and MLML. The Hill loss is a robust loss that can re-weight negatives to avoid the effect of false negatives. And the Self-Paced Loss Correction (SPLC) is derived from the maximum likelihood criterion under an approximate distribution of missing labels.
    \item The Causal Context Debiasing (CCD) \cite{DBLP:conf/cvpr/LiuLLHYY22} incorporates the casual inference to eliminate the contextual hard-negative objects and highlight the hard-positive objects.
\end{itemize}
\footnotetext[2]{Official implementation.}

\begin{figure}[t]
  \centering
  \subfigure[MS-COCO]{
    \label{fig:COCO_normal_distribution}
    \includegraphics[width=0.46\linewidth]{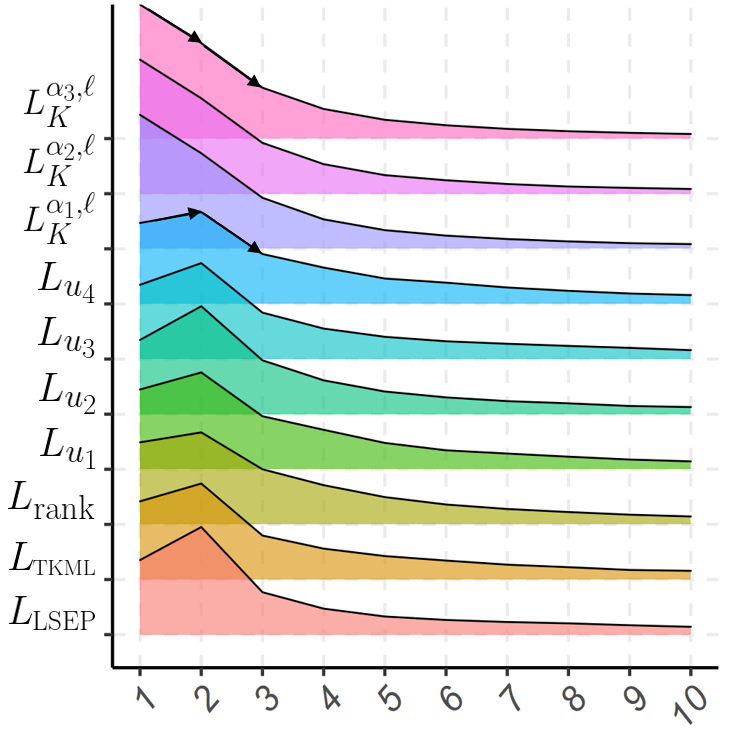}}
  \subfigure[Pascal VOC 2007]{
    \label{fig:VOC_normal_distribution}
    \includegraphics[width=0.46\linewidth]{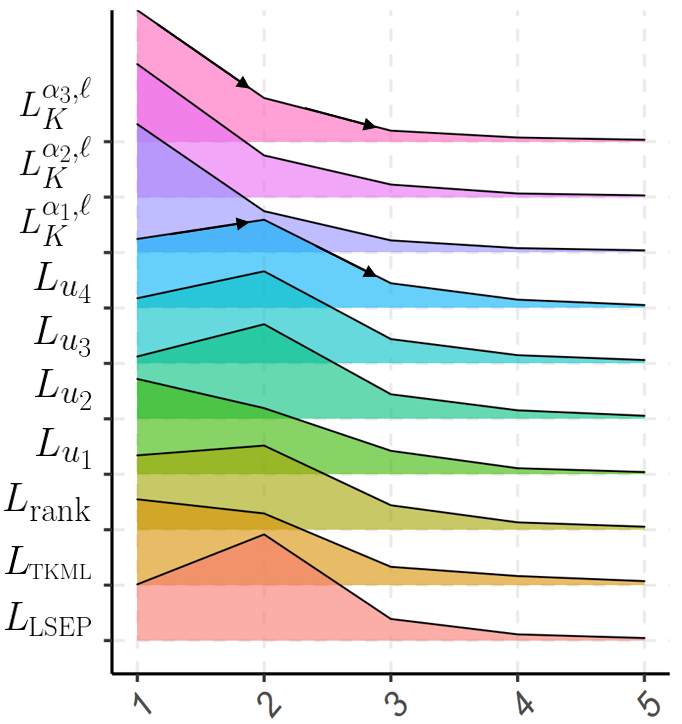}}
\caption{According to the model predictions, we visualize the rank distributions of relevant labels on MS-COCO and Pascal VOC 2007. The results show that the proposed methods can rank more relevant labels on the top-1 position, which explains why the proposed methods can improve the ranking-based measures.}
\label{fig:COCO_distribution}
\end{figure}

\subsubsection{Overall Performance}
  Here, we report the results on MS-COCO with the CNN backbone. The results on the transformer backbone and those on Pascal VOC 2007 and NUSWDIE can be found in Appendix \ref{app:more_results}. From Tab.\ref{tab:ranking_loss_coco_MLC} and Tab.\ref{tab:sota_coco_MLC}, we have the following observations:

  Compared with the ranking-based losses, the proposed algorithm for TKPR optimization can significantly improve the model performances on \texttt{mAP@K}, \texttt{NDCG@K}, the ranking loss, and the TKPR measures. For example, the performance gap between the proposed methods and the runner-up method is 9.3\%, 6.3\%, 5.4\%, and 6\textperthousand in terms of \texttt{mAP@3}, \texttt{NDCG@3}, $\texttt{TKPR}^{\alpha_2}\texttt{@3}$, and ranking loss, respectively. Notably, the performance gains on \texttt{mAP@K}, \texttt{NDCG@K}, the ranking loss not only validate our theoretical analyses in Sec.\ref{sec:tkpr} and Sec.\ref{sec:generalization} but also show the proposed learning algorithm can be a promising solution to optimizing existing ranking-based measures.

However, the performances on \texttt{P@K} and \texttt{R@K} are not so impressive. To explain this counter-intuitive phenomenon, we visualize the rank distribution of relevant labels in Fig.\ref{fig:COCO_distribution}. It is clear that the proposed methods rank most relevant labels at the top-1 position, while the competitors rank most relevant labels at the top-2 position. On one hand, this difference explains how the pairwise/listwise measures are improved under similar \texttt{P@K} and \texttt{R@K} performances. On the other hand, the rank distribution is consistent with the Bayes optimality presented in Sec.\ref{subsec:bayes}, which again validate the effectiveness of the proposed framework.

Compared with the state-of-the-art methods, the proposed methods achieve the best performances consistently on all the measures, which again validate the effectiveness of the proposed framework.

\begin{table*}[t]
  \centering
    \caption{The empirical results of the ranking-based losses and TKPR on MS-COCO-MLML, where the backbone is ResNet50. The best and runner-up results on each metric are marked with {\color{Top1}red} and {\color{Top2}blue}, respectively. The best competitor on each measure is marked with \underline{underline}.}
    \renewcommand\arraystretch{1.5}
    \tiny 
    \newcommand{\tabincell}[2]{\begin{tabular}{@{}#1@{}}#2\end{tabular}}
    \begin{tabular}{c|c|cc|cc|cc|cc|cc|cc|cc|c}
      \multicolumn{1}{c|}{\multirow{2}[4]{*}{Type}} & Metrics & \multicolumn{2}{c|}{P@K} & \multicolumn{2}{c|}{R@K} & \multicolumn{2}{c|}{mAP@K} & \multicolumn{2}{c|}{NDCG@K} & \multicolumn{2}{c|}{$\text{TKPR}^{\alpha_1}$} & \multicolumn{2}{c|}{$\text{TKPR}^{\alpha_2}$} & \multicolumn{2}{c|}{$\text{TKPR}^{\alpha_3}$} & \multicolumn{1}{c}{\multirow{2}[4]{*}{\tabincell{c}{Ranking \\ loss}}} \\
      \cmidrule{2-16}          & K     & 3     & 5     & 3     & 5     & 3     & 5     & 3     & 5     & 3     & 5     & 3     & 5     & 3     & 5     &  \\
      \toprule
      \multicolumn{1}{c|}{\multirow{7}[2]{*}{\tabincell{c}{Ranking \\ loss}}} & $L_{\text{rank}}$ & .517 & .402 & .639 & .781 & .544 & .570 & .660 & .702 & 1.102 & 1.426 & .513 & .590 & .217 & .141 & .0417 \\
      & $L_{u_1}$ & .519 & .360 & .658 & .728 & .686 & .673 & .748 & .749 & 1.246 & 1.447 & .602 & .626 & .251 & .148 & .0643 \\
      & $L_{u_2}$ & .512 & .358 & .652 & .724 & .677 & .668 & .738 & .742 & 1.225 & 1.431 & .594 & .620 & .247 & .146 & .0671 \\
      & $L_{u_3}$ & \underline{.570} & \underline{.426} & \underline{.700} & \underline{.815} & .664 & .677 & .745 & \underline{.770} & 1.252 & \underline{1.569} & .585 & \underline{.646} & .247 & \underline{.155} & \underline{.0312} \\
      & $L_{u_4}$ & .531 & .371 & .671 & .745 & \underline{.700} & \underline{.688} & \underline{.761} & .763 & \underline{1.268} & 1.482 & \underline{.611} & .637 & \underline{.255} & .150 & .0714 \\
      & $L_{\text{LSEP}}$ & .542 & .414 & .672 & .800 & .575 & .596 & .696 & .730 & 1.168 & 1.493 & .543 & .617 & .230 & .147 & .0346 \\
      & $L_{\text{TKML}}$ & .536 & .387 & .668 & .764 & .575 & .575 & .687 & .701 & 1.167 & 1.442 & .540 & .600 & .229 & .143 & .0452 \\
      \toprule
      \multicolumn{1}{c|}{\multirow{3}[2]{*}{\tabincell{c}{TKPR \\ (Ours)}}}         & $\alpha_1$ & \color{Top2}.627 & \color{Top2}.449 & \color{Top2}.763 & \color{Top2}.854 & \color{Top2}.820 & \color{Top2}.814 & \color{Top2}.864 & \color{Top2}.871 & \color{Top2}1.451 & \color{Top2}1.741 & \color{Top1}.688 & \color{Top2}.722 & \color{Top2}.289 & \color{Top2}.172 & \color{Top1}.0228 \\
      & $\alpha_2$ & .626 & .446 & .762 & .850 & .820 & .813 & .864 & .869 & 1.450 & 1.734 & .687 & .721 & .288 & .171 & .0249 \\
      & $\alpha_3$ & \color{Top1}.628 & \color{Top1}.450 & \color{Top1}.763 & \color{Top1}.854 & \color{Top1}.821 & \color{Top1}.815 & \color{Top1}.865 & \color{Top1}.871 & \color{Top1}1.453 & \color{Top1}1.741 & \color{Top2}.687 & \color{Top1}.722 & \color{Top1}.289 & \color{Top1}.172 & \color{Top2}.0234 \\
    \end{tabular}%
  \label{tab:ranking_loss_coco_MLML}%
\end{table*}%

\subsubsection{Fine-grained Analysis}
To further validate the theoretical results in Sec.\ref{sec:tkpr}, we visualize the normalized measures \textit{w.r.t.} the training epoch on the Pascal VOC 2007 dataset. Specifically, given a sequence of measures $\mathcal{M} := \{m^{(t)}\}_{t = 1}^{T}$, where $T$ is the maximum number of epochs, we normalize these measures by 
$$
  \tilde{m}^{(t)} := \frac{\max\{\mathcal{M}\} - m^{(t)}}{\max\{\mathcal{M}\} - \min\{\mathcal{M}\}}.
$$
As shown in Fig.\ref{fig:tendency_normal}, we have the following observations: (1) For the competitors, \textit{i.e.}, in (a)-(c), different measures reach the peak values at different epochs. In other words, the tendency of different measures are inconsistent when optimizing these losses. (2) Optimizing the ranking loss and its surrogates fails to guarantee the model performances on the other measures. For example, in (a), \texttt{mAP@5} even shows a decreasing trend at the late epochs. (3) By contrast, when optimizing TKPR, all the measures display a similar increasing trend, which is consistent with the our analyses in Sec.\ref{sec:tkpr}.

\begin{figure}[t]
  \centering
\subfigure[Optimize $L_{\text{rank}}$]{
  \label{fig:tendency_L_rank_normal}
  \includegraphics[width=0.46\linewidth]{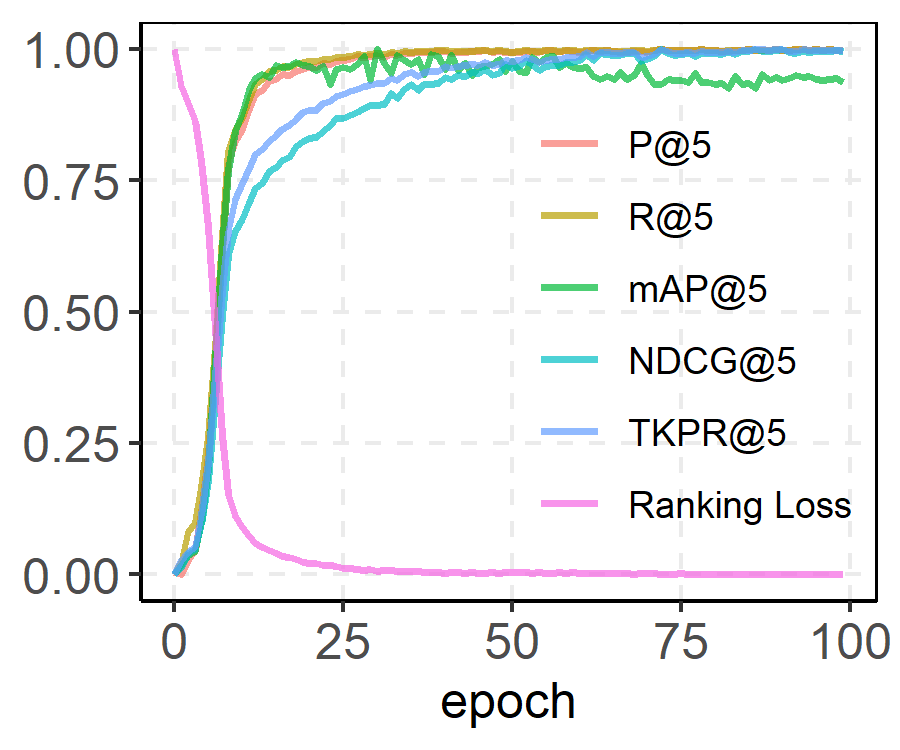}}
\subfigure[Optimize $L_{u_4}$]{
    \label{fig:tendency_L_u4_normal}
    \includegraphics[width=0.46\linewidth]{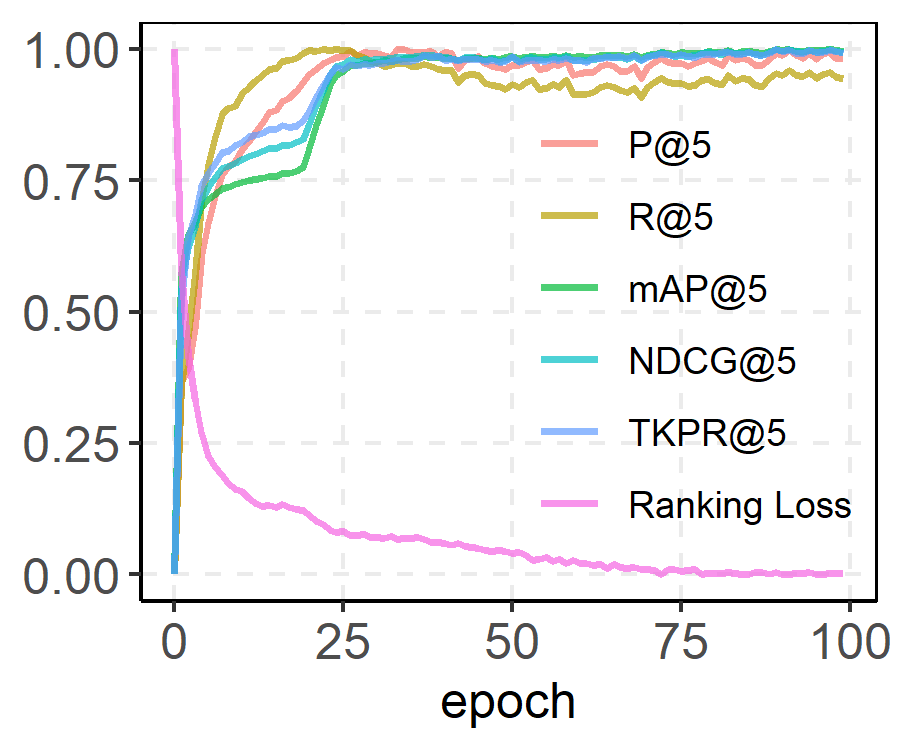}}
\subfigure[Optimize $L_{\text{LSEP}}$]{
    \label{fig:tendency_LSEP_normal}
    \includegraphics[width=0.46\linewidth]{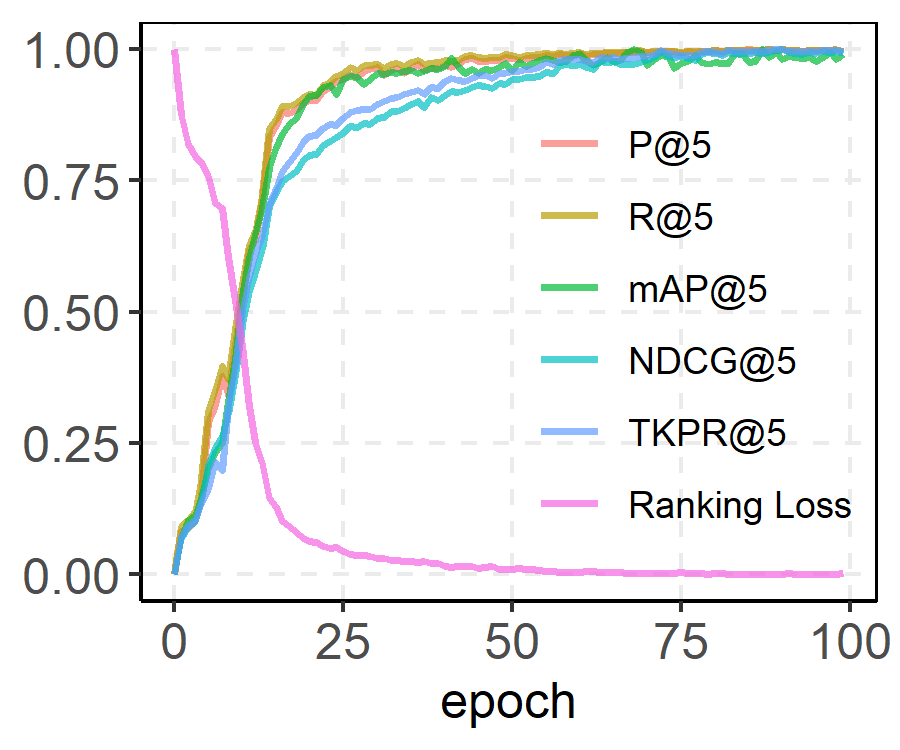}}
\subfigure[Optimize $L_K^{\alpha, \ell}$ (Ours)]{
  \label{fig:tendency_TKPR_normal}
  \includegraphics[width=0.46\linewidth]{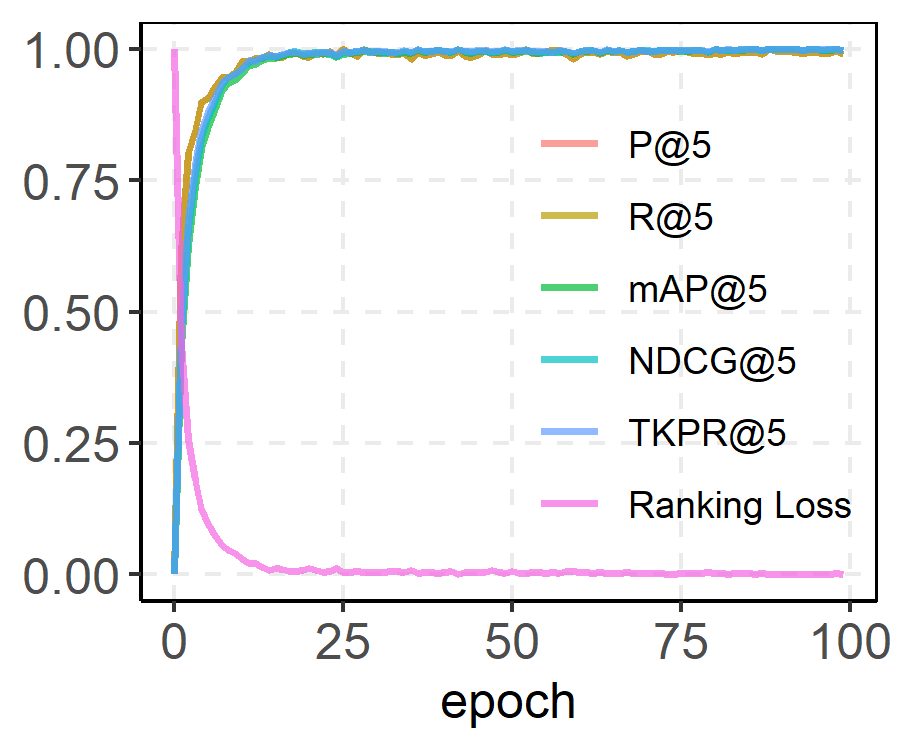}}
\caption{Normalized ranking-based measures \textit{w.r.t.} the training epoch on the Pascal VOC 2007 dataset in the MLC setting. (a)-(c) When optimizing the competitors, the changes in different measures are inconsistent. (d) By contrast, when optimizing TKPR, the changes are highly consistent, which validates our analyses in Sec.\ref{sec:tkpr}.}
\label{fig:tendency_normal}
\end{figure}

\begin{table*}[t]
  \centering
    \caption{The empirical results of state-of-the-art MLML methods and TKPR on MS-COCO-MLML, where the backbone is ResNet50. The best and runner-up results on each metric are marked with {\color{Top1}red} and {\color{Top2}blue}, respectively. The best competitor on each measure is marked with \underline{underline}.}
    \renewcommand\arraystretch{1.5}
    \tiny 
    \newcommand{\tabincell}[2]{\begin{tabular}{@{}#1@{}}#2\end{tabular}}
    \begin{tabular}{c|c|cc|cc|cc|cc|cc|cc|cc|c}
      \multicolumn{1}{c|}{\multirow{2}[4]{*}{Type}} & Metrics & \multicolumn{2}{c|}{P@K} & \multicolumn{2}{c|}{R@K} & \multicolumn{2}{c|}{mAP@K} & \multicolumn{2}{c|}{NDCG@K} & \multicolumn{2}{c|}{$\text{TKPR}^{\alpha_1}$} & \multicolumn{2}{c|}{$\text{TKPR}^{\alpha_2}$} & \multicolumn{2}{c|}{$\text{TKPR}^{\alpha_3}$} & \multicolumn{1}{c}{\multirow{2}[4]{*}{\tabincell{c}{Ranking \\ loss}}} \\
      \cmidrule{2-16}  &   K     & 3     & 5     & 3     & 5     & 3     & 5     & 3     & 5     & 3     & 5     & 3     & 5     & 3     & 5     &  \\
      \toprule
      \multicolumn{1}{c|}{\multirow{7}[4]{*}{\tabincell{c}{Loss \\ Oriented}}}    
      & \multicolumn{1}{c|}{ROLE\footnotemark[2]} & .618 & .439 & .752 & .837 & .812 & .801 & .858 & .860 & 1.441 & 1.715 & .683 & .713 & .287 & .170 & .0415 \\
      & \multicolumn{1}{c|}{EM+APL\footnotemark[2]} & \underline{.638} & \underline{.456} & \underline{.772} & \underline{.862} & \underline{.838} & \underline{.831} & \underline{.879} & \underline{.883} & \underline{1.478} & \underline{1.770} & \underline{.699} & \underline{.732} & \underline{.294} & \underline{.175} & .0265 \\
      & \multicolumn{1}{c|}{Hill\footnotemark[2]} & .594 & .424 & .726 & .816 & .775 & .768 & .825 & .832 & 1.386 & 1.652 & .658 & .690 & .276 & .164 & .0296 \\
      & SPLC\footnotemark[2] & .599 & .424 & .731 & .814 & .784 & .773 & .833 & .836 & 1.402 & 1.662 & .664 & .693 & .279 & .165 & .0319 \\
      & LL-R\footnotemark[2] & .501 & .349 & .761 & .850 & .708 & .741 & .768 & .807 & 1.184 & 1.389 & .637 & .710 & .252 & .158 & \underline{.0250} \\
      & LL-Ct\footnotemark[2] & .500 & .348 & .761 & .849 & .709 & .741 & .768 & .807 & 1.185 & 1.389 & .638 & .710 & .253 & .158 & .0251 \\
      & LL-Cp\footnotemark[2] & .495 & .345 & .755 & .842 & .702 & .734 & .762 & .801 & 1.177 & 1.378 & .633 & .704 & .251 & .157 & .0272 \\ 
      \toprule
      \multicolumn{1}{c|}{\multirow{3}[2]{*}{\tabincell{c}{TKPR \\ (Ours)}}} & $\alpha_1$ & \color{Top2}.642 & \color{Top2}.456 & \color{Top2}.776 & \color{Top2}.863 & \color{Top2}.843 & \color{Top2}.834 & \color{Top2}.882 & \color{Top2}.885 & \color{Top2}1.485 & \color{Top2}1.775 & \color{Top2}.701 & \color{Top2}.733 & \color{Top2}.295 & \color{Top2}.175 & \color{Top2}.0200 \\
      & $\alpha_2$ & \color{Top1}.643 & \color{Top1}.457 & \color{Top1}.777 & \color{Top1}.864 & \color{Top1}.845 & \color{Top1}.835 & \color{Top1}.883 & \color{Top1}.885 & \color{Top1}1.488 & \color{Top1}1.778 & \color{Top1}.702 & \color{Top1}.734 & \color{Top1}.296 & \color{Top1}.176 & \color{Top1}.0196 \\
      & $\alpha_3$ & .640 & .453 & .774 & .859 & .842 & .831 & .881 & .882 & 1.483 & 1.769 & .701 & .732 & .295 & .175 & .0215 \\
    \end{tabular}%
  \label{tab:sota_coco_MLML}%
\end{table*}%

\subsection{Multi-Label Classification with Missing Labels}
\label{subsec:experiment_mlml}
  \subsubsection{Protocols}
  \label{subsubsec:protocol_mlml}
  \noindent \textbf{Datasets.} The experiments are conducted on MS-COCO \cite{DBLP:conf/eccv/LinMBHPRDZ14} and Pascal VOC 2012 \cite{DBLP:journals/ijcv/EveringhamGWWZ10}. VOC 2012 contains 20 classes and 5,717 training images, which are non-overlapping with the Pascal VOC 2007 dataset. To simulate the MLML setting, we randomly select one positive label for each training example, as performed in \cite{DBLP:conf/cvpr/ColeALPMJ21,DBLP:conf/eccv/ZhouCWCH22, DBLP:conf/cvpr/KimKA022}.

  \noindent \textbf{Backbone and Optimization Method.} Following the prior arts \citep{DBLP:conf/cvpr/ColeALPMJ21,DBLP:conf/eccv/ZhouCWCH22, DBLP:conf/cvpr/KimKA022}, we use ResNet50 \citep{DBLP:conf/cvpr/HeZRS16} pre-trained on ImageNet \citep{DBLP:conf/cvpr/DengDSLL009} as the backbone and Adam as the optimizer, with a weight decay of 1e-4, a batch size of 64, a learning rate searched in \{1e-4, 1e-5\} with an 1-cycle policy. The input images are rescaled to $448 \times 448$.  More details can be found in Appendix \ref{app:more_imp_details}.

\noindent \textbf{Competitors.} Besides the ranking-based losses, Hill, and SPLC, we select the following state-of-the-art MLML methods as the competitors:
\begin{itemize}
  \item ROLE \cite{DBLP:conf/cvpr/ColeALPMJ21} proposes a regularization that constrains the number of expected relevant labels to tackle the single-relevant problem.
  \item EM+APL \cite{DBLP:conf/eccv/ZhouCWCH22} combines the entropy-maximization (EM) loss and an asymmetric pseudo-labeling (APL) scheme to address the single-relevant problem.
  \item LL-R, LL-Ct, and LL-Cp \cite{DBLP:conf/cvpr/KimKA022} belong to the Large-Loss-Matter framework, where the missing labels are regarded as noises. Based on this observation, these methods reject or correct samples with large losses to prevent the model from memorizing the noisy labels. 
\end{itemize}

\begin{figure*}[t]
    \centering
    \subfigure[Optimize $L_K^{\alpha_1, \ell}$]{
      \label{fig:voc2012_K_M}
      \includegraphics[width=0.28\linewidth]{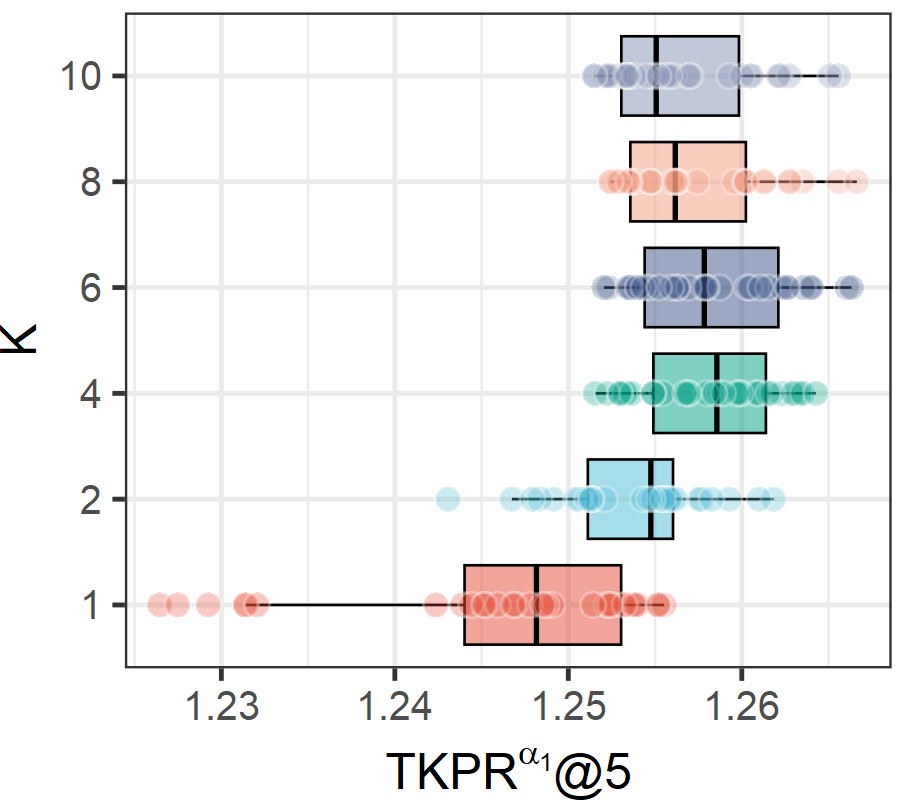}}
    \subfigure[Optimize $L_K^{\alpha_2, \ell}$]{
        \label{fig:voc2012_K_L}
        \includegraphics[width=0.28\linewidth]{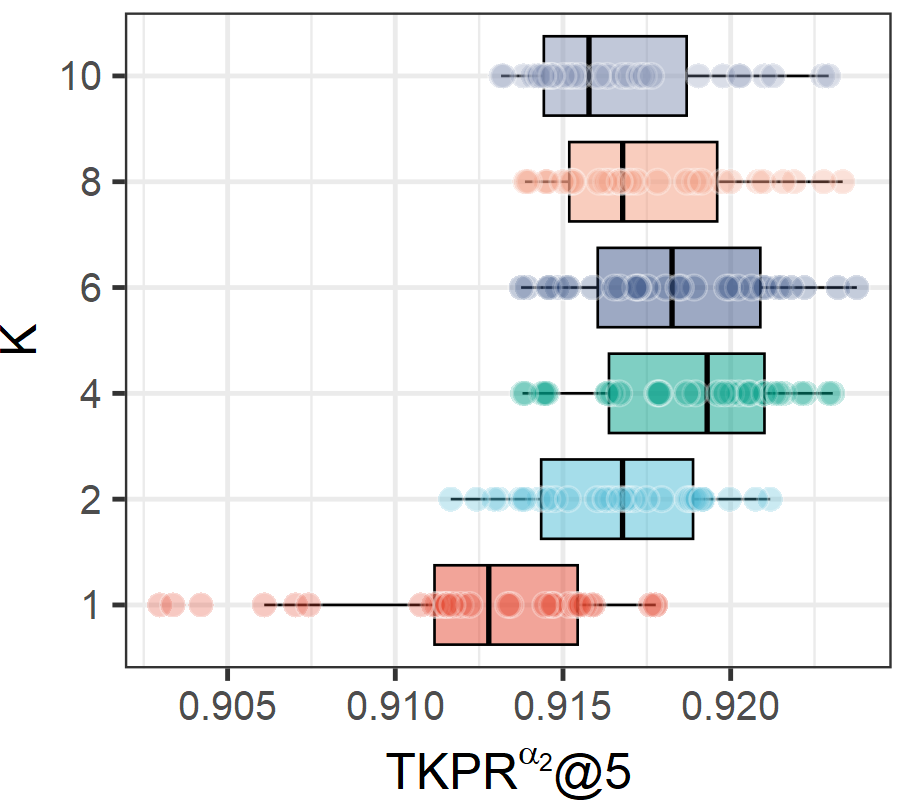}}
    \subfigure[Optimize $L_K^{\alpha_3, \ell}$]{
        \label{fig:voc2012_K_Q}
        \includegraphics[width=0.28\linewidth]{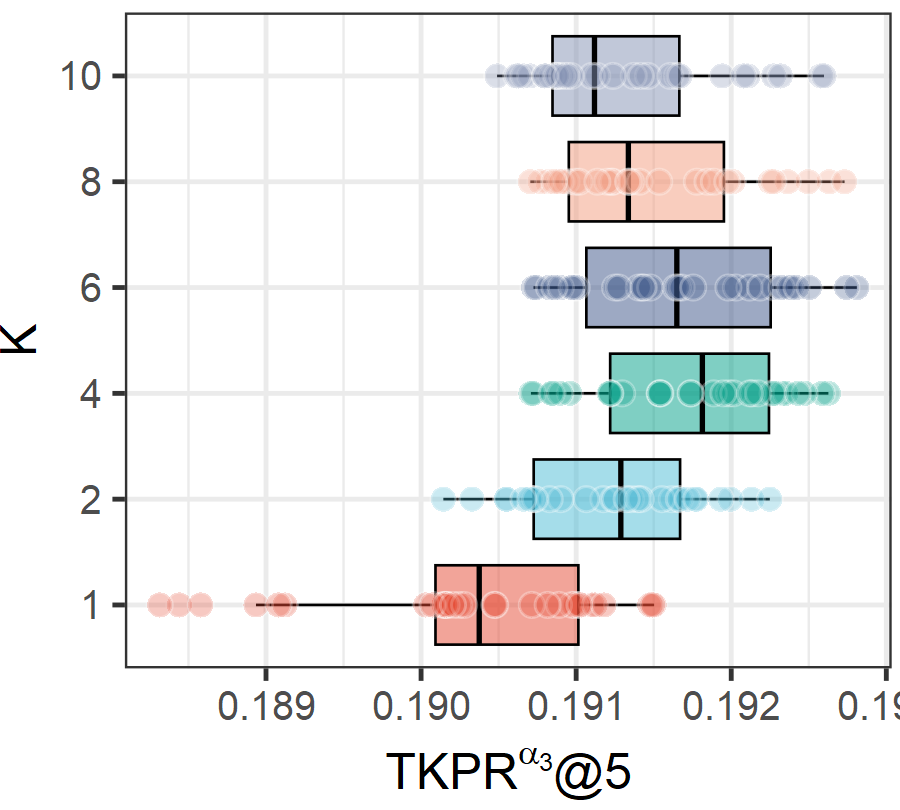}}
    \caption{Sensitivity analysis of the proposed methods on Pascal VOC 2012-MLML. The y-axis denotes the values of the hyperparameter $K$, and the x-axis represents the value of \texttt{TKPR@5} under the corresponding $K$.}
    \label{fig:voc2012_K}
\end{figure*}

\begin{figure*}[t]
    \centering
    \subfigure[Optimize $L_K^{\alpha_1, \ell}$]{
      \label{fig:voc2012_Ew_M}
      \includegraphics[width=0.28\linewidth]{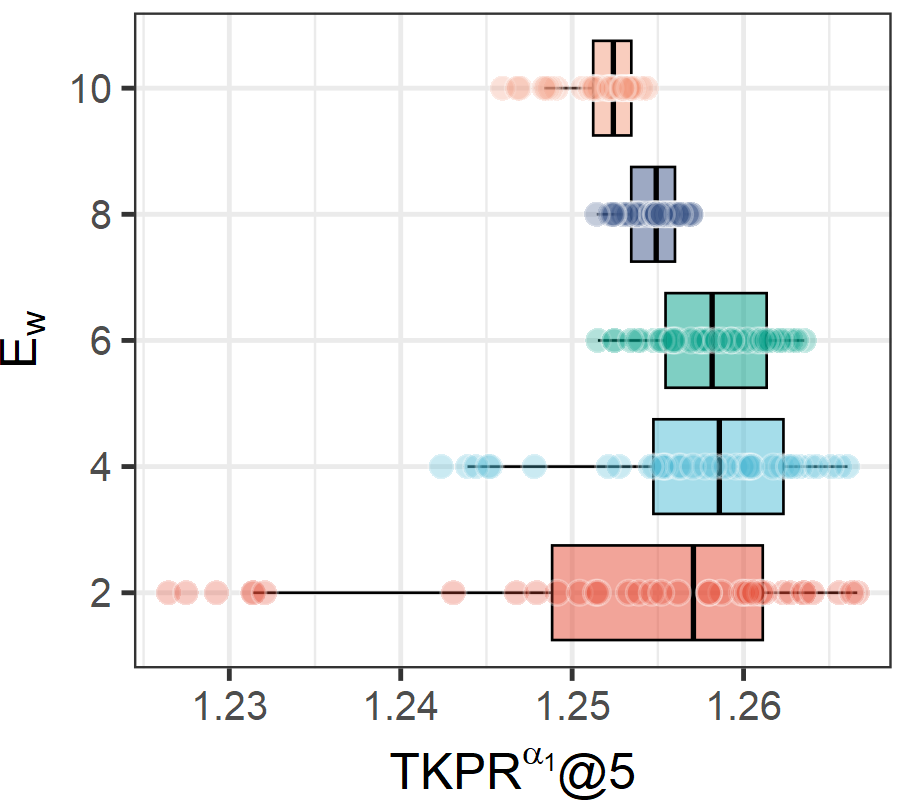}}
    \subfigure[Optimize $L_K^{\alpha_2, \ell}$]{
        \label{fig:voc2012_Ew_L}
        \includegraphics[width=0.28\linewidth]{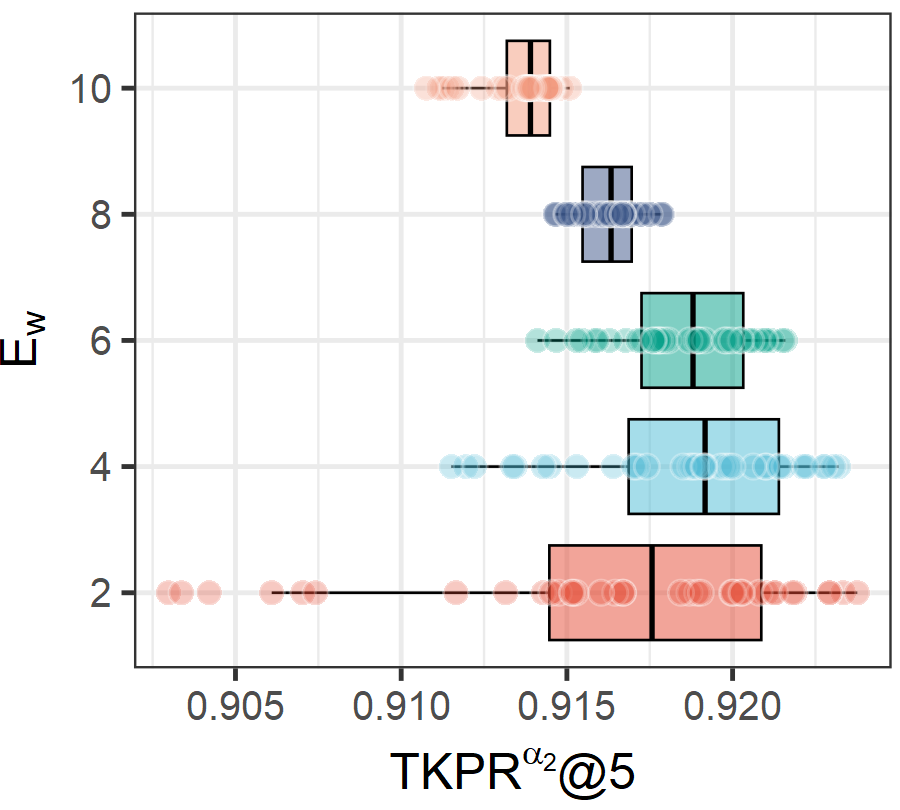}}
    \subfigure[Optimize $L_K^{\alpha_3, \ell}$]{
        \label{fig:voc2012_Ew_Q}
        \includegraphics[width=0.28\linewidth]{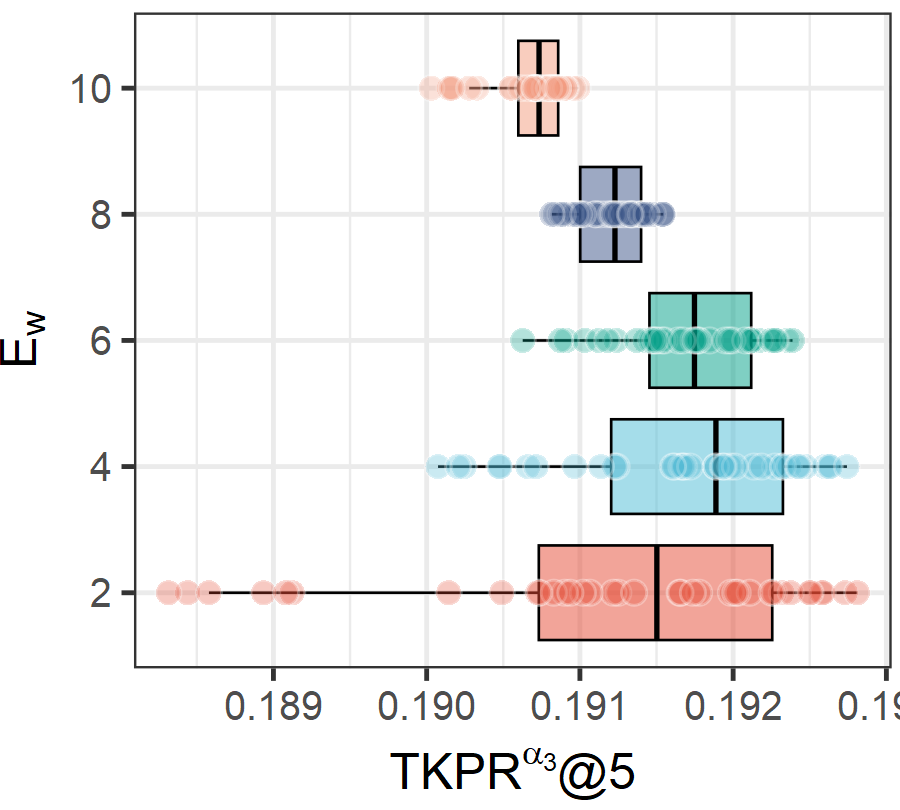}}
    \caption{Sensitivity analysis of the proposed methods on Pascal VOC 2012-MLML. The y-axis denotes the values of the hyperparameter $E_w$, and the x-axis represents the value of \texttt{TKPR@5} under the corresponding $E_w$.}
    \label{fig:voc2012_Ew}
\end{figure*}

\subsubsection{Overall Performance}
  Here, we report the results on MS-COCO. The results on Pascal VOC 2012 and the corresponding fine-grained analyses can be found in Appendix \ref{app:pascal12}. From Tab.\ref{tab:ranking_loss_coco_MLML} and Tab.\ref{tab:sota_coco_MLML}, we have the following observations:

  Compared with the ranking-based losses, the proposed methods demonstrate consistent improvements. For example, the performance gap between the proposed methods and the runner-up method increases to 12.1\%, 10.4\%, 7.7\%, 8.4\% in terms of \texttt{mAP@3}, \texttt{NDCG@3}, $\texttt{TKPR}^{\alpha_2}\texttt{@3}$, and the ranking loss, respectively. This phenomenon is not surprising since the competitors explicitly require irrelevant labels as inputs. Consequently, the missing relevant labels will bias the learning process and degenerate the model performance. By contrast, the proposed framework only takes the relevant labels as inputs and thus alleviates the negative impact of the missing labels. 
  
  Compared with the state-of-the art MLML methods, the proposed methods also achieve consistent improvements. For example, the performance gap between the proposed methods and the runner-up method is 0.7\%, 0.4\%, 0.3\%, and 5.4\% in terms of \texttt{mAP@3}, \texttt{NDCG@3}, $\texttt{TKPR}^{\alpha_2}\texttt{@3}$, and ranking loss, respectively. It is worth mentioning that although EM+APL achieves superior performance on top-ranking measures, it fails to outperform LL-R and LL-Ct on the ranking loss. This phenomenon validates the inconsistency among different measures. Hence, it is necessary to optimize compatible measures such as the proposed TKPR.

\subsubsection{Sensitivity Analysis}
  Next, we investigate the sensitivity of the proposed methods to the hyperparameters $K$ and $E_w$. The results are shown in Fig.\ref{fig:voc2012_K} and Fig.\ref{fig:voc2012_Ew}, where we adopt $L_\text{rank} $ as the warm-up loss. From the results, we have the following observations:

  For $K$, the best performances of the proposed methods are achieved when $K \in \{4, 6\}$. A small $K$ will lead to a lack of supervision, while a large $K$ means ranking positions that are out of interest are optimized. Hence, both of them will degrade the performance. Besides, the best $\texttt{TKPR@K}$ performances are achieved when $K$ is close to 5, which is consistent with the expectation of the proposed framework.  

  For $E_w$, a larger $E_w$ leads to more stable performances. However, as $E_w$ increases, the effect of the proposed methods becomes weaker, inducing a significant performance degradation. Hence, a moderate $E_w$ is preferred to balance the learning between the global pattern and the local ranking. 


\section{Conclusion, Limitation, and Future Work}
This paper proposes a novel measure for multi-label ranking named TKPR. A series of analyses show that optimizing TKPR is compatible with existing ranking-based measures. In view of this, an empirical surrogate risk minimization framework is further established for TKPR with theoretical support. On one hand, this ERM framework enjoys convex surrogate losses. On the other hand, a novel technique, named data-dependent contraction, helps the proposed framework achieve a sharp generalization bound on both TKPR and the ranking loss. Finally, experimental results on different benchmark datasets and settings speak to the effectiveness of the proposed framework, as well as the theoretical analyses.

    TKPR is closely related to: 1) the scenarios where the top-ranking performance is of interest. As shown in Sec.4, TKPR is more discriminating than \texttt{P@K} and \texttt{R@K}. Although one can select a suitable $K$ such that \texttt{P@K} and \texttt{R@K} can find the performance difference on a specific sample, the number of ground-truth labels differs among samples, making the selection of $K$ rather challenging. Hence, TKPR is a better choice in such scenarios. Of course, \texttt{AP@K} and \texttt{NDCG@K} are also reasonable options, but it is challenging to perform direct optimization on these measures. 2) the scenarios where multiple ranking-based measures are important. As shown in Sec.4, if the model achieves a good performance on TKPR, it tends to perform well on the others. Hence, we can reduce a multi-objective problem to a single-objective problem, which is rather convenient. 

    The proposed learning algorithm is useful in: 1) the scenarios where the label distribution is highly imbalanced such as medical image classification and remote-sensing image classification. This is because the proposed algorithm can induce better performance on the ranking loss and $\texttt{AP@K}$, which are popular measures in imbalanced learning. 2) the scenarios where the ranking of the top-$K$ labels is of interest. The proposed algorithm can effectively optimize the top-$K$ ranking list. 3) the scenarios where some ground-truth labels are missing, which is universal in multi-learning since full annotations are rather expensive and time-consuming. As shown in Sec.6.3, the proposed algorithm shows superior performance in such scenarios.

    Theoretically, this work focuses on ranking-based measures and does not discuss the connection between TKPR and threshold-based measures such as the Hamming loss, the subset accuracy, and the F-measure. Methodologically, it might be a promising direction to replace the na\"ive ranking operators in the TKPR objective with differentiable ranking operators \cite{NEURIPS2020_ec24a54d}, such that the loss computation can be more efficient. Additionally, the application of TKPR in other learning tasks such as retrieval and recommendation is also an interesting direction.

\section*{Declarations}
\subsection*{Availability of data and materials}
    The datasets that support the experiments in Sec.\ref{subsec:experiment_mlc} are available in Pascal VOC 2007\footnote{\tiny \url{http://host.robots.ox.ac.uk/pascal/VOC/voc2007/index.html}} \cite{DBLP:journals/ijcv/EveringhamGWWZ10}, MS-COCO\footnote{\tiny \url{https://cocodataset.org}} \cite{DBLP:conf/eccv/LinMBHPRDZ14}, and NUSWDIE\footnote{\tiny \url{https://lms.comp.nus.edu.sg/wp-content/uploads/2019/research/nuswide/NUS-WIDE.html}} \cite{DBLP:conf/civr/ChuaTHLLZ09}.

    Furthermore, the datasets that support the experiments in Sec.\ref{subsec:experiment_mlml} are also generated from Pascal VOC 2012\footnote{\tiny \url{http://host.robots.ox.ac.uk/pascal/VOC/voc2012/}}\cite{DBLP:journals/ijcv/EveringhamGWWZ10} and MS-COCO \cite{DBLP:conf/eccv/LinMBHPRDZ14} based on the protocol described in \cite{DBLP:conf/cvpr/ColeALPMJ21}, which we have reviewed in Sec.\ref{subsubsec:protocol_mlml}.

\subsection*{Acknowledgment}
This work was supported in part by the National Key R\&D Program of China under Grant 2018AAA0102000, in part by National Natural Science Foundation of China: 62236008, U21B2038, U23B2051, U2001202, 61931008,  62122075 and 61976202, in part by Youth Innovation Promotion Association CAS, in part by the Strategic Priority Research Program of the Chinese Academy of Sciences, Grant No. XDB0680000, in part by the Innovation Funding of ICT, CAS under Grant No.E000000, in part by the China National Postdoctoral Program for Innovative Talents under Grant BX20240384.

\begin{appendices}
\onecolumn
\section{Bayes optimality of P@K and R@K Optimization}
\label{app:bayesofpr}

\cite{DBLP:conf/nips/X19} has presented the Bayes Optimality of \texttt{P@K} and \texttt{R@K} under the no-tie assumption. Here, we extend the results under the with-tie assumption, \textit{i.e.}, Asm.\ref{asm:tiesexist}, which is necessary for the further comparison in Sec.\ref{subsec:bayes}.
\begin{restatable}[Bayes optimality of \texttt{P@K} and \texttt{R@K}]{proposition}{bayesofpr}
    \label{prop:bayes_of_pr}
    Under Asm.\ref{asm:tiesexist}, we have
    $$\begin{aligned}
        & f^* \in \arg\max_{f} \texttt{P@K}(f) \Leftrightarrow \forall \boldsymbol{x} \in \mathcal{X}, \widetilde{\texttt{Top}}_K(\eta) \subset \texttt{Top}_K(f^*), \\
        & \texttt{Top}_K(f^*) - \widetilde{\texttt{Top}}_K(\eta) \subset \texttt{Tie}_K(\eta), \\
    \end{aligned}$$
    where$ \widetilde{\texttt{Top}}_K(\eta) := \texttt{Top}_K(\eta) - \texttt{Tie}_K(\eta)$, $\texttt{Top}_K(\eta)$ denotes the indices of the top-K entries of $\eta(\boldsymbol{x}) \in \mathbb{R}^C$. Meanwhile, let 
    \begin{equation}
        \eta'(\boldsymbol{x})_i := \eta(\boldsymbol{x})_i \cdot \underset{\boldsymbol{y}_{\setminus i} \mid \boldsymbol{x}, y_i=1}{\mathbb{E}}\left[\frac{1}{1 + N(\boldsymbol{y}_{\setminus i})}\right],
    \end{equation}
    and $\boldsymbol{y}_{\setminus i} \in \{0, 1\}^{C-1}$ denotes the vector of all but the $i$-th label. Then, we have
    $$\begin{aligned}
        & f^* \in \arg\max_{f} \texttt{R@K}(f) \Leftrightarrow \forall \boldsymbol{x} \in \mathcal{X}, \widetilde{\texttt{Top}}_K(\eta') \subset \texttt{Top}_K(f^*), \\
        & \texttt{Top}_K(f^*) - \widetilde{\texttt{{Top}}}_K(\eta') \subset \texttt{Tie}_K(\eta'). \\
    \end{aligned}$$ 
\end{restatable}

\begin{proof}
    Since the sample $(\boldsymbol{x}, \boldsymbol{y})$ is \textit{i.i.d.} sampled from $\mathcal{D}$, we next consider the following conditional formulations:
    $$\begin{aligned}
        \texttt{P@K}(f \mid \boldsymbol{x}) = \E{\boldsymbol{y} \mid \boldsymbol{x}}{\frac{1}{K} \sum_{k=1}^{K} y_{\sigma(f, k)}}, \\
        \texttt{R@K}(f \mid \boldsymbol{x}) = \E{\boldsymbol{y} \mid \boldsymbol{x}}{\frac{1}{N(\boldsymbol{y})} \sum_{k=1}^{K} y_{\sigma(f,  k)}}. \\
    \end{aligned}$$
    For \texttt{P@K}, we have
    $$\begin{aligned}
        \texttt{P@K}(f \mid \boldsymbol{x}) & = \frac{1}{K} \sum_{\boldsymbol{y} \in \mathcal{Y}} \pp{\boldsymbol{y} \mid \boldsymbol{x}} \sum_{k=1}^{K} y_{\sigma(f,  k)} \\
        & = \frac{1}{K} \sum_{\boldsymbol{y} \in \mathcal{Y}} \pp{\boldsymbol{y} \mid \boldsymbol{x}} \sum_{i \in \mathcal{L}} y_i \cdot \I{i \in \texttt{Top}_K(f)} \\
        & = \frac{1}{K} \sum_{i \in \texttt{Top}_K(f)} \sum_{\boldsymbol{y}: y_i = 1} \pp{\boldsymbol{y} \mid \boldsymbol{x}} \\
        & = \frac{1}{K} \sum_{i \in \texttt{Top}_K(f)} \eta(\boldsymbol{x})_i \\
    \end{aligned}$$
    Note that $\size{\texttt{Top}_K(\eta) \cup \texttt{Tie}_K(\eta)}$ might be greater than $K$. To maximize $\texttt{P@K}(f \mid \boldsymbol{x})$, $\texttt{Top}_K(f)$ must consist of two parts: all the labels in $\tilde{\texttt{Top}}_K(\eta)$ and $K - \size{\tilde{\texttt{Top}}_K(\eta)}$ labels from $\texttt{Tie}_K(\eta)$. That is, 
    $$
        \tilde{\texttt{Top}}_K(\eta) \subset \texttt{Top}_K(f^*),  \texttt{Top}_K(f^*) - \tilde{\texttt{Top}}_K(\eta) \subset \texttt{Tie}_K(\eta).
    $$
    For \texttt{R@K}, we have 
    $$\begin{aligned}
        \texttt{R@K}(f \mid \boldsymbol{x}) & = \sum_{i \in \texttt{Top}_K(f)} \E{\boldsymbol{y} \mid \boldsymbol{x}}{\frac{y_i}{N(\boldsymbol{y})}} \\
        & = \sum_{i \in \texttt{Top}_K(f)} \E{y_i}{\E{\boldsymbol{y}_{\setminus i} \mid \boldsymbol{x}, y_i}{\frac{y_i}{N(\boldsymbol{y})}}} \\
        & = \sum_{i \in \texttt{Top}_K(f)} \eta(\boldsymbol{x})_i \cdot \underset{\boldsymbol{y}_{\setminus i} \mid \boldsymbol{x}, y_i=1}{\mathbb{E}}\left[\frac{1}{1 + N(\boldsymbol{y}_{\setminus i})}\right] \\
        & = \sum_{i \in \texttt{Top}_K(f)} \eta'(\boldsymbol{x})_i
    \end{aligned}$$
    Then, the proof follows the above analysis of \texttt{P@K}.

    As shown in \cite{DBLP:conf/nips/X19}, the order of $\eta(\boldsymbol{x})$ and $\eta'(\boldsymbol{x})$ are generally inconsistent. Thus, optimizing one measure cannot guarantee the performance on the other one.
\end{proof}

\section{Comparison Between TKPR and Other Metrics}
\subsection{TKPR v.s. NDCG@K (Proof of Prop.\ref{prop:tkprtondcg})}
\label{app:tkprtondcg}
\tkprtondcg*
\begin{proof}
    According to the definition of TKPR,
    \begin{equation}
        \label{eq:tkprtondcg}
        \begin{aligned}
            \texttt{TKPR}(f, \boldsymbol{y}) & = \frac{1}{\alpha K} \sum_{y \in \mathcal{P}(\boldsymbol{y})} \sum_{k \le K} \I{\pi_{f}(y) \le k} \\
            & = \frac{1}{\alpha K} \sum_{k \le K} \sum_{i \in \mathcal{L}} y_i \I{i \in \texttt{Top}_k(f)} \\ 
            & = \frac{1}{\alpha K} \sum_{k \le K} \sum_{i \in \texttt{Top}_k(f)} y_i \\ 
            & = \frac{1}{\alpha K} \sum_{k \le K} (K + 1 - k) y_{\sigma(f, k)}, \\ 
        \end{aligned}
    \end{equation}
    where the last equation comes from
    \renewcommand{\arraystretch}{2}
    $$\begin{array}{ccl}
        k = 1, & \sum_{i \in \texttt{Top}_1(f)} y_i = & {\color{blue} y_{\sigma(f, 1)}}, \\
        k = 2, & \sum_{i \in \texttt{Top}_2(f)} y_i = & {\color{blue} y_{\sigma(f, 1)}} + {\color{red} y_{\sigma(f, 2)}}, \\
        \vdots & \vdots & {\color{blue} y_{\sigma(f, 1)}} + {\color{red} y_{\sigma(f, 2)}} + \cdots + {\color{orange} y_{\sigma(f, k)}}, \\
        k = K, & \sum_{i \in \texttt{Top}_K(f)} y_i = & {\color{blue} y_{\sigma(f, 1)}} + {\color{red} y_{\sigma(f, 2)}} + \cdots + {\color{orange} y_{\sigma(f, k)}} + \cdots + {\color{purple} y_{{\sigma(f, K)}}}. \\
    \end{array}$$
    Meanwhile, benefiting from the linear discount function, it is clear that 
    $$
        \texttt{IDCG-l@K}(\boldsymbol{y}) = \frac{1}{2} \cdot N_K(\boldsymbol{y}) (2K + 1 - N_K(\boldsymbol{y})) = N_{K}(\boldsymbol{y}) \tilde{N}_K(\boldsymbol{y}), 
    $$
    which ends the proof.
\end{proof}

\subsection{TKPR v.s. P@K and R@K (Proof of Thm.\ref{thm:conanddis})}
\label{app:conanddis}

\conanddis*
\begin{proof}
    Without loss of generality, we assume that $\alpha = 1$. Frist of all, the following listwise formulation is useful:
    $$
        \texttt{TKPR}(f, \boldsymbol{y})  = \frac{1}{K} \sum_{k \le K} (K + 1 - k) y_{\sigma(f, k)},
    $$
    whose derivation can be found in the proof of Prop.\ref{prop:tkprtondcg}. Then, we define the following partition number, where $D$ is written as the sum of $b$ distinct bounded positive integers:
    $$
        pn_K(a, b) := \size{\left\{ (r_1, r_2, \cdots, r_b) \in \mathbb{N}_{+}^{b} \mid a = r_1 + r_2 + \cdots + r_b, K \ge r_1 > r_2 > \cdots > r_b \ge 1 \right\}}.
    $$
    Besides, if $a \le 0$ or $b \le 0$, $pn_K(a, b) = 0$. Given the input sample $(\boldsymbol{x}, \boldsymbol{y})$, it is clear that $pn_K(a, b)$ is exactly the number of predictions $\boldsymbol{s}$ such that 
    $$
        \texttt{P@K}(\boldsymbol{s}, \boldsymbol{y}) = b / K, \texttt{TKPR}(\boldsymbol{s}, \boldsymbol{y}) = a / K.
    $$
    On one hand, we have 
    $$\begin{aligned}
        \size{P} & = \size{\{ (\boldsymbol{s}, \boldsymbol{s}') \mid \texttt{TKPR}(\boldsymbol{s}, \boldsymbol{y}) > \texttt{TKPR}(\boldsymbol{s}', \boldsymbol{y}), \texttt{P@K}(\boldsymbol{s}, \boldsymbol{y}) = \texttt{P@K}(\boldsymbol{s}', \boldsymbol{y}) \}} \\
        & = \sum_{b=0}^K \size{\{ (\boldsymbol{s}, \boldsymbol{s}') \mid \texttt{TKPR}(\boldsymbol{s}, \boldsymbol{y}) > \texttt{TKPR}(\boldsymbol{s}', \boldsymbol{y}), \texttt{P@K}(\boldsymbol{s}, \boldsymbol{y}) = \texttt{P@K}(\boldsymbol{s}', \boldsymbol{y}) = b / K \}} \\
        & = \sum_{b=0}^K \frac{1}{2} \left[ \tbinom{K}{b}^2 - \size{\{ (\boldsymbol{s}, \boldsymbol{s}') \mid \texttt{TKPR}(\boldsymbol{s}, \boldsymbol{y}) = \texttt{TKPR}(\boldsymbol{s}', \boldsymbol{y}), \texttt{P@K}(\boldsymbol{s}, \boldsymbol{y}) = \texttt{P@K}(\boldsymbol{s}', \boldsymbol{y}) = b / K \}} \right] \\
        & = \frac{1}{2} \sum_{b=0}^K \tbinom{K}{b}^2 - \frac{1}{2} \sum_{b=0}^K \sum_{a = 0}^{\overline{a}} pn_K(a, b)^2, \\
    \end{aligned}$$
    where $\tbinom{K}{b}$ denotes the combination number, which exactly equals to the number of predictions with $\texttt{P@K}(f, \boldsymbol{y}) = b / K$. 
    
    On the other hand, let $\overline{a} := \frac{K(K + 1)}{2}$ denote the maximum value of $K \cdot \texttt{TKPR}(f, \boldsymbol{y})$. And $ pn_K(a) := \sum_{b = 0}^ K pn_K(a, b)$ denotes the number of predictions whose $\texttt{TKPR}(f, \boldsymbol{y})$ equals to $a / K$. Then,
    $$\begin{aligned}
        \size{S} & = \size{\{ (\boldsymbol{s}, \boldsymbol{s}') \mid \texttt{TKPR}(\boldsymbol{s}, \boldsymbol{y}) = \texttt{TKPR}(\boldsymbol{s}', \boldsymbol{y}), \texttt{P@K}(\boldsymbol{s}, \boldsymbol{y}) > \texttt{P@K}(\boldsymbol{s}', \boldsymbol{y}) \}} \\
        & = \sum_{a = 0}^{\overline{a}} \size{\{ (\boldsymbol{s}, \boldsymbol{s}') \mid \texttt{TKPR}(\boldsymbol{s}, \boldsymbol{y}) = \texttt{TKPR}(\boldsymbol{s}', \boldsymbol{y}) = a / K, \texttt{P@K}(\boldsymbol{s}, \boldsymbol{y}) > \texttt{P@K}(\boldsymbol{s}', \boldsymbol{y}) \}} \\
        & = \sum_{a = 0}^{\overline{a}} \frac{1}{2} \left[ pn_K(a)^2 - \size{\{ (\boldsymbol{s}, \boldsymbol{s}') \mid \texttt{TKPR}(\boldsymbol{s}, \boldsymbol{y}) = \texttt{TKPR}(\boldsymbol{s}', \boldsymbol{y}) = a / K, \texttt{P@K}(\boldsymbol{s}, \boldsymbol{y}) = \texttt{P@K}(\boldsymbol{s}', \boldsymbol{y}) \}} \right]\\
        & = \frac{1}{2} \sum_{a=0}^{\overline{a}} pn_K(a)^2 - \frac{1}{2} \sum_{a=0}^{\overline{a}} \sum_{b=0}^{K} pn_K(a, b)^2,
    \end{aligned}$$
    Thus, we have 
    \begin{equation}
        \label{eq:p_s_dis}
        2 (\size{P} - \size{S}) = \underbrace{\sum_{b=0}^K \tbinom{K}{b}^2 }_{(I)} - \underbrace{\sum_{a=0}^{\overline{a}} pn_K(a)^2}_{(II)}.
    \end{equation}
    Notice that the following number equals to $pn_K(a, b)$, where $t_i = r_i - 1, i \in \{1, 2, \cdots, b\}$:
    $$
        \size{\left\{ (t_1, t_2, \cdots, t_b) \in \mathbb{N}^{b} \mid a - b = t_1 + t_2 + \cdots + t_b, K - 1 \ge t_1 > t_2 > \cdots > t_b \ge 0 \right\}}.
    $$
    Thus, we have the following recurrence formula:
    $$
        pn_K(a, b) = \underbrace{pn_{K - 1}(a - b, b)}_{{(a)}} + \underbrace{pn_{K - 1}(a - b, b-1)}_{(b)},
    $$
    where $(a)$ and $(b)$ correspond to the case where $t_b > 0$ and $t_b = 0$, respectively. Notice that in each recurrence, $K \to K - 1, a \to a - b, b \to \{b, b - 1\}$. Thus, by applying this recurrence iteratively, the coefficients will be exactly the Pascal's triangle:
    $$\begin{aligned}
        pn_K(a, b) & = {\color{blue}1} \cdot pn_{K - 1}(a - {\color{red}b}, {\color{orange}b}) + {\color{blue}1} \cdot pn_{K - 1}(a - {\color{red}b}, {\color{orange}b - 1}) \\
        & = {\color{blue}1} \cdot pn_{K - 2}(a - {\color{red}2b}, {\color{orange}b}) + {\color{blue}2} \cdot pn_{K - 2}(a - {\color{red}2b}, {\color{orange}b - 1}) + {\color{blue}1} \cdot pn_{K - 2}(a - {\color{red}2b}, {\color{orange}b - 2}) \\
        & = {\color{blue}1} \cdot pn_{K - 3}(a - {\color{red}3b}, {\color{orange}b}) + {\color{blue}3} \cdot pn_{K - 3}(a - {\color{red}3b}, {\color{orange}b - 1}) + {\color{blue}3} \cdot pn_{K - 3}(a - {\color{red}3b}, {\color{orange}b - 2}) + {\color{blue}1} \cdot pn_{K - 3}(a - {\color{red}3b}, {\color{orange}b - 3}) \\
        & \cdots \\
        & = \sum_{t=0}^{K - 1} {\color{blue} \tbinom{K - 1}{t}} \cdot pn_1(a - {\color{red}b(K - 1)}, {\color{orange}b - t}).
    \end{aligned}$$
    Although the last term consists of $K - 1$ combination numbers and $K - 1$ partition numbers, it is fortunate that only when $b - t = 1$, $pn_1(a - b(K - 1), b - t) > 0$. Thus, 
    $$
        pn_K(a) = \sum_{b=0}^K pn_K(a, b) = \sum_{b = 0}^K \tbinom{K - 1}{b - 1} pn_1(a - b(K - 1), 1).
    $$
    Similarly, only when $a - b(K - 1) = 1$, $pn_1(a - b(K - 1), 1) = 1$. Thus, for any $a$ such that $\exists b \in \mathbb{N}_+, a - b(K - 1) = 1$, 
    $$
        pn_K(a) = \tbinom{K - 1}{(a - K) / (K - 1)}.
    $$
    On top of this, we have
    $$
        (II) = \sum_{a=0}^{\overline{a}} pn_K(a)^2 = \sum_{b = 1}^{\lfloor K \rfloor / 2 + 1} \tbinom{K - 1}{b - 1}^2,
    $$
    where $\lfloor K \rfloor / 2 + 1$ comes from solve the equation $\overline{a} - b(K - 1) = 1$. Finally, since $\tbinom{K}{b} = \frac{K}{b} \tbinom{K - 1}{b - 1} \ge \tbinom{K - 1}{b - 1}$, and $K \ge \lfloor K \rfloor / 2 + 1$, we have $(I) \ge (II)$, where the equality holds only if $K = 1$. 
    
    Besides, this result is clearly applicable to \texttt{R@K} since the weighting term $\alpha$ does not affect the performance comparison between models. 
\end{proof}

\subsection{TKPR v.s. the ranking loss (Proof of Prop.\ref{prop:reformulation})}
\label{app:reformulation}
\reformulation*
\begin{proof}
    We consider the following two situations:
    \begin{itemize}
        \item \textbf{Case (a):} For each relevant label $y$ that is ranked at $k \le K$. In this case, we should not punish the term $\ell_{0-1}(s_y - s_{[k]})$:
        $$
            \sum_{k \le K} \I{\pi_{\boldsymbol{s}}(y) > k} = -1 + \sum_{k \le K} \I{s_y \le s_{[k]}}.
        $$
        Note that if $s_{[K]} = s_{[K + 1]}$, according to Asm.\ref{asm:wrongly_break_ties}, we have $s_y > s_{[K]} = s_{[K + 1]}$. And if $s_{[K]} > s_{[K + 1]}$, we have $s_y \ge s_{[K]} > s_{[K + 1]}$. Hence, we can conclude that $s_y > s_{[K + 1]}$ in this case.
        \item \textbf{Case (b):} For each relevant label $y$ that is ranked lower than $K$, it is clear that
        $$
            \sum_{k \le K} \I{\pi_{\boldsymbol{s}}(y) > k} = K = \sum_{k \le K} \I{s_y \le s_{[k]}}.
        $$
        Note that $s_y \le s_{[K + 1]}$ in this case.
    \end{itemize}
    To sum up, we have the following equation for any relevant labels $y$:
    $$
        \sum_{k \le K} \I{\pi_{\boldsymbol{s}}(y) > k} = - 1 + \I{s_y \le s_{[K + 1]}} + \sum_{k \le K} \I{s_y \le s_{[k]}} = - 1 + \sum_{k \le K + 1} \I{s_y \le s_{[k]}}.
    $$
    Thus, we have
    $$
        \sum_{y \in \mathcal{P}(\boldsymbol{y})} \sum_{k \le K} \I{\pi_{\boldsymbol{s}}(y) > k} = - N(\boldsymbol{y}) + \sum_{y \in \mathcal{P}(\boldsymbol{y})} \sum_{k \le K + 1} \I{s_y \le s_{[k]}} = - N(\boldsymbol{y}) + \sum_{y \in \mathcal{P}(\boldsymbol{y})} \sum_{k \le K + 1} \ell_{0-1}(s_y - s_{[k]}).
    $$
    Since $N(\boldsymbol{y})$ is independent with the score function, the two formulations are essentially equivalent.
    
    Furthermore, it is clear that 
    $$\frac{1}{K} \sum_{k \le K + 1} s_{[k]} \ge \frac{1}{K} \sum_{k \le K} s_{[k]} \ge \frac{1}{N_{-}(\boldsymbol{y})} \sum_{j \in \mathcal{N}(\boldsymbol{y})} s_j,$$
    Since $\ell_{0-1}$ is non-increasing, we have 
    $$
        L_\text{rank}(f, \boldsymbol{y}) \le \frac{1}{N(\boldsymbol{y}) K} \sum_{y \in \mathcal{P}(\boldsymbol{y})} \sum_{k \le K + 1} \ell_{0-1}(s_y - s_{[k]}) = \frac{\alpha}{N(\boldsymbol{y})} L_{K}^{\alpha}(f, \boldsymbol{y}).
    $$
\end{proof}

\subsection{TKPR v.s. AP@K (Proof of Thm.\ref{thm:ndcgboundmap})}
\label{app:ndcgboundmap}
\begin{lemma}
    \label{lem:ln_lip}
    For $0 \le \beta \le x \le y$, $\ln{y} - \ln{x} \le \frac{1}{\beta} (y - x)$.
\end{lemma}
\begin{proof}
    It is clear that $$\ln{y} - \ln{x} = \ln{(1 + \frac{y}{x} - 1)} \le \frac{y}{x} - 1 = \frac{1}{x} \cdot (y - x) \le \frac{1}{\beta} \cdot (y - x).$$
\end{proof}

\ndcgboundmap*
\begin{proof}
    For the left inequality, given $(\boldsymbol{x}, \boldsymbol{y}) \in \mathcal{Z}$, let $p := K \cdot \texttt{P@K}(f, \boldsymbol{y}) > 0$ and $\mathcal{K} := \{k_i\}_{i=1}^{p}$ such that $1 \le k_1 < k_2 < \cdots < k_p \le K$ and $\forall k \in \mathcal{K}, y_{ \sigma(f, k)} = 1$. Then, we have 
    \begin{equation}
        \label{eq:reformulate_ap}
        \begin{aligned}
            \texttt{AP@K}(f, \boldsymbol{y}) & = \frac{1}{N_K(\boldsymbol{y})} \sum_{k=1}^{K} \frac{y_{\sigma(f, k)}}{k} \sum_{i=1}^{k} y_{\sigma(f, i)} \\
            & \overset{(a)}{=} \frac{1}{N_K(\boldsymbol{y})} \sum_{k=1}^{K} \left[ \frac{y_{\sigma(f, k)}}{k} + \frac{y_{\sigma(f, k+1)}}{k + 1} + \cdots + \frac{y_{\sigma(f, K)}}{K} \right] y_{\sigma(f, k)} \\
            & \overset{(b)}{=} \frac{1}{N_K(\boldsymbol{y})} \sum_{i = 1}^{p} \sum_{j=k_i}^K \frac{y_{\sigma(f, j)}}{j} \\
            & \overset{(c)}{=} \frac{1}{N_K(\boldsymbol{y})} \sum_{i = 1}^{p} \sum_{j=i}^p \frac{1}{k_j} \\
            & \overset{(d)}{=} \frac{1}{N_K(\boldsymbol{y})} \sum_{i = 1}^{p} \frac{i}{k_i},
        \end{aligned}
    \end{equation}
    where $(a)$ holds since
    \renewcommand{\arraystretch}{2}
    $$\begin{array}{ccl}
        k = 1, & \frac{y_{\sigma(f, 1)}}{1}: & {\color{blue} y_{\sigma(f, 1)}}, \\
        k = 2, & \frac{y_{\sigma(f, 2)}}{2}: & {\color{blue} y_{\sigma(f, 1)}} + {\color{red} y_{\sigma(f, 2)}}, \\
        \vdots & \vdots & {\color{blue} y_{\sigma(f, 1)}} + {\color{red} y_{\sigma(f, 2)}} + \cdots + {\color{orange} y_{\sigma(f, k)}}, \\
        k = K, & \frac{y_{\sigma(f, K)}}{K}: & {\color{blue} y_{\sigma(f, 1)}} + {\color{red} y_{\sigma(f, 2)}} + \cdots + {\color{orange} y_{\sigma(f, k)}} + \cdots + {\color{purple} y_{{\sigma(f, K)}}}, \\
    \end{array}$$
    $(b)$ is induced by replacing $\sum_{k=1}^K y_{\sigma(f, k)}$ with $\sum_{i=1}^p y_{\sigma(f, k_i)}$, $(c)$ is the same as $(b)$, and $(d)$ follows the fact
    $$\begin{array}{cr}
        i = 1, & {\color{blue} \frac{1}{k_1}} + {\color{red} \frac{1}{k_2}} + \cdots + {\color{orange} {\frac{1}{k_j}}} + \cdots + {\color{purple} \frac{1}{k_p}}, \\
        i = 2, & {\color{red} \frac{1}{k_2}} + \cdots + {\color{orange} {\frac{1}{k_j}}} + \cdots + {\color{purple} \frac{1}{k_p}}, \\
        \vdots & {\color{orange} {\frac{1}{k_j}}} + \cdots + {\color{purple} \frac{1}{k_p}}, \\
        i = p, & {\color{purple} \frac{1}{k_p}}. \\
    \end{array}$$
    Meanwhile, we have
    $$\begin{aligned}
        \texttt{TKPR}^{\alpha_2}(f, \boldsymbol{y}) & = \frac{1}{N_{K}(\boldsymbol{y})  K} \sum_{k = 1}^{K} (K + 1 - k) y_{\sigma(f, k)} \\
        & = \frac{1}{N_K(\boldsymbol{y})} \sum_{k \in \mathcal{K}} \frac{K + 1 - k}{ K} \\
        & = \frac{1}{N_K(\boldsymbol{y})} \sum_{i = 1}^{p} \frac{K + 1 - k_i}{ K}.
    \end{aligned}$$
    Since $\texttt{AP@K}(f, \boldsymbol{y})$ and $\texttt{TKPR}^{\alpha_2}(f, \boldsymbol{y})$ have the same number of bounded summation terms, there must exist a constant $\rho > 0$ such that 
    $$
        \texttt{AP@K}(f, \boldsymbol{y}) \ge \rho \cdot \texttt{TKPR}^{\alpha_2}(f, \boldsymbol{y}).
    $$
    To obtain the range of $\rho$'s upper bounds, let
    $$
        C_1 := \frac{\sum_{i=1}^{p} \frac{i}{k_i}}{\sum_{i=1}^{p} \frac{1}{k_i}} \in [1, p], C_2 := \frac{1}{p} \sum_{i=1}^{p} k_i \in [k_1, k_p].
    $$
    Then, 
    $$\begin{aligned}
        \texttt{AP@K}(f, \boldsymbol{y}) - \rho \cdot \texttt{TKPR}^{\alpha_2}(f, \boldsymbol{y}) & = \frac{1}{N_K(\boldsymbol{y})} \left[ C_1 \sum_{i=1}^{p} \frac{1}{k_i} - \rho \cdot \frac{p(K + 1 - C_2)}{ K} \right] \\
        & \ge \frac{1}{N_K(\boldsymbol{y})} \left[ C_1 \sum_{i=1}^{p} \frac{1}{K + 1 - i} - \rho \cdot \frac{p(K + 1 - C_2)}{ K} \right] \\
        & \overset{(a)}{\ge} \frac{1}{N_K(\boldsymbol{y})} \left[ C_1 \ln{\frac{K + 1}{K + 1 - p}} - \rho \cdot \frac{p(K + 1 - C_2)}{ K} \right], \\
    \end{aligned}$$
    where $(a)$ is induced by the fact that for any $1 \le k \le K$,
    $$
        \frac{1}{k} + \cdots + \frac{1}{K} \ge \int_{k}^{K + 1} \frac{1}{t} \, dt = \ln{(K + 1)} - \ln{k} = \ln \frac{K + 1}{k}.
    $$
    In other words, when
    $$\begin{aligned}
        \rho & \le \frac{ C_1 \cdot  K }{ p (K + 1 - C_2) } \ln{\frac{K + 1}{K + 1 - p}} \\
        & = \underbrace{\frac{ C_1 }{ K + 1 - C_2 }}_{U_1} \cdot \underbrace{\frac{1}{\texttt{P@K}(f, \boldsymbol{y})} \ln{\frac{1 + 1/K}{1 + 1/K - \texttt{P@K}(f, \boldsymbol{y})}}}_{U_2}, \\
    \end{aligned}$$
    $\texttt{AP@K}(f, \boldsymbol{y}) \ge \rho \cdot \texttt{TKPR}^{\alpha_2}(f, \boldsymbol{y})$ will hold. Next, we consider this upper bound from the two parts: 
    \begin{itemize}
        \item $U_1$ is increasing \textit{w.r.t.} $C_1$ and $C_2$ that are determined by $\mathcal{K}$. Since $C_1 \in [1, p], C_2 \in [k_1, k_p]$, we have 
        $$
            \frac{1}{K} \le \frac{1}{K + 1 - k_1} \le (II) \le \frac{p}{K + 1 - k_p} \le p \le K.
        $$
        \item $U_2$ depends on the model performance on \texttt{P@K}. Let 
        $$
            g(t) := \frac{1}{t} \ln\frac{1}{1 - C_g t}, t \in (0, 1], C_g := \frac{1}{1 + \frac{1}{K}} \in [0.5, 1).
        $$ 
        We have 
        $$
            g'(t) = \frac{1}{t^2} \cdot \ln(1 - C_g t) + \frac{1}{t} \cdot \frac{C_g}{1 - C_g t} = \frac{(1 - C_g t) \ln(1 - C_g t) + C_g t}{t^2 (1 - C_g t)}.
        $$
        Let $$
            h(t) = (1 - t) \ln(1 - t) + t, t \in [0, 1).
        $$ 
        Then, we have 
        $$
            h'(t) = - \ln(1 - t) - 1 + 1 \ge 0.
        $$
        Hence, $h(t) \ge h(0) = 0$, that is, $g'(t) \ge 0$. In other words, $U_2$ is increasing \textit{w.r.t.} $\texttt{P@K}(f, \boldsymbol{y})$. As $t \to 0$, $g(t) \to C_g = \frac{K}{1 + K}$. And $g(1) = \ln\frac{1}{1 - C_g} = \ln(K + 1)$.
    \end{itemize}
    To sum up, the upper bound of $\rho$ is bounded in $\left[ 1 / (K + 1), K \ln(K + 1) \right]$.

    \noindent \rule[2pt]{\linewidth}{0.1em}
    For the right inequality, 
    \begin{equation}
        \begin{aligned}
            \texttt{AP@K}(f, \boldsymbol{y}) & = \frac{1}{N_K(\boldsymbol{y})} \sum_{k=1}^{K} y_{\sigma(f, k)} \cdot \frac{1}{k} \sum_{i=1}^{k} y_{\sigma(f, i)} \\
            & \le \frac{1}{N_K(\boldsymbol{y})} \sum_{k=1}^{K} \frac{1}{k} \sum_{i=1}^{k} y_{\sigma(f, i)} \\
            & \overset{(a)}{=} \frac{1}{N_K(\boldsymbol{y})} \sum_{k=1}^{K} \left[ \frac{1}{k} + \frac{1}{k + 1} + \cdots + \frac{1}{K} \right] y_{\sigma(f, k)} \\
            & \overset{(b)}{\le} \frac{1}{N_K(\boldsymbol{y})} \left[(1 + \ln{K}) y_{\sigma(f, 1)} + \sum_{k=2}^{K} \left[ \ln{K} - \ln{(k-1)} \right] y_{\sigma(f, k)} \right]  \\
            & \overset{(c)}{\le} \frac{1}{N_K(\boldsymbol{y})} \left[K \cdot y_{\sigma(f, 1)} + \sum_{k=2}^{K} \left[ K - (k-1) \right] y_{\sigma(f, k)} \right]  \\
            & = \frac{1}{N_K(\boldsymbol{y})} \left[\sum_{k=1}^{K} \left[ K + 1 - k \right] y_{\sigma(f, k)} \right]  \\
            & = K \cdot \texttt{TKPR}^{\alpha_1}(f, \boldsymbol{y})) \\
        \end{aligned}
    \end{equation}
    where $(a)$ follows the derivation of Eq.(\ref{eq:reformulate_ap}), $(b)$ comes from the fact that for any $k \in \{2, 3, \cdots, K\}$,
    $$\begin{aligned}
        \frac{1}{k} + \cdots + \frac{1}{K} & \le \int_{k-1}^{K} \frac{1}{t} \, dt = \ln{K} - \ln{(k-1)}, \\
        1 + \frac{1}{2} + \cdots + \frac{1}{K} & \le 1 + \ln{K} - \ln{(2-1)} = 1 + \ln{K}, \\
    \end{aligned}$$ 
    and $(c)$ is induced by Lem.\ref{lem:ln_lip} and the fact that $x \ge 1 + \ln x$.
\end{proof}

\textbf{Further discussion about the monotonicity of the upper bound of $\rho$.} So far, we have known that $U_1$ is increasing \textit{w.r.t.} $\texttt{P@K}(f, \boldsymbol{y})$. Next, we analyze the monotonicity of $U_2$. Note that two factors of $\mathcal{K}$ can affect the value of $U_2$: the number of elements $p$ and the value of each $k_i$. Hence, we next consider the following two orthogonal cases:
\begin{itemize}
    \item Two score functions share the same ranking performance but difference performances on \texttt{P@K}. Formally, let $\mathcal{K}' := \{ k'_i \}_{i=1}^{p+1}$ such that $\forall i \le p, k_i = k'_i < k'_{p+1}$. On one hand, since $k'_{p+1} / 1 = k'_{p+1} > \sum_{i=1}^{p} k_i / p$,
    $$
        C_2(\mathcal{K}') - C_2(\mathcal{K}) = \frac{k'_{p+1} + \sum_{i=1}^{p} k_i}{1 + p} - \frac{\sum_{i=1}^{p} k_i}{p} > 0.
    $$
    On the other hand, since $\frac{p+1}{k'_{p + 1}} / \frac{1}{k'_{p + 1}} = p + 1 > \sum_{i=1}^{p} \frac{i}{k_i} / \sum_{i=1}^{p} \frac{1}{k_i}$,
    $$
        C_1(\mathcal{K}') - C_1(\mathcal{K}) =  \frac{\frac{p+1}{k'_{p + 1}} + \sum_{i=1}^{p} \frac{i}{k_i}}{\frac{1}{k'_{p + 1}} + \sum_{i=1}^{p} \frac{1}{k_i}} - \frac{\sum_{i=1}^{p} \frac{i}{k_i}}{\sum_{i=1}^{p} \frac{1}{k_i}} > 0.
    $$
    Since $U_2$ is increasing \textit{w.r.t.} $C_1$ and $C_2$, we have $U_2(\mathcal{K}') > U_2(\mathcal{K})$. In other words, $U_2$ is increasing \textit{w.r.t.} $\texttt{P@K}(f, \boldsymbol{y})$ under the given the ranking performance.
    \item Two score functions share the same performance on \texttt{P@K} but difference ranking performance. It is a pity that in this case, \ul{$U_2$ is not necessarily increasing \textit{w.r.t.} the ranking performance}. For example, assume that $p = 1, K \ge 2$, and let $\mathcal{K}_1 = \{1\}, \mathcal{K}_2 = \{2\}$. It is clear that $\mathcal{K}_1$ has a better ranking performance. However, we have 
    $$
        U_2(\mathcal{K}_1) = \frac{1}{K} < \frac{1}{K - 1} = U_2(\mathcal{K}_2).
    $$
    Meanwhile, $U_2$ is not necessarily decreasing \textit{w.r.t.} the ranking performance. Assume that $p = 2, K = 10$, and let $\mathcal{K}_3 = \{1, 2\}, \mathcal{K}_4 = \{1, 3\}$. It is clear that $\mathcal{K}_3$ has a better ranking performance. Meanwhile, 
    $$\begin{aligned}
        U_2(\mathcal{K}_3) & = \frac{\frac{1}{1} + \frac{2}{2}}{\left( \frac{1}{1} + \frac{1}{2} \right) \left( 10 + 1 - \frac{1 + 2}{2} \right) } = \frac{2}{1.5 \cdot 9.5} \approx 0.1404, \\
        U_2(\mathcal{K}_4) & = \frac{\frac{1}{1} + \frac{2}{3}}{\left( \frac{1}{1} + \frac{1}{3} \right) \left( 10 + 1 - \frac{1 + 3}{2} \right) } = \frac{5}{36} \approx 0.1389. \\
    \end{aligned}$$
    Thus, we have $U_2(\mathcal{K}_3) > U_2(\mathcal{K}_4)$.
\end{itemize}
To sum up, as one improves the precision performance, $U_2$ tends to increase. But when one improves the ranking performance under the given precision performance, $U_2$ does not necessarily increase.

\section{Consistency Analysis of TKPR Optimization}
\subsection{Bayes Optimality of TKPR (Proof of Prop.\ref{prop:bayes})}
\label{app:byestkpr}
\begin{lemma}[Rearrangement inequality \cite{hardy1952inequalities}]
    \label{lem:rearr_ineq}
    For any two real number sets $\{a_i\}_{i = 1}^n$ and $\{b_i\}_{i = 1}^n$, 
    $$
        \sum_{i = 1}^n a_{[i]} b_{[i]} \ge \sum_{i = 1}^n a_i b_i.
    $$
\end{lemma}

\bayestkpr*
\begin{proof}
    We consider the conditional TKPR risk:
    \begin{equation}
        \begin{aligned}
            \mathcal{R}_K^\alpha(f \mid \boldsymbol{x}) & := \E{\boldsymbol{y} \mid \boldsymbol{x}}{\frac{1}{\alpha K} \sum_{i \in \mathcal{L}} \sum_{k=1}^{K} y_i \I{i \notin \texttt{Top}_k(f)}} \\
            & = \E{\boldsymbol{y} \mid \boldsymbol{x}}{ \frac{1}{\alpha K} \sum_{i \in \mathcal{L}} \sum_{k=1}^{K} y_i \left[ 1 - \I{i \in \texttt{Top}_k(f)} \right] } \\
            & = \frac{1}{K} \sum_{k=1}^K \sum_{i \in \mathcal{L}} \E{\boldsymbol{y} \mid \boldsymbol{x}}{\frac{y_i}{\alpha}} - \frac{1}{K} \E{\boldsymbol{y} \mid \boldsymbol{x}}{\sum_{k=1}^{K} \sum_{i \in \texttt{Top}_k(f)} \frac{y_i}{\alpha}} \\
            & = \sum_{i \in \mathcal{L}} \E{\boldsymbol{y} \mid \boldsymbol{x}}{\frac{y_i}{\alpha}} - \frac{1}{K} \sum_{k=1}^{K} \sum_{i \in \texttt{Top}_k(f)} \E{\boldsymbol{y} \mid \boldsymbol{x}}{\frac{y_i}{\alpha}} \\
            & = \sum_{i \in \mathcal{L}} \Delta(\boldsymbol{x})_i - \frac{1}{K} \sum_{k=1}^{K} \sum_{i \in \texttt{Top}_k(f)} \Delta(\boldsymbol{x})_i \\
            & = \sum_{i \in \mathcal{L}} \Delta(\boldsymbol{x})_i - \frac{1}{K} \sum_{k=1}^{K} (K + 1 - k) \Delta(\boldsymbol{x})_{\sigma(f, k)}, \\
        \end{aligned}
    \end{equation}
    where the last equation follows the derivation of Eq.(\ref{eq:tkprtondcg}).

    Note that $K + 1 - k$ is decreasing \textit{w.r.t.} $k$. Thus, when no ties exist in $\Delta(\boldsymbol{x})_i$, according to Lem.\ref{lem:rearr_ineq}, it is clear that the optimal solution $f^*$ should satisfy
    $$
        \forall k \le K, \sigma(f^*, k) = \sigma(\Delta, k).
    $$
    When ties exists, that is, $\exists k \in \mathcal{L}, \Delta(\boldsymbol{x})_{\sigma(f, k)} = \Delta(\boldsymbol{x})_{\sigma(f, k + 1)}$, $f^*$ can further exchange the value of $f(\boldsymbol{x})_{[k]}$ and $f(\boldsymbol{x})_{[k+1]}$. Since $\size{\texttt{Top}_K(\Delta) \cup \texttt{Tie}_K(\Delta)}$ might be greater than $K$, under Asm.\ref{asm:wrongly_break_ties}, $f^*$ should satisfy
    $$
        \forall k \le K - 1, \texttt{Tie}_k(f^*) = \texttt{Tie}_k(\Delta), \texttt{Tie}_K(f^*) \subset \texttt{Tie}_K(\Delta).
    $$
\end{proof}

\subsection{Sufficient Condition for TKPR Consistency (Proof of Thm.\ref{thm:condition_for_consistency})}
\label{app:suffcondicons}
\conditionforconsistency*
\begin{proof}
    We first define the conditional risk and the optimal conditional risk of the surrogate loss:
    \begin{equation}
        \begin{aligned}
            \mathcal{R}_K^{\alpha, \ell} (\boldsymbol{s} \mid \boldsymbol{x}) & := \E{\boldsymbol{y} \mid \boldsymbol{x}} {\frac{1}{\alpha K} \sum_{y \in \mathcal{P}(\boldsymbol{x})}\sum_{k \le K+1} \ell \left(s_y - s_{[k]}\right)}, \\
            \mathcal{R}_K^{\alpha, \ell, *} (\boldsymbol{x}) & := \inf_{\boldsymbol{s} \in \mathbb{R}^C} \mathcal{R}_K^{\alpha, \ell} (\boldsymbol{s} \mid \boldsymbol{x}). \\
        \end{aligned}
    \end{equation}
    Then, we prove the theorem by the following steps.

    \noindent \rule[2pt]{\linewidth}{0.1em}
    \textbf{Claim 1.} If $\boldsymbol{s}^* \in \arg \inf_{\boldsymbol{s}} \mathcal{R}_K^{\alpha, \ell}(\boldsymbol{s} \mid \boldsymbol{x})$, then $\mathsf{RPT}(\boldsymbol{s}^*, \Delta(\boldsymbol{x}))$.

    \noindent \rule[2pt]{\linewidth}{0.1em}
    According to the definition, we have
    \begin{equation}
        \begin{aligned}
            \mathcal{R}_K^{\alpha, \ell} (\boldsymbol{s} \mid \boldsymbol{x}) & =  \frac{1}{K} \sum_{\boldsymbol{y} \in \mathcal{Y}} \frac{\pp{\boldsymbol{y} \mid \boldsymbol{x}}}{\alpha} \sum_{i \in \mathcal{P}(\boldsymbol{x})}\sum_{k \le K+1} \ell \left(s_i - s_{[k]}\right) \\
            & =  \frac{1}{K} \sum_{i \in \mathcal{L}} \sum_{k \le K+1} \ell \left(s_i - s_{[k]}\right) \sum_{\boldsymbol{y}: y_i = 1} \frac{\pp{\boldsymbol{y} \mid \boldsymbol{x}}}{\alpha} \\
            & =  \frac{1}{K} \sum_{i \in \mathcal{L}} \sum_{k \le K+1} \Delta_i \ell \left(s_i - s_{[k]}\right) \\
        \end{aligned}
    \end{equation}
    Next, we show this claim by showing that if $\lnot \mathsf{RPT}(\boldsymbol{s}, \Delta)$, then $\mathcal{R}_K^{\alpha, \ell} (\boldsymbol{s} \mid \boldsymbol{x}) > \mathcal{R}_K^{\alpha, \ell, *} (\boldsymbol{s} \mid \boldsymbol{x})$. Note that $\lnot \mathsf{RPT}(\boldsymbol{s}, \Delta)$ consists of the following cases:
    \begin{itemize}
        \item \textbf{Case (1):} Give $\boldsymbol{s} \in \mathbb{R}^C$ and $i, j \in \mathcal{L}$ such that $\Delta_i = \Delta_j$ but $s_i \neq s_j$, where $\pi_{\boldsymbol{s}}(i), \pi_{\boldsymbol{s}}(j) \le K + 1$. Without loss of generality, we assume that $s_i < s_j$. Then, the claim can be obtained by a contradiction. To be specific, we assume that $\mathcal{R}_K^{\alpha, \ell}(\boldsymbol{s} \mid \boldsymbol{x}) = \mathcal{R}_K^{\alpha, \ell, *}(\boldsymbol{s} \mid \boldsymbol{x})$. According to the first-order condition, we have
        $$
            \frac{\partial}{\partial s_i} \mathcal{R}_K^{\alpha, \ell}(\boldsymbol{s} \mid \boldsymbol{x}) = \frac{\partial}{\partial s_j} \mathcal{R}_K^{\alpha, \ell}(\boldsymbol{s} \mid \boldsymbol{x}) = 0.
        $$
        That is, 
        $$\begin{aligned}
            \underbrace{\Delta_i \sum_{k=1, [k] \neq i}^{K+1} \ell'(s_{i} - s_{[k]})}_{(I)} & = \underbrace{\sum_{y \neq i} \Delta_{y} \ell'(s_y - s_{i})}_{(II)}; \\
            \underbrace{\Delta_j \sum_{k=1, [k] \neq j}^{K+1} \ell'(s_{j} - s_{[k]})}_{(III)} & = \underbrace{\sum_{y \neq j} \Delta_{y} \ell'(s_y - s_{j})}_{(IV)},
        \end{aligned}$$
        where $[k] \neq i$ means that the calculation of the derivative will be skipped when $\pi_{\boldsymbol{s}}(i) = k$. Since $\Delta_i = \Delta_j$, we have
        $$\begin{aligned}
            (I) & = \Delta_i \ell'(s_i - s_j) + \Delta_i \sum_{k=1, [k] \notin \{i, j\}}^{K+1} \ell'(s_i - s_{[k]}), \\
            (II) & = \Delta_i\ell'(s_j - s_i) + \sum_{y \notin \{i, j\}} \Delta_{y} \ell'(s_y - s_i), \\
            (III) & = \Delta_i \ell'(s_j - s_i) + \Delta_i \sum_{k=1, [k] \notin \{i, j\}}^{K+1} \ell'(s_j - s_{[k]}), \\
            (IV) & = \Delta_i\ell'(s_i - s_j) + \sum_{y \notin \{i, j\}} \Delta_{y} \ell'(s_y - s_j). \\
        \end{aligned}$$
        Then, since $(I) - (III) = (II) - (IV)$, we have 
        $$\begin{aligned}
            \underbrace{2\Delta_i \ell'(s_i - s_j)}_{(I')} & + \underbrace{\Delta_i \sum_{k=1, [k] \notin \{i, j\}}^{K+1} \left[ \ell'(s_i - s_{[k]}) - \ell'(s_j - s_{[k]}) \right]}_{(II')} \\
            & = \underbrace{2 \Delta_i\ell'(s_j - s_i)}_{(III')} + \underbrace{\sum_{y \notin \{i, j\}} \Delta_{y} \left[ \ell'(s_y - s_i) - \ell'(s_y - s_j) \right]}_{(IV')}. 
        \end{aligned}$$
        Since $\ell'(t)$ is strictly increasing and $s_i < s_j$, it is clear that $(I') < (III'), (II') < 0$, and $(IV') > 0$. That is, $(I') + (II') < (III') + (IV')$, which induces contradiction.
        \item \textbf{Case (2):} Give $\boldsymbol{s} \in \mathbb{R}^C$ and $i, j \in \mathcal{L}$ such that $\Delta_i \neq \Delta_j$ but $s_i = s_j$, where $\pi_{\boldsymbol{s}}(i), \pi_{\boldsymbol{s}}(j) \le K + 1$. The claim can also be obtained by a contradiction. Similarly, According to the first-order condition and $s_i = s_j$, we have 
        $$\begin{aligned}
            (I) & = \Delta_i \ell'(0) + \Delta_i \sum_{k=1, [k] \notin \{i, j\}}^{K+1} \ell'(s_i - s_{[k]}), \\
            (II) & = \Delta_j \ell'(0) + \sum_{y \notin \{i, j\}} \Delta_{y} \ell'(s_y - s_i), \\
            (III) & = \Delta_j \ell'(0) + \Delta_j \sum_{k=1, [k] \notin \{i, j\}}^{K+1} \ell'(s_i - s_{[k]}), \\
            (IV) & = \Delta_i\ell'(0) + \sum_{y \notin \{i, j\}} \Delta_{y} \ell'(s_y - s_i). \\
        \end{aligned}$$
        Then, since $(I) - (III) = (II) - (IV)$, we have 
        $$
            (\Delta_i - \Delta_j) \sum_{k=1, [k] \notin \{i, j\}}^{K+1} \ell'(s_i - s_{[k]}) = 2 (\Delta_j - \Delta_i) \ell'(0)
        $$
        which leads to a contradiction since $\ell(t) < 0$ and $\Delta_i \neq \Delta_j$.
        \item \textbf{Case (3):} Give $\boldsymbol{s}_1 \in \mathbb{R}^C$ and $i, j \in \mathcal{L}$ such that $\Delta_i < \Delta_j$ but $s_{1, i} > s_{1, j}$, where $\pi_{\boldsymbol{s}}(i), \pi_{\boldsymbol{s}}(j) \le K + 1$. Next, we obtain the claim by showing that any swapping breaking the ranking of $\mathcal{T}_{\Delta}$ will induce a larger conditional risk. To be specific, let $\boldsymbol{s}_2 \in \mathbb{R}^C$ sucht that $s_{2, i} = s_{1, j}, s_{2, j} = s_{1, i}$ and $\forall k \notin \{i, j\}, s_{2, k} = s_{1, k}$. Then, we have 
        $$\begin{aligned}
            & K \left[ \mathcal{R}_K^{\alpha, \ell} (\boldsymbol{s}_1 \mid \boldsymbol{x}) - \mathcal{R}_K^{\alpha, \ell} (\boldsymbol{s}_2 \mid \boldsymbol{x})\right] \\
            & = \left[ \Delta_{i} \sum_{k=1}^{K+1} \ell(s_{1, i} - s_{1, [k]}) + \Delta_{j} \sum_{k=1}^{K+1} \ell(s_{1, j} - s_{1, [k]}) \right] - \left[ \Delta_{i} {\color{orange} \sum_{k=1}^{K+1} \ell(s_{2, i} - s_{2, [k]})} + \Delta_{j} {\color{blue} \sum_{k=1}^{K+1} \ell(s_{2, j} - s_{2, [k]}) }  \right] \\
            & = \left[ \Delta_{i} \sum_{k=1}^{K+1} \ell(s_{1, i} - s_{1, [k]}) + \Delta_{j} \sum_{k=1}^{K+1} \ell(s_{1, j} - s_{1, [k]}) \right] - \left[ \Delta_{i} {\color{orange} \sum_{k=1}^{K+1} \ell(s_{1, j} - s_{1, [k]})} + \Delta_{j} {\color{blue} \sum_{k=1}^{K+1} \ell(s_{1, i} - s_{1, [k]}) }  \right] \\
            & = (\Delta_j - \Delta_i) \sum_{k=1}^{K+1} \left[ \ell(s_{1, j} - s_{1, [k]}) - \ell(s_{1, i} - s_{1, [k]}) \right] \\
            & > 0,
        \end{aligned}$$
        where the inequality is induced by $\Delta_j > \Delta_i, s_{1, j} < s_{2, j}$ and the surrogate loss $\ell$ is strictly decreasing. 
    \end{itemize}
    Given any $\boldsymbol{a} \in \mathbb{R}^C$, let $\mathcal{T}_{\boldsymbol{a}} := \{\texttt{Tie}_{k}(\boldsymbol{a})\}_{k=1}^{K + 1} = \{\mathcal{T}_1, \mathcal{T}_2, \cdots, \mathcal{T}_{\size{\mathcal{T}_{\boldsymbol{a}}}}\}$ denote the set of tie sets. Note that $\size{\mathcal{T}_{\boldsymbol{a}}} \le C$ due to the possible ties in $\boldsymbol{a}$. From the analysis in \textbf{Case (1)} and \textbf{Case (2)}, we know that $\mathcal{T}_{\boldsymbol{s}^*} = \mathcal{T}_{\Delta}$. Without loss of generality, we assume that $\forall i \le \size{\mathcal{T}_\Delta}, \mathcal{T}_{\Delta, i} = \mathcal{T}_{\boldsymbol{s}^*, i}$. Definet the partial ranking between tie sets as 
    $$
        \mathcal{T}_i \prec \mathcal{T}_j \Leftrightarrow \forall y_1 \in \mathcal{T}_i, y_2 \in \mathcal{T}_j, s_{y_1} < s_{y_2}.
    $$
    Then, from the analysis in \textbf{Case (3)}, we know that if $\mathcal{T}_{\Delta, 1} \prec \mathcal{T}_{\Delta, 2} \prec \cdots \prec \mathcal{T}_{\Delta, \size{\mathcal{T}_{\Delta}}}$, $\mathcal{T}_{\boldsymbol{s}^*, 1} \prec \mathcal{T}_{\boldsymbol{s}^*, 2} \prec \cdots \prec \mathcal{T}_{\boldsymbol{s}^*, \size{\mathcal{T}_{\boldsymbol{s}^*}}}$, which obtains the claim.

    \noindent \rule[2pt]{\linewidth}{0.1em}
    \textbf{Claim 2.}
    $$
        \inf_{\boldsymbol{s}: \lnot \mathsf{RPT}(\boldsymbol{s},  \boldsymbol{x})} \mathcal{R}_K^{\alpha, \ell}(\boldsymbol{s} \mid \boldsymbol{x}) > \inf_{\boldsymbol{s}: \mathsf{RPT}(\boldsymbol{s},  \boldsymbol{x})} \mathcal{R}_K^{\alpha, \ell}(\boldsymbol{s} \mid \boldsymbol{x})
    $$
    
    \noindent \rule[2pt]{\linewidth}{0.1em}
    It is clear that \textbf{Claim 2} follows \textbf{Claim 1}.

    \noindent \rule[2pt]{\linewidth}{0.1em}
    \textbf{Claim 3.} For any sequence $\{\boldsymbol{s}^{(t)}\}_{t \in \mathbb{N}_+}$, $\boldsymbol{s}^{(t)} \in \mathbb{R}^C$,
    $$
        \mathcal{R}_K^{\alpha, \ell}(\boldsymbol{s}^{(t)}\mid \boldsymbol{x}) \to \inf_{\boldsymbol{s} \in \mathbb{R}^C} \mathcal{R}_K^{\alpha, \ell}(\boldsymbol{s}\mid \boldsymbol{x}) \Rightarrow \mathcal{R}_{K}^\alpha(\boldsymbol{s}^{(t)}\mid \boldsymbol{x}) \to \inf_{\boldsymbol{s} \in \mathbb{R}^C} \mathcal{R}_{K}^\alpha(\boldsymbol{s}\mid \boldsymbol{x}).
    $$
    \noindent \rule[2pt]{\linewidth}{0.1em}
    According to Prop.\ref{prop:bayes} and \textbf{Claim 1}, we only need to show that when $t \to \infty$, $\mathsf{RPT}(\boldsymbol{s}^{(t)}, \Delta(\boldsymbol{x}))$. Define 
    $$
        \delta := \inf_{\boldsymbol{s}: \lnot \mathsf{RPT}(\boldsymbol{s}, \Delta({x}))} \mathcal{R}_K^{\alpha, \ell}(\boldsymbol{s} \mid \boldsymbol{x}) - \inf_{\boldsymbol{s} \in \mathbb{R}^C} \mathcal{R}_K^{\alpha, \ell}(\boldsymbol{s} \mid \boldsymbol{x}).
    $$
    According to \textbf{Claim 2}, $0 < \delta < \infty$. Suppose that when $t \to \infty$, $\lnot \mathsf{RPT}(\boldsymbol{s}^{(t)}, \Delta(\boldsymbol{x}))$. Then, there exists a large enough $T$ such that
    $$
        \mathcal{R}_K^{\alpha, \ell}(\boldsymbol{s}^{(t)} \mid \boldsymbol{x}) - \inf_{\boldsymbol{s} \in \mathbb{R}^C} \mathcal{R}_K^{\alpha, \ell}(\boldsymbol{s} \mid \boldsymbol{x}) > \delta,
    $$
    which is contradicts with $\mathcal{R}_K^{\alpha, \ell}(\boldsymbol{s}^{(t)} \mid \boldsymbol{x}) \to \inf_{\boldsymbol{s} \in \mathbb{R}^C} \mathcal{R}_K^{\alpha, \ell}(\boldsymbol{s} \mid \boldsymbol{x})$. Thus, we obtain \textbf{Claim 3}.

    \noindent \rule[2pt]{\linewidth}{0.1em}
    \textbf{Claim 4.} For any sequence of score functions $\{f^{(t)}\}_{t \in \mathbb{N}_+}$,
    $$
        \mathcal{R}_K^{\alpha, \ell}(f^{(t)}) \to \inf_{f} \mathcal{R}_K^{\alpha, \ell}(f) \Rightarrow \mathcal{R}_{K}^\alpha(f^{(t)}) \to \inf_{f} \mathcal{R}_{K}^\alpha(f).
    $$

    \noindent \rule[2pt]{\linewidth}{0.1em}
    It is clear that \textbf{Claim 4} holds with \textbf{Claim 3} and 
    $$
        \mathcal{R}_K^{\alpha, \ell}(f^{(t)}) = \E{\boldsymbol{x}}{ \mathcal{R}_K^{\alpha, \ell}(f^{(t)} \mid \boldsymbol{x}) }.
    $$
    Then, the proof of Thm.\ref{thm:condition_for_consistency} ends.

\end{proof}

\section{Generalization Analysis for TKPR Optimization}
\subsection{The Generalization Bound by Traditional Techniques}
\subsubsection{Lipschitz Property of the TKPR Surrogate Loss (Proof of Prop.\ref{prop:r_tra_lipschitz})}
\label{app:r_tra_lipschitz}

\Rtralipschitz*
\begin{proof}
    Given two score functions $f, f' \in \mathcal{F}_K$, then we have
    $$\begin{aligned}
        & | L_{K}^{\alpha, \ell}(f, \boldsymbol{y}) - L_{K}^{\alpha, \ell}(f', \boldsymbol{y}) | \\
        & = \frac{1}{\alpha K} \left| \sum_{y \in \mathcal{P}(\boldsymbol{x})} \sum_{k \le K+1} \ell(s_y - s_{[k]}) - \sum_{y \in \mathcal{P}(\boldsymbol{x})} \sum_{k \le K+1} \ell(s'_y - s'_{[k]}) \right| \\
        & \overset{(a)}{=} \frac{1}{\alpha K} \left| \sum_{y \in \mathcal{P}(\boldsymbol{x})} \sum_{k \le K+1} \ell(s_y - \boldsymbol{s}')_{[k]} - \sum_{y \in \mathcal{P}(\boldsymbol{x})} \sum_{k \le K+1} \ell(s'_y - \boldsymbol{s}')_{[k]} \right| \\
        & \overset{(b)}{=}  \frac{1}{\alpha K} \left| \sum_{y \in \mathcal{P}(\boldsymbol{x})} \max_{\mathcal{K} \subset [C] \atop \size{\mathcal{K}} = K + 1}\sum_{k \in \mathcal{K}} \ell(s_y - s_k) - \sum_{y \in \mathcal{P}(\boldsymbol{x})} \max_{\mathcal{K} \subset [C] \atop \size{\mathcal{K}} = K + 1} \sum_{k \in \mathcal{K}} \ell(s'_y - s'_k) \right| \\
        & \overset{(c)}{\le} \frac{1}{\alpha K} \sum_{y \in \mathcal{P}(\boldsymbol{x})} \max_{\mathcal{K} \subset [C] \atop \size{\mathcal{K}} = K + 1} \left| \sum_{k \in \mathcal{K}} \left[ \ell(s_y - s_k) - \ell(s'_y - s'_k) \right] \right| \\
        & \le \frac{1}{\alpha K} \sum_{y \in \mathcal{P}(\boldsymbol{x})} \max_{\mathcal{K} \subset [C] \atop \size{\mathcal{K}} = K + 1} \sum_{k \in \mathcal{K}} \left| \ell(s_y - s_k) - \ell(s'_y - s'_k) \right|. \\
    \end{aligned}$$
    In this process, $(a)$ houds since $\ell$ is strictly decreasing. $(b)$ holds since
    $$
        \sum_{k=1}^{K}t_{[k]} = \max_{k \le K} \sum_{k=1}^{K}t_{k}, \forall \boldsymbol{t} \in \mathbb{R}^{C}.
    $$
    And $(c)$ holds since
    $$
        \left|\max \left\{a_{1}, \ldots, a_{K}\right\}-\max \left\{b_{1}, \ldots, b_{K}\right\}\right| \le \max \left\{\left|a_{1}-b_{1}\right|, \ldots,\left|a_{K}-b_{K}\right|\right\}, \quad \forall \boldsymbol{a}, \boldsymbol{b} \in \mathbb{R}^{K}.
    $$
    Furthermore, since $\ell$ is $\mu_\ell$-Lipschitz continuous, the last term is bounded by 
    $$\begin{aligned}
        & \frac{\mu_{\ell}}{\alpha K} \sum_{y \in \mathcal{P}(\boldsymbol{x})} \max_{\mathcal{K} \subset [C] \atop \size{\mathcal{K}} = K + 1} \sum_{k \in \mathcal{K}} \left|(s_y - s'_y) - (s_k - s'_k) \right| \\
        & \le \frac{\mu_{\ell}(K + 1)}{\alpha K} \sum_{y \in \mathcal{P}(\boldsymbol{x})} \left| s_y - s'_y \right| + \frac{\mu_{\ell}N(\boldsymbol{y})}{\alpha K} \max_{\mathcal{K} \subset [C] \atop \size{\mathcal{K}} = K + 1} \sum_{k \in \mathcal{K}} \left| s_k - s'_k \right| \\
        & \le \frac{\mu_{\ell}(K + 1) \sqrt{N(\boldsymbol{y})}}{\alpha K} \left[ \sum_{y \in \mathcal{P}(\boldsymbol{x})} (s_y - s'_y )^{2} \right]^{\frac{1}{2}} + \frac{\mu_{\ell}N(\boldsymbol{y}) \sqrt{K + 1}}{\alpha K} \max_{\mathcal{K} \subset [C] \atop \size{\mathcal{K}} = K + 1} \left[ \sum_{k \in \mathcal{K}} (s_k - s'_k )^{2} \right]^{\frac{1}{2}} \\
        & \le \frac{\mu_{\ell} \left[ (K + 1) \sqrt{N(\boldsymbol{y})} + N(\boldsymbol{y}) \sqrt{K + 1} \right]}{\alpha K} \left[ \sum_{y \in \mathcal{L}} (s_y - s'_y )^{2} \right]^{\frac{1}{2}} \\
        & = \frac{\mu_\ell \left[ (K + 1) \sqrt{N(\boldsymbol{y})} + N(\boldsymbol{y}) \sqrt{K + 1} \right]}{\alpha K} \Vert \boldsymbol{s} - \boldsymbol{s}' \Vert\\
    \end{aligned}$$
    When $\alpha = \alpha_1$, it is clear that 
    $$
        | L_{K}^{\alpha, \ell}(f, \boldsymbol{y}) - L_{K}^{\alpha, \ell}(f', \boldsymbol{y}) | \le \mu_\ell \left( \frac{K + 1}{\sqrt{K}} + \sqrt{K + 1} \right) \Vert \boldsymbol{s} - \boldsymbol{s}' \Vert,
    $$
    and 
    $$
        L_{K}^{\alpha, \ell}(f, \boldsymbol{y}) \le \frac{1}{K} \sum_{y \in \mathcal{P}(\boldsymbol{x})} \sum_{k \le K+1} M_\ell \le (K + 1) M_\ell.
    $$
    When $\alpha = \alpha_2$, it is clear that 
    $$
        | L_{K}^{\alpha, \ell}(f, \boldsymbol{y}) - L_{K}^{\alpha, \ell}(f', \boldsymbol{y}) | = \mu_\ell \left( \frac{K + 1}{\sqrt{N(\boldsymbol{y}) K}} + \frac{\sqrt{K + 1}}{K} \right) \Vert \boldsymbol{s} - \boldsymbol{s}' \Vert \le \mu_\ell \left( \frac{K + 1}{\sqrt{K}} + \frac{\sqrt{K + 1}}{K} \right) \Vert \boldsymbol{s} - \boldsymbol{s}' \Vert,
    $$
    and 
    $$
        L_{K}^{\alpha, \ell}(f, \boldsymbol{y}) \le \frac{1}{N(\boldsymbol{x}) K} \sum_{y \in \mathcal{P}(\boldsymbol{x})} \sum_{k \le K+1} M_\ell \le (K + 1) M_\ell.
    $$
    When $\alpha = \alpha_3$, we have
    $$
        | L_{K}^{\alpha, \ell}(f, \boldsymbol{y}) - L_{K}^{\alpha, \ell}(f', \boldsymbol{y}) | = \frac{2 \mu_\ell}{K} \left[ \frac{K + 1}{\sqrt{N(\boldsymbol{y})} (2K + 1 - N(\boldsymbol{y})) } + \frac{\sqrt{K + 1}}{2K + 1 - N(\boldsymbol{y})} \right] \Vert \boldsymbol{s} - \boldsymbol{s}' \Vert.
    $$
    Let $g(t) := \sqrt{t} (2 K + 1 - t), t \in [1, K]$. Then, it is clear that $g'(t) > 0 $ when $t \in [1, \frac{2K + 1}{3}]$, and $g'(t) < 0 $ when $t \in [\frac{2K + 1}{3}, K]$. Since $2K \le \sqrt{K} (K + 1)$, we have 
    $$
        | L_{K}^{\alpha, \ell}(f, \boldsymbol{y}) - L_{K}^{\alpha, \ell}(f', \boldsymbol{y}) | \le \frac{2 \mu_\ell}{K} \left( \frac{K + 1}{2K } + \frac{\sqrt{K + 1}}{K + 1} \right) \Vert \boldsymbol{s} - \boldsymbol{s}' \Vert = \mu_\ell \left( \frac{K + 1}{K^2} + \frac{2}{K \sqrt{K + 1}} \right) \Vert \boldsymbol{s} - \boldsymbol{s}' \Vert.
    $$
    Meanwhile, we have 
    $$
        L_{K}^{\alpha, \ell}(f, \boldsymbol{y}) \le \frac{2}{N(\boldsymbol{y}) (2K + 1 - N(\boldsymbol{y})) K} \sum_{y \in \mathcal{P}(\boldsymbol{x})} \sum_{k \le K+1} M_\ell \le \frac{K + 1}{K} M_\ell. 
    $$
\end{proof}

\subsubsection{Generalization Bound of TKPR Optimization (Proof of Prop.\ref{prop:gen_existing})}
\label{app:gen_existing}
\genexisting*
\begin{proof}
    According to Lem.\ref{lem:lem_gen_lin_1}, let $\mathcal{G}_{K}^{\ell} := \{L_{K}^{\alpha, \ell} \circ f: f \in \mathcal{F}\}$. Then, with probability at least $1 - \delta$ over the training set $\mathcal{S}$, the following generalization bound holds for all the $f \in \mathcal{F}$, we have
    $$
        \mathcal{R}_K^{\alpha, \ell}(f) \le \Phi(L_{K}^{\alpha, \ell}, \delta) + 2 \hat{\mathfrak{C}}_\mathcal{S}(\mathcal{G}_{K}^{\ell}).
    $$
    When $\alpha = \alpha_1$, based on Lem.\ref{lem:vector_contraction} and Prop.\ref{prop:r_tra_lipschitz}, we have 
    $$
        \hat{\mathfrak{C}}_\mathcal{S}(\mathcal{G}_{K}^{\ell}) \le \sqrt{2} \mu_\ell \left( \frac{K + 1}{\sqrt{K}} + \sqrt{K + 1} \right) \hat{\mathfrak{C}}_\mathcal{S}(\mathcal{F}) \sim \mathcal{O}( \sqrt{K} ) \cdot \hat{\mathfrak{C}}_\mathcal{S}(\mathcal{F}).
    $$
    Similarly, when $\alpha = \alpha_2$ we have
    $$
        \hat{\mathfrak{C}}_\mathcal{S}(\mathcal{G}_{K}^{\ell}) \le \sqrt{2} \mu_\ell \left( \frac{K + 1}{\sqrt{K}} + \frac{\sqrt{K + 1}}{K} \right) \hat{\mathfrak{C}}_\mathcal{S}(\mathcal{F}) \sim \mathcal{O}( \sqrt{K} ) \cdot \hat{\mathfrak{C}}_\mathcal{S}(\mathcal{F}).
    $$
    Similarly, when $\alpha = \alpha_3$ we have
    $$
        \hat{\mathfrak{C}}_\mathcal{S}(\mathcal{G}_{K}^{\ell}) \le \sqrt{2} \mu_\ell \left( \frac{K + 1}{K^2} + \frac{2}{K \sqrt{K + 1}} \right) \hat{\mathfrak{C}}_\mathcal{S}(\mathcal{F}) \sim\mathcal{O}( \frac{1}{K} ) \cdot \hat{\mathfrak{C}}_\mathcal{S}(\mathcal{F}).
    $$
\end{proof}

\subsubsection{Generalization Bound of the Traditional Ranking Loss (Proof of Prop.\ref{prop:sharp_bound})}
\begin{lemma}
    \label{lem:sharp_bound}
    The following inequality holds:
    \begin{equation}
        L_\text{rank}^{\ell}(f, \boldsymbol{y}) \le \frac{\alpha}{N(\boldsymbol{y})} L_{K}^{\alpha, \ell}(f, \boldsymbol{y}).
    \end{equation}
\end{lemma}

\begin{proof}
    The proof is similar to that of Prop.\ref{prop:reformulation}, except that $\ell$ is strictly decreasing.
\end{proof}

\label{app:sharp_bound}
\sharpbound*
\begin{proof}
    According to Prop.\ref{prop:gen_existing} and Lem.\ref{lem:sharp_bound}, when $\alpha = \alpha_1$, we have
    $$\begin{aligned}
        \mathcal{R}_\text{rank}^{\ell}(f) & = \E{(\boldsymbol{x}, \boldsymbol{y}) \sim \mathcal{D}} {L_\text{rank}^{\ell}(f, \boldsymbol{y})} \le \E{(\boldsymbol{x}, \boldsymbol{y}) \sim \mathcal{D}} {\frac{1}{N(\boldsymbol{y})} \cdot L_{K}^{\alpha, \ell}(f, \boldsymbol{y})} \\ 
        & \le \E{(\boldsymbol{x}, \boldsymbol{y}) \sim \mathcal{D}} {L_{K}^{\alpha, \ell}(f, \boldsymbol{y})} = \mathcal{R}_K^{\alpha, \ell}(f) \precsim \Phi(L_{K}^{\alpha, \ell}, \delta) + \mathcal{O}( \sqrt{K} ) \cdot \hat{\mathfrak{C}}_\mathcal{S}(\mathcal{F}). \\
    \end{aligned}$$
    Similarly, when $\alpha = \alpha_2$ we have
    $$
        \mathcal{R}_\text{rank}^{\ell}(f) = \E{(\boldsymbol{x}, \boldsymbol{y}) \sim \mathcal{D}} {L_\text{rank}^{\ell}(f, \boldsymbol{y})} \le \E{(\boldsymbol{x}, \boldsymbol{y}) \sim \mathcal{D}} {L_{K}^{\alpha, \ell}(f, \boldsymbol{y})} = \mathcal{R}_K^{\alpha, \ell}(f) \precsim \Phi(L_{K}^{\alpha, \ell}, \delta) + \mathcal{O}( \sqrt{K} ) \cdot \hat{\mathfrak{C}}_\mathcal{S}(\mathcal{F}). \\
    $$
    Similarly, when $\alpha = \alpha_3$ we have
    $$\begin{aligned}
        \mathcal{R}_\text{rank}^{\ell}(f) & = \E{(\boldsymbol{x}, \boldsymbol{y}) \sim \mathcal{D}} {L_\text{rank}^{\ell}(f, \boldsymbol{y})} \le \E{(\boldsymbol{x}, \boldsymbol{y}) \sim \mathcal{D}} {\left[ 2K + 1 - N(\boldsymbol{y}) \right] / 2 \cdot L_{K}^{\alpha, \ell}(f, \boldsymbol{y})} \\ 
        & \le K \E{(\boldsymbol{x}, \boldsymbol{y}) \sim \mathcal{D}} {L_{K}^{\alpha, \ell}(f, \boldsymbol{y})} = K \cdot \mathcal{R}_K^{\alpha, \ell}(f) \precsim K \cdot \Phi(L_{K}^{\alpha, \ell}, \delta) + \mathcal{O}( 1 ) \cdot \hat{\mathfrak{C}}_\mathcal{S}(\mathcal{F}). \\
    \end{aligned}$$
\end{proof}

\subsection{The Generalization Bound by the Data-dependent Contraction Technique}
\subsubsection{Data-Dependent Contraction Inequality (Proof of Prop.\ref{prop:data_contraction})}
\label{app:data_contraction}

\labelcontraction*
\begin{proof}
    Let $\{\boldsymbol{\xi}_q\}_{q=1}^Q$ be a partition of $\boldsymbol{\xi}$ such that $\forall q \in \{1, 2, \cdots, Q\}, \boldsymbol{\xi}_q \in \{-1, +1\}^{N_q}$. Then, according to the definition of the complexity measure, we have
    $$\begin{aligned}
        \hat{\mathfrak{C}}_\mathcal{S}(\mathcal{G}) & = \E{\boldsymbol{\xi}}{\sup_{g \in \mathcal{G}} \frac{1}{N}\sum_{n=1}^{N} \xi^{(n)} g(\boldsymbol{z}^{(n)}) } = \frac{1}{N} \E{\boldsymbol{\xi}}{\sup_{g \in \mathcal{G}} \sum_{q=1}^{Q} \sum_{n=1}^{N_q} \xi_q^{(n)} g(\boldsymbol{z}_q^{(n)}) } \le \frac{1}{N} \sum_{q=1}^{Q} \E{\boldsymbol{\xi}_q}{\sup_{g \in \mathcal{G}} \sum_{n=1}^{N_q} \xi_q^{(n)} g(\boldsymbol{z}_q^{(n)}) } \\
        & = \sum_{q=1}^{Q}\pi_q \hat{\mathfrak{C}}_{\mathcal{S}_q}(\mathcal{G}) \overset{(a)}{\le} \sqrt{2} \sum_{q=1}^{Q}\pi_q \mu_q \hat{\mathfrak{C}}_{\mathcal{S}_q}(\mathcal{F}) \overset{(b)}{\sim} \mathcal{O}\left( \hat{\mathfrak{C}}_{\mathcal{S}}(\mathcal{F}) \cdot \sum_{q=1}^{Q} \mu_q \sqrt{\pi_q} \right), 
    \end{aligned}$$
    where $(a)$ comes from Lem.\ref{lem:vector_contraction}, and $(b)$ is induced by $\hat{\mathfrak{C}}_{\mathcal{S}_q}(\mathcal{F}) \propto \sqrt{\frac{1}{N_q}} = \sqrt{\frac{1}{N \pi_q}} \propto \sqrt{\frac{1}{\pi_q}} \hat{\mathfrak{C}}_{\mathcal{S}}(\mathcal{F})$.
\end{proof}

\subsubsection{Local Lipschitz Continuity of the TKPR Loss (Proof of Prop.\ref{prop:local_continuous})}
\label{app:local_continuous}
\Localcontinuous*
\begin{proof}
    The proof follows that of Prop.\ref{prop:r_tra_lipschitz}
\end{proof}

\subsubsection{Generalization Bound of $\mathcal{R}_K^{\alpha, \ell}(f)$ Induced by Data-dependent Contraction (Proof of Thm.\ref{thm:final_bound})}
\label{app:final_bound}
\finalbound*
\begin{proof}
    We can obtain Eq.(\ref{eq:final_bound_abstract}) directly by Lem.\ref{lem:lem_gen_lin_1} and Prop.\ref{prop:data_contraction}, and Prop.\ref{prop:local_continuous}. Next, we focus on the two concrete cases.
    
    \noindent \rule[2pt]{\linewidth}{0.1em}
    If $\pi_q \propto e^{- \lambda q}$, since $\mu_q \sim \mathcal{O}(\sqrt{q} / \alpha(q))$, we have 
    $$
        \sum_{q=1}^{K} \mu_q \sqrt{\pi_q} \sim \mathcal{O} \left( \sum_{q=1}^{K} \frac{\sqrt{q}}{e^{\lambda q / 2} \cdot \alpha(q)} \right).
    $$
    When $\alpha(q) = 1$, we have
    $$
        \sum_{q=1}^{K} \frac{\sqrt{q}}{e^{\lambda q / 2}} \le \sum_{q=1}^{K} \frac{q}{e^{\lambda q / 2}} \le \int_{0}^{K} \frac{t}{e^{\lambda t / 2}} \, dt = \frac{4}{\lambda^2} - \frac{2\lambda K + 4}{\lambda^2 e^{\lambda K /2}} \sim \mathcal{O}\left( \frac{1}{\lambda^2} \right).
    $$
    When $\alpha(q) = q$, we have
    $$\begin{aligned}
        \sum_{q=1}^{K} \frac{1}{e^{\lambda q / 2} \sqrt{q}} & \le \frac{1}{e^{\lambda / 2}} + \int_{1}^{K} \frac{1}{e^{\lambda t / 2}\sqrt{t}} \, dt \le \frac{1}{e^{\lambda / 2}} + \int_{1}^{K} \frac{1}{e^{\lambda t / 2}} \, dt \\
        & = \frac{1}{e^{\lambda / 2}} + \frac{2}{\lambda e^{\lambda / 2}} - \frac{2}{\lambda e^{\lambda K / 2}} \sim \mathcal{O}\left( \frac{1}{e^{\lambda / 2}} \right)
    \end{aligned}$$
    When $\alpha(q) = \frac{q (2K + 1 - q)}{2}$, we have
    $$
        \sum_{q=1}^{K} \frac{2}{e^{\lambda q / 2} \sqrt{q} (2K + 1 - q)} \le \frac{1}{K + 1} \sum_{q=1}^{K} \frac{2}{e^{\lambda q / 2} \sqrt{q}} \precsim \mathcal{O}\left( \frac{1}{K e^{\lambda / 2} } \right).
    $$

    \noindent \rule[2pt]{\linewidth}{0.1em}
    If $\pi_q \propto q^{- \lambda}, \lambda > 0$, since $\mu_q \sim \mathcal{O}(\sqrt{q} / \alpha(q))$, we have 
    $$
        \sum_{q=1}^{K} \mu_q \sqrt{\pi_q} \sim \mathcal{O} \left( \sum_{q=1}^{K} \frac{1}{q^{(\lambda - 1) / 2} \cdot \alpha(q)} \right).
    $$
    When $\alpha(q) = 1$, we have
    $$
        \sum_{q=1}^{K} \mu_q \sqrt{\pi_q} \sim \mathcal{O} \left( \sum_{q=1}^{K} \frac{1}{q^{(\lambda - 1) / 2}} \right).
    $$
    Essentially, this bound is a cut-off version of Riemann zeta function \cite{titchmarsh1986theory}. Since the order of Riemann zeta function is out of the scope of this paper, we next provide a coarse-grained result, and the fine-grained results can be found in \cite{fokas2022asymptotics}.
    $$\begin{aligned}
        \lambda \in (0, 3): & \sum_{q=1}^{K} \frac{1}{q^{(\lambda - 1) / 2}} \le \int_{0}^{K} \frac{1}{t^{(\lambda - 1) / 2}} \, dt = \frac{2 \cdot K^{\frac{3-\lambda}{2}} - 2}{3 - \lambda}\sim \mathcal{O}\left( K^{\frac{3-\lambda}{2}} \right), \\
        \lambda \in [3, 5): & \sum_{q=1}^{K} \frac{1}{q^{(\lambda - 1) / 2}} \le \sum_{q=1}^{K} \frac{1}{q} \sim \mathcal{O}\left( \ln K \right), \\
        \lambda \in [5, \infty]: & \sum_{q=1}^{K} \frac{1}{q^{(\lambda - 1) / 2}} \le \sum_{q=1}^{K} \frac{1}{q^2} \le \sum_{q=1}^{\infty} \frac{1}{q^2} = \frac{\pi^2}{6} \sim \mathcal{O}\left( 1 \right).
    \end{aligned}$$
    When $\alpha(q) = q$, we have
    $$
        \sum_{q=1}^{K} \mu_q \sqrt{\pi_q} \sim \mathcal{O} \left( \sum_{q=1}^{K} \frac{1}{q^{(\lambda + 1) / 2}} \right).
    $$
    Similarly, 
    $$\begin{aligned}
        \lambda \in (0, 1): & \sum_{q=1}^{K} \frac{1}{q^{(\lambda + 1) / 2}} \le \int_{0}^{K} \frac{1}{t^{(\lambda + 1) / 2}} \, dt = \frac{2 \cdot K^{\frac{1-\lambda}{2}} - 2}{1 - \lambda}\sim \mathcal{O}\left( K^{\frac{1 - \lambda}{2}} \right), \\
        \lambda \in [1, 3): & \sum_{q=1}^{K} \frac{1}{q^{(\lambda + 1) / 2}} \le \sum_{q=1}^{K} \frac{1}{q} \sim \mathcal{O}\left( \ln K \right), \\
        \lambda \in [3, \infty]: & \sum_{q=1}^{K} \frac{1}{q^{(\lambda + 1) / 2}} \le \sum_{q=1}^{K} \frac{1}{q^2} \le \sum_{q=1}^{\infty} \frac{1}{q^2} = \frac{\pi^2}{6} \sim \mathcal{O}\left( 1 \right).
    \end{aligned}$$
    When $\alpha(q) = \frac{q (2K + 1 - q)}{2}$, we have
    $$
        \sum_{q=1}^{K} \mu_q \sqrt{\pi_q} \sim \mathcal{O} \left( \sum_{q=1}^{K} \frac{1}{(2K + 1 - q) q^{(\lambda + 1) / 2}} \right).
    $$
    Similarly, 
    $$\begin{aligned}
        \lambda \in (0, 1): & \sum_{q=1}^{K} \frac{1}{(2K + 1 - q) q^{(\lambda + 1) / 2}} \le \frac{1}{K+1} \int_{0}^{K} \frac{1}{t^{(\lambda + 1) / 2}} \, dt \sim \mathcal{O}\left( \frac{1}{ K^{(\lambda+1)/2}} \right), \\
        \lambda \in [1, 3): & \sum_{q=1}^{K} \frac{1}{(2K + 1 - q) q^{(\lambda + 1) / 2}} \precsim \mathcal{O}\left( \frac{\ln K}{K} \right), \\
        \lambda \in [3, \infty]: & \sum_{q=1}^{K} \frac{1}{(2K + 1 - q) q^{(\lambda + 1) / 2}} \precsim \mathcal{O}\left( \frac{1}{K} \right).
    \end{aligned}$$
\end{proof}

\subsubsection{Generalization Bound of $\mathcal{R}_{\text{rank}}^\ell(f)$ Induced by Data-dependent Contraction (Proof of Prop.\ref{prop:final_bound_ranking})}
\label{app:final_bound_ranking}

\finalboundranking*
\begin{proof}
    According to Lem.\ref{lem:sharp_bound}, we have
    $$
        \mathcal{R}_\text{rank}^{\ell}(f) = \E{(\boldsymbol{x}, \boldsymbol{y}) \sim \mathcal{D}} {L_\text{rank}^{\ell}(f, \boldsymbol{y})} \le \E{(\boldsymbol{x}, \boldsymbol{y}) \sim \mathcal{D}} {\frac{\alpha}{N(\boldsymbol{y})} \cdot L_{K}^{\alpha, \ell}(f, \boldsymbol{y})} = \E{(\boldsymbol{x}, \boldsymbol{y}) \sim \mathcal{D}} {\tilde{L}_{K}^{\ell}(f, \boldsymbol{y})}
    $$
    Similar to Prop.\ref{prop:local_continuous}, it is clear that $\tilde{L}_{K}^{\ell}(f, \boldsymbol{y})$ is local Lipschitz continuous with constants $\{\tilde{\mu}_q\}_{q=1}^Q$ such that 
    $$
        \tilde{\mu}_q = \frac{\mu_{\ell} \left[ (K + 1) \sqrt{q} + q \sqrt{K + 1} \right]}{q K} \sim \mathcal{O}\left( \frac{1}{\sqrt{q}} \right),
    $$
    Then, let $\tilde{\mathcal{G}}_{K}^{\ell} := \{\tilde{L}_{K}^{\ell} \circ f: f \in \mathcal{F}\}$. According to Lem.\ref{lem:lem_gen_lin_1}, with probability at least $1 - \delta$ over the training set $\mathcal{S}$, the following generalization bound holds for any $f \in \mathcal{F}$, we have
    $$
        \mathcal{R}_\text{rank}^{\ell}(f) \le \Phi(\tilde{L}_{K}^{\ell}, \delta) + 2 \hat{\mathfrak{C}}_\mathcal{S}(\tilde{\mathcal{G}}_{K}^{\ell}).
    $$
    Furthermore, according to Prop.\ref{prop:data_contraction},
    $$
        \hat{\mathfrak{C}}_\mathcal{S}(\tilde{\mathcal{G}}_{K}^{\ell}) \precsim \hat{\mathfrak{C}}_\mathcal{S}(\mathcal{F}) \cdot \sum_{q=1}^{K} \tilde{\mu}_q \sqrt{\pi_q}.
    $$

    \noindent \rule[2pt]{\linewidth}{0.1em}
    If $\pi_q \propto e^{- \lambda q}$, we have
    $$
        \sum_{q=1}^{K} \tilde{\mu}_q \sqrt{\pi_q} \sim \mathcal{O} \left( \sum_{q=1}^{K} \frac{1}{e^{\lambda q / 2} \sqrt{q}} \right).
    $$
    Following the proof of Thm.\ref{thm:final_bound},
    $$
        \sum_{q=1}^{K} \frac{1}{e^{\lambda q / 2} \sqrt{q}} \precsim \mathcal{O}( e^{-\lambda / 2} ).
    $$
    Thus, we have 
    $$
        \mathcal{R}_\text{rank}^{\ell}(f) \precsim \Phi(\tilde{L}_{K}^{\ell}, \delta) + \mathcal{O}( e^{-\lambda / 2} ) \cdot \hat{\mathfrak{C}}_\mathcal{S}(\mathcal{F}).
    $$

    \noindent \rule[2pt]{\linewidth}{0.1em}
    If $\pi_q \propto q^{- \lambda}$, we have 
    $$
        \sum_{q=1}^{K} \tilde{\mu}_q \sqrt{\pi_q} \sim \mathcal{O} \left( \sum_{q=1}^{K} \frac{1}{q^{(\lambda + 1) / 2}} \right).
    $$
    Following the proof of Thm.\ref{thm:final_bound},
    \renewcommand{\arraystretch}{1.5}
    $$
        \sum_{q=1}^{K} \frac{1}{q^{(\lambda + 1) / 2}} \precsim \left\{\begin{array}{ll}
            \mathcal{O}\left( K^{\frac{1 - \lambda}{2}} \right) \cdot \hat{\mathfrak{C}}_\mathcal{S}(\mathcal{F}), \ & \lambda \in (0, 1), \\
            \mathcal{O}\left( \ln K \right) \cdot \hat{\mathfrak{C}}_\mathcal{S}(\mathcal{F}), \ & \lambda \in [1, 3), \\
            \mathcal{O}\left( 1 \right) \cdot \hat{\mathfrak{C}}_\mathcal{S}(\mathcal{F}), \ & \lambda \in [3, \infty). \\
        \end{array}\right.
    $$
    Thus, we have 
    $$
        \mathcal{R}_\text{rank}^{\ell}(f) \precsim \Phi(\tilde{L}_{K}^{\ell}, \delta) + \left\{\begin{array}{ll}
            \mathcal{O}\left( K^{\frac{1 - \lambda}{2}} \right) , \ & \lambda \in (0, 1), \\
            \mathcal{O}\left( \ln K \right), \ & \lambda \in [1, 3), \\
            \mathcal{O}\left( 1 \right), \ & \lambda \in [3, \infty). \\
        \end{array}\right.
    $$
\end{proof}

\subsection{Practical Generalization Bounds}
\subsubsection{Practical Bounds for Kernel-Based Models (Proof of Prop.\ref{prop:gen_kernel} and Prop.\ref{prop:gen_kernel_ranking})}
\label{app:bounds_kernel}
\begin{lemma}[The Rademacher complexity of the kernel-based models \cite{DBLP:conf/nips/WuZ20}]
    \label{lem:complexity_kernel}
    The Rademacher complexity of kernel-based models has the following upper bound:
    \begin{equation}
        \hat{\mathfrak{R}}_\mathcal{S}(\mathcal{F}_{\mathbb{H}}) \le \sqrt{\frac{C \Lambda^2 r^2}{N}}.
    \end{equation}
\end{lemma}

\genkernel* 
\begin{proof}
    The proof completes by Thm.\ref{thm:final_bound} and Lem.\ref{lem:complexity_kernel}.
\end{proof}

\genkernelranking*
\begin{proof}
    The proof completes by Prop.\ref{prop:final_bound_ranking} and Lem.\ref{lem:complexity_kernel}.
\end{proof}

\subsubsection{Practical Bounds for Convolutional Neural Networks (Proof of Prop.\ref{prop:gen_cnn})}
\label{app:bounds_cnn}
\begin{lemma}[The Gaussian complexity of Convolutional Neural Networks \cite{DBLP:conf/iclr/LongS20,DBLP:journals/pami/WangXYHCH23}]
    \label{lem:complexity_cnn}
    The Rademacher complexity of $\mathcal{F}_\nu$ has the following upper bound:
    \begin{equation}
        \hat{\mathfrak{G}}_\mathcal{S}(\mathcal{F}_{\beta, \nu, \chi}) \precsim \mathcal{O} \left( \frac{ d \log \left( B_{\beta, \nu, \chi}N \right) }{\sqrt{N}} \right).
    \end{equation}
\end{lemma}

\gencnn*
\begin{proof}
    Due to the additional term $\log N$ in the upper bound of $\hat{\mathfrak{G}}_\mathcal{S}(\mathcal{F}_{\beta, \nu, \chi})$, we need to make minor adjustments to Prop.\ref{prop:data_contraction}. Let $\mathcal{G}_{\beta, \nu, \chi} := \{L_K^\ell \circ f: f \in \mathcal{F}_{\beta, \nu, \chi}\}$. Then, we have 
    $$\begin{aligned}
        \hat{\mathfrak{C}}_\mathcal{S}(\mathcal{G}_{\beta, \nu, \chi}) & = \sum_{q=1}^{Q}\pi_q \hat{\mathfrak{C}}_{\mathcal{S}_q}(\mathcal{G}_{\beta, \nu, \chi}) \le \sqrt{2} \sum_{q=1}^{Q} \pi_q \mu_q \hat{\mathfrak{C}}_{\mathcal{S}_q}(\mathcal{F}_{\beta, \nu, \chi}) \sim \mathcal{O}\left( \sum_{q=1}^{Q} \pi_q \mu_q \hat{\mathfrak{C}}_{\mathcal{S}_q}(\mathcal{F}_{\beta, \nu, \chi}) \right), \\
        & = \mathcal{O} \left( \sum_{q=1}^{Q}\pi_q \mu_q \frac{ d \log \left( B_{\beta, \nu, \chi} N \pi_q \right) }{\sqrt{N \pi_q}} \right) \le \mathcal{O} \left( \sum_{q=1}^{Q} \sqrt{\pi_q} \mu_q \frac{ d \log \left( B_{\beta, \nu, \chi} N \right) }{\sqrt{N}} \right) \\
        & = \mathcal{O} \left( \hat{\mathfrak{C}}_{\mathcal{S}}(\mathcal{F}_{\beta, \nu, \chi}) \sum_{q=1}^{Q} \sqrt{\pi_q} \mu_q \right)
    \end{aligned}$$
    Then the proof ends by Prop.\ref{prop:final_bound_ranking} and Lem.\ref{lem:complexity_cnn}.
\end{proof}

\clearpage
\section{More Implementation Details}
\label{app:more_imp_details}
  \textbf{Infrastructure.} For CNN backbone, we carry out the experiments on an ubuntu 16.04 server equipped with Intel(R) Xeon(R) Silver 4110 CPU and an Nvidia(R) TITAN RTX GPU. The version of CUDA is 10.2 with GPU driver version 440.44. Our codes are implemented via \texttt{python} (v-3.8.11) with the main third-party packages including \texttt{pytorch} \cite{DBLP:conf/nips/PaszkeGMLBCKLGA19} (v-1.9.0), \texttt{numpy} (v-1.20.3), \texttt{scikit-learn} (v-0.24.2) and \texttt{torchvision} (v-0.10.0). 
  
  For transformer backbone, we carry out the experiments on an ubuntu 20.04 server equipped with AMD(R) EPYC(R) 7763 CPU and an Nvidia(R) A100 GPU. The version of CUDA is 11.6 with GPU driver version 510.108.03. The packages are same as those in the CNN backbone.

  \textbf{Warm-up strategy.} Note that TKPR focuses on the predictions of top-ranked labels. However, focusing on a few labels at the early training process ignores the learning of the other labels and can bring a high risk of over-fitting. To address this issue, we adopt a warm-up training strategy that encourages the model learning the global information. To be specific, in the first $E_w$ epochs, the models is trained by a global loss. Afterwards, the warm-up loss is replaced with the proposed TKPR losses. The whole learning process is summarized in Alg.\ref{alg:tkpr}.

  In MLC, for the comparison with the ranking-based losses, all the methods adopt the warm-up strategy with the warm-up loss $L_{\text{rank}}$. For the comparison with state-of-the-art methods, the proposed method uses DB-loss/ASL as the warm-up loss for the CNN/transformer backbones, respectively. For the CNN backbones, we set $E=100, 200$ on Pascal VOC 2007 and MS-COCO, respectively, and $E_w$ is searched in \{10, 20, 30, 40, 50, 60\}. For the transformer backbones, we set $E = 80$, and $E_w$ is searched in \{5, 10, 15, 20, 25\}.

  In MLML, for the comparison with the ranking-based losses, all the methods use $L_{\text{rank}}$ as the warm-up loss. For the comparison with state-of-the-art methods, the proposed method uses SPLC as the warm-up loss. We set $E = 80$, and $E_w$ is searched in \{2, 4, 6, 8, 10\}.

  \textbf{Implementation of competitors.} It is not a trivial task to evaluate the competitors on the proposed TKPR measures due to the significant differences among their settings. In the official setting, $L_{\text{rank}}, L_{u_1}, L_{u_2}, L_{u_3}, L_{u_4}$ and TKML conduct the experiments on traditional datasets such as  emotions, bibtex, delicious \cite{DBLP:journals/jmlr/TsoumakasXVV11}, which are no longer popular benchmark datasets for modern backbones. LSEP uses VGG16 as the backbone, which is somewhat outdated. DB-Loss trains the model on a modified version of MS-COCO and Pascal VOC. Fortunately, ASL, CCD, Hill, and SPLC share the same protocols and implementation details, and EM+APL, ROLE, LL-R, LL-Ct, and LL-Cp share the other implementation. Hence, we adopt the following implementation strategy for fair comparison:

  \begin{itemize}
    \item For $L_{\text{rank}}, L_{u_1}, L_{u_2}, L_{u_3}, L_{u_4}$, TKML, LSEP, and DB-Loss, we re-implement the methods based on the code released by ASL (\url{https://github.com/Alibaba-MIIL/ASL/blob/main/train.py}). For fair comparison, we align the training details with ours, including warm-up, optimizer, batch size, learning rate, input size and so on. The hyperparameters are also searched as suggested in the original paper.
    \item For ASL, CCD, Hill, SPLC, EM+APL, ROLE, LL-R, LL-Ct, and LL-Cp, we adopt the official implementation and the hyperparameters suggested in the original paper. The mAP performance is also evaluated to guarantee our checkpoints share similar performances as those reported in the original paper.
  \end{itemize}
  Besides, as analyzed in \citep{DBLP:journals/ai/GaoZ13} and \citep{DBLP:conf/nips/WuLXZ21}, $\ell_\text{arc}(t) = - \texttt{arctan}(t)$ and $\ell_\text{exp}$ are consistent surrogates for the ranking loss and $L_{u_2}$, respectively. Our implementation follows these theoretical results.

  \textbf{Hyper-parameter search.} We withhold 20\% of the training set for hyper-parameter search. After this, the model is trained on the full training set with the best hyper-parameters. For fair comparison, the common hyper-parameters, such as learning rate and batch size, are searched for competitors in the same space. The specific hyper-parameters of each method are also searched as suggested in the original paper. If the original paper has provided the optimal hyper-parameters, we will use them directly. Since the search space is somewhat large, we adopt an early stopping strategy with 5 patience epochs. The hyper-parameters of the warm-up loss follows those of the main loss, requiring no explicit validation set.

\begin{algorithm}[th]
      \caption{Learning Algorithm of the ERM Framework}
      \label{alg:tkpr}
      \begin{algorithmic}[1]
          \Require Training set $\mathcal{S}$, the model parameterized by $\Theta$, the hyperparameters $\{\alpha, \ell, K\}$, the warm-up loss $L_w$.
          \State Initialize the model parameters $\Theta$.
          \For{$e = 1, 2, \cdots, E$}
              \State $\mathcal{B} \leftarrow \text{SampleMiniBatch}(\mathcal{S}, m)$
              \If{$e \le E_w$}
                  \State Update $\Theta$ by minimizing $L_w$.
              \Else
                  \State Update $\Theta$ by minimizing $L_{K}^{\alpha, \ell}$.
              \EndIf
              \State Optional: anneal the learning rate $\eta$.
          \EndFor
      \end{algorithmic}
\end{algorithm}

\clearpage
\section{More Empirical Results}
\label{app:more_results}
\subsection{MLC experiments on Pascal VOC 2007}
\label{app:pascal07}
\textbf{Overall Performance.} Tab.\ref{tab:ranking_loss_voc_MLC} presents the comparison results with the ranking-based losses in the MLC setting, from which we have the following observations: 
  \begin{itemize}
    \item The proposed methods demonstrate consistent improvements on \texttt{mAP@K}, \texttt{NDCG@K}, the ranking loss, and the TKPR measures. All these performance gains again validate our theoretical analyses in Sec.\ref{sec:tkpr} and Sec.\ref{sec:generalization}.
    \item Similar to the results on MS-COCO, the improvement on \texttt{P@K} and \texttt{R@K} is not so significant. This shows that the performance enhancement comes from the improvement on ranking of predictions, which is consistent with the label distributions presented in Fig.\ref{fig:VOC_normal_distribution}.
    \item The performances of the competitors are  inconsistent on different measures. For example,$L_{\text{rank}}$ achieve the best \texttt{NDCG@K} and \texttt{TKPR} performance. However, on \texttt{P@K}, \texttt{mAP@K}, and the ranking loss, $L_{u_4}$ and $L_{\text{LSEP}}$ are the best. By contrast, the superior performance of TKPR optimization is consistent, which again validates the necessity of the proposed framework.
  \end{itemize}
  Tab.\ref{tab:sota_voc_MLC} presents the comparison results with the state-of-the art loss-oriented methods in the MLC setting, from which we have the following observations: 
  \begin{itemize}
    \item Compared with the state-of-the-art methods, the proposed methods achieve the best performances consistently on all the measures, which again validate the effectiveness of the proposed learning framework.
    \item The DB-Loss and ASL are competitive in this setting. Their success might come from the enphasis on inherent imbalanced label distributions in multi-label learning. However, the proposed methods still outperform the two methods, as the analysis in Sec.\ref{sec:generalization} show that the proposed method can also generalize well on imbalanced label distributions.
    \item Similarly, performances of the competitors are inconsistent on different measures while the proposed methods achieve the best performances consistently. 
  \end{itemize}

  \textbf{Sensitivity Analysis.} In Fig.\ref{fig:VOC_sensitivity}, we present the sensitivity of the proposed framework \textit{w.r.t.} hyperparameter $K$ on Pascal VOC 2007 with $E_w = 50$. From the results, we have the following observations:
  \begin{itemize}
    \item The model tends to achieve the best median performance with $K \in \{3, 4, 5\}$, which suggests that an appropriately larger $K$ is necessary for the proposed framework.
    \item The performance degeneration under $K = 2$ is more significant when $\alpha = \alpha_3$. In other words, TKPR optimization with $\alpha = \alpha_3$ is more sensitive than those with $\alpha \in \{\alpha_1, \alpha_2\}$.
  \end{itemize}

\begin{table*}[ht]
  \centering
    \caption{The empirical results of the ranking-based losses and TKPR on Pascal VOC 2007, where the backbone is ResNet101. The best and runner-up results on each metric are marked with {\color{Top1}red} and {\color{Top2}blue}, respectively. The best competitor on each measure is marked with \underline{underline}.}
    \label{tab:ranking_loss_voc_MLC}%
  \renewcommand\arraystretch{1.5}
  \tiny 
  \newcommand{\tabincell}[2]{\begin{tabular}{@{}#1@{}}#2\end{tabular}}
  \begin{tabular}{c|c|cc|cc|cc|cc|cc|cc|cc|c}
    \multicolumn{1}{c|}{\multirow{2}[3]{*}{Type}} & Metrics & \multicolumn{2}{c|}{P@K} & \multicolumn{2}{c|}{R@K} & \multicolumn{2}{c|}{mAP@K} & \multicolumn{2}{c|}{NDCG@K} & \multicolumn{2}{c|}{$\text{TKPR}^{\alpha_1}$} & \multicolumn{2}{c|}{$\text{TKPR}^{\alpha_2}$} & \multicolumn{2}{c|}{$\text{TKPR}^{\alpha_3}$} & \multicolumn{1}{c}{\multirow{2}[3]{*}{\tabincell{c}{Ranking \\ Loss}}} \\
    \cmidrule{2-16}          & K     & 3     & 5     & 3     & 5     & 3     & 5     & 3     & 5     & 3     & 5     & 3     & 5     & 3     & 5     &  \\
    \toprule
    \multicolumn{1}{c|}{\multirow{7}[2]{*}{\tabincell{c}{Ranking \\ Loss}}} & \multicolumn{1}{c|}{$L_{\text{rank}}$} & .444  & .278 & .959 &	.989 & .613  & .625  & \underline{.826}  & \underline{.840}  & \underline{1.045} & \underline{1.181} & \underline{.762}  & \underline{.850}  & \underline{.273}  & \underline{.177}  & .043 \\
    & $L_{u1}$ & .443  & .278 & .960  & .989 & .653  & .665  & .766  & .780  & .985  & 1.144 & .708  & .818  & .255  & .171  & .041 \\
    & $L_{u2}$ & .443  & {\color{Top2} .278} & .958	& .988 & .652  & .665  & .764  & .778  & .984  & 1.144 & .707  & .817  & .255  & .171  & .041 \\
    & $L_{u3}$ & .443  & .279 & .956 & {\color{Top2} .991} & .650  & .664  & .796  & .812  & 1.008 & 1.160 & .730  & .832  & .262  & .174  & .040 \\
    & $L_{u4}$ & \underline{.446}  & .278 & \underline{{\color{Top2} .965}} & .989  & \underline{.674}  & \underline{.683}  & .785  & .796  & 1.015 & 1.162 & .733  & .833  & .264  & .174  & \underline{.038} \\
    & $L_{\text{LSEP}}$ & .442  & \underline{{\color{Top1} .279}} & .959 & \underline{{\color{Top1} .991}} & .626  & .641  & .817  & .833  & 1.033 & 1.175 & .752  & .845  & .270  & .176  & .042 \\
    & $L_{\text{TKML}}$ & .422  & .272  & .915 & .972  & .614  & .634 & .788  & .815  & .998  & 1.138 & .729  & .823  & .262  & .171  & .052 \\
    \toprule
    \multicolumn{1}{c|}{\multirow{3}[1]{*}{\tabincell{c}{TKPR \\ (Ours)}}} & $\alpha_1$ & .437  & .272  & .945  & .969  & .912  & .921  & .930  & .941  & 1.154 & 1.233 & .864  & .904  & .308  & .188  & .021 \\
    & $\alpha_2$  & {\color{Top1} .450} & .277  & {\color{Top1} .970} & .987  & {\color{Top2} .947} & {\color{Top1} .954} & {\color{Top1} .960} & {\color{Top1} .968} & {\color{Top1} 1.190} & {\color{Top1} 1.267} & {\color{Top1} .892} & {\color{Top1} .928} & {\color{Top1} .318} & {\color{Top1} .193} & {\color{Top1} .010} \\
    &  $\alpha_3$  & {\color{Top2} .447} & .277  & .948 & .965  & {\color{Top1} .985} & {\color{Top2} .947} & {\color{Top2} .953} & {\color{Top2} .963} & {\color{Top2} 1.185} & {\color{Top2} 1.262} & {\color{Top2} .886} & {\color{Top2} .924} & {\color{Top2} .316} & {\color{Top2} .192} & {\color{Top2} .011} \\
\end{tabular}%
\end{table*}%

\begin{table*}[ht]
  \centering
    \caption{The empirical results of state-of-the-art MLC methods and TKPR on Pascal VOC 2007, where the backbone is ResNet101. The best and runner-up results on each metric are marked with {\color{Top1}red} and {\color{Top2}blue}, respectively. The best competitor on each measure is marked with \underline{underline}.}
    \label{tab:sota_voc_MLC}%
  \renewcommand\arraystretch{1.5}
  \tiny 
  \newcommand{\tabincell}[2]{\begin{tabular}{@{}#1@{}}#2\end{tabular}}
  \begin{tabular}{c|c|cc|cc|cc|cc|cc|cc|cc|c}
    \multicolumn{1}{c|}{\multirow{2}[4]{*}{Type}} & Metrics & \multicolumn{2}{c|}{P@K} & \multicolumn{2}{c|}{R@K} & \multicolumn{2}{c|}{mAP@K} & \multicolumn{2}{c|}{NDCG@K} & \multicolumn{2}{c|}{$\text{TKPR}^{\alpha_1}$} & \multicolumn{2}{c|}{$\text{TKPR}^{\alpha_2}$} & \multicolumn{2}{c|}{$\text{TKPR}^{\alpha_3}$} & \multicolumn{1}{c}{\multirow{2}[4]{*}{\tabincell{c}{Ranking \\ Loss}}} \\
    \cmidrule{2-16}          & K     & 3     & 5     & 3     & 5     & 3     & 5     & 3     & 5     & 3     & 5     & 3     & 5     & 3     & 5     &  \\
    \toprule
    \multicolumn{1}{c|}{\multirow{5}[2]{*}{\tabincell{c}{Loss \\ Oriented}}} & ASL\footnotemark[2]   & .452  & .278  & \underline{.973}  & .989  & \underline{.956}  & \underline{.963}  & .966  & .973  & 1.203 & 1.277 & .899  & .933  & .321  & .194  & .008 \\
    & DB-Loss & \underline{.453}  & \underline{.279}  & .973  & .988  & .956  & .962  & \underline{.966}  & \underline{.973}  & \underline{1.205} & \underline{1.278} & \underline{.900}  & \underline{.933}  & \underline{.321}  & \underline{.195}  & \underline{.008} \\
    & CCD\footnotemark[2]   & .451  & .279  & .969  & \underline{{\color{Top2} .990}}  & .948  & .956  & .960  & .970  & 1.194 & 1.272 & .893  & .930  & .319  & .194  & .009 \\
    & \multicolumn{1}{c|}{Hill\footnotemark[2]} & .446  & .276  & .960  & .981  & .932  & .940  & .948  & .957  & 1.179 & 1.256 & .881  & .919  & .315  & .192  & .014 \\
    & \multicolumn{1}{c|}{SPLC\footnotemark[2]} & .451  & .278  & .969  & .986  & .947  & .954  & .959  & .967  & 1.195 & 1.271 & .892  & .928  & .319  & .193  & .010 \\
    \toprule
    \multicolumn{1}{c|}{\multirow{3}[2]{*}{\tabincell{c}{TKPR \\ (Ours)}}} & $\alpha_1$ & {\color{Top2} .455} & {\color{Top1} .281} & .975 & .989 & {\color{Top2} .961} & {\color{Top2} .967} & {\color{Top2} .971} & {\color{Top2} .977} & {\color{Top2} 1.210} & {\color{Top1} 1.283} & .903 & {\color{Top2} .937} & {\color{Top2} .322} & {\color{Top2} .196} & {\color{Top2} .007} \\
    & $\alpha_2$ & .454 & .279 & {\color{Top2} .976} & {\color{Top1} .990} & .961 & .966 & .971 & .977 & 1.208 & 1.282 & {\color{Top2} .904} & .936 & .322 & .196 & .007 \\
    & $\alpha_3$ & {\color{Top1} .455} & {\color{Top2} .279} & {\color{Top1} .977} & .989 & {\color{Top1} .962} & {\color{Top1} .967} & {\color{Top1} .972} & {\color{Top1} .977} & {\color{Top1} 1.211} & {\color{Top2} 1.283} & {\color{Top1} .905} & {\color{Top1} .937} & {\color{Top1} .322} & {\color{Top1} .196} & {\color{Top1} .007} \\
\end{tabular}%
\end{table*}%

\begin{figure}[ht]
    \centering
    \subfigure{
        \label{fig:VOC_ALPHA1}
        \includegraphics[width=0.8\linewidth]{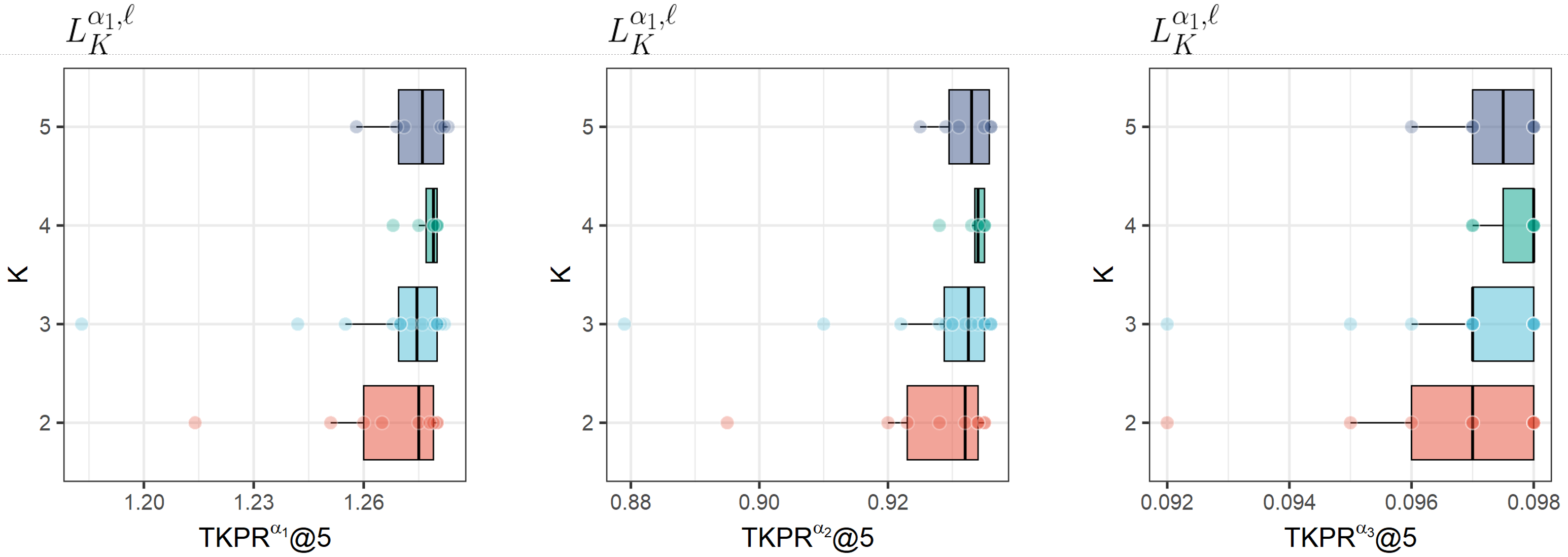}}
    \subfigure{
        \label{fig:VOC_ALPHA2}
        \includegraphics[width=0.8\linewidth]{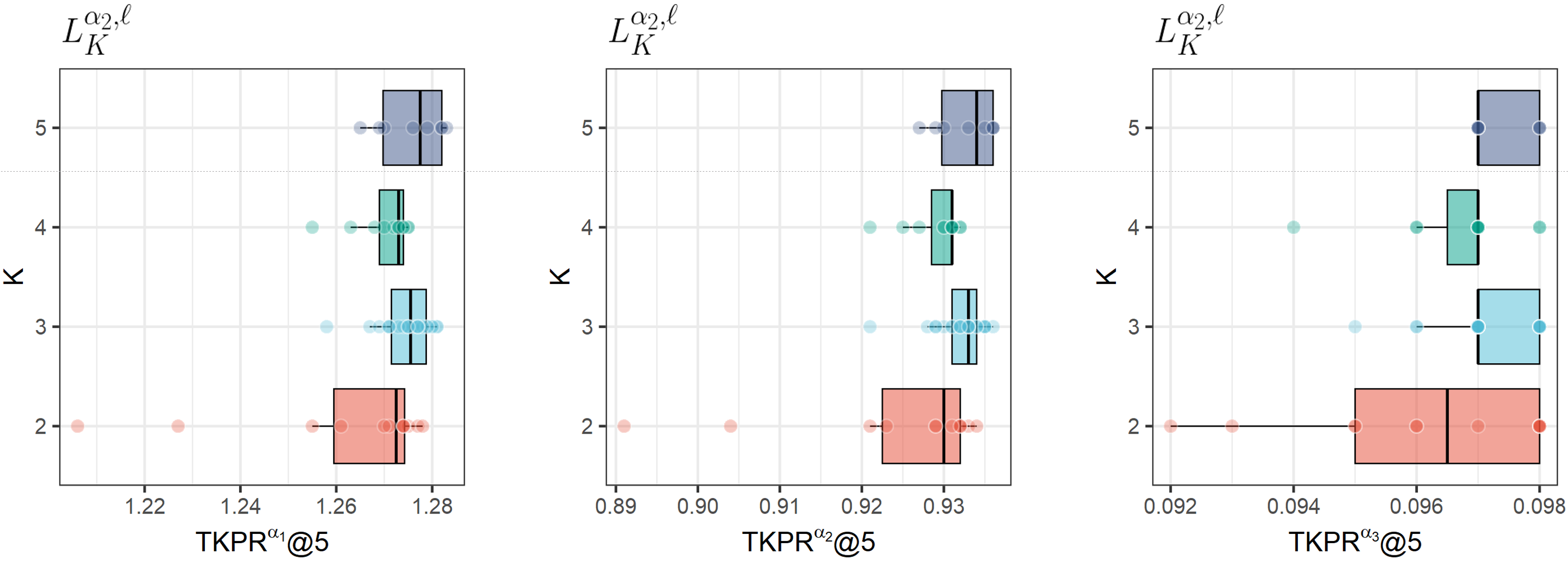}}
    \subfigure{
        \label{fig:VOC_ALPHA3}
        \includegraphics[width=0.8\linewidth]{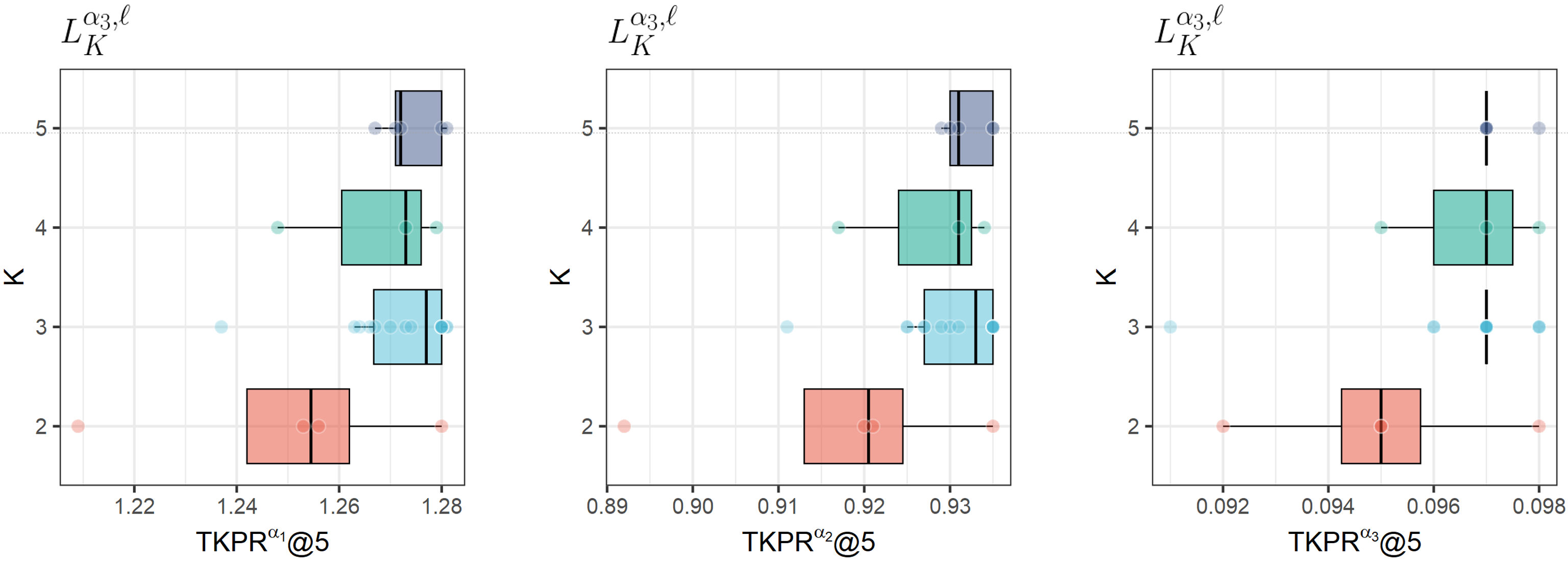}}
    \caption{Sensitivity analysis of the proposed methods on Pascal VOC 2007. The y-axis denotes the values of the hyperparameter $K$, and the x-axis represents the value of \texttt{TKPR@5} under the corresponding $K$.}
    \label{fig:VOC_sensitivity}
\end{figure}

\clearpage
\subsection{MLC experiments with Transformer Backbone}
\label{app:transformer}
  \textbf{Overall Performance.} Tab.\ref{tab:sota_coco_MLC_transformer} and Tab.\ref{tab:sota_nus_MLC_transformer} present the comparison results with the state-of-the-art loss-oriented methods with the swin-transformer backbone in the MLC setting, from which we have the following observations: 
  \begin{itemize}
    \item The proposed methods still outperform the competitors consistently on all the measures, which validates the robustness of the proposed framework.
    \item Compared with the improvements with the CNN backbone, the performance gains on the transformer backbone are more significant on MS-COCO. This phenomenon shows that the proposed framework can better exploit the potential of powerful pre-trained backbones.
    \item Benefiting from the powerful backbones, CCD achieves the best performance among the competitors. However, its performances on the ranking-based measures with $K=3$ are comparable/inferior to ASL. In other words, CDD pay more attention to the whole ranking list, rather than the top-ranked ones.
  \end{itemize}

  \textbf{Case study.} Fig.\ref{fig:case_coco} and Fig.\ref{fig:case_nus} present visual results on the MS-COCO and NUS-WIDE datasets, respectively, where the proposed methods perform better than all the competitors. For each input image, we present the ground-truth labels and the top-5 predicted labels of the proposed methods and the competitors. All the ground-truth labels are marked with {\color[RGB]{26, 184, 98} \textbf{green}}. From the results, we have the following observations:
  \begin{itemize}
    \item TKPR tends to rank the ground-truth labels higher than the competitors. This improvement brings the enhancement on \texttt{mAP@K}, \texttt{NDCG@K}, \texttt{TKPR}, and the ranking loss, which is consistent with the Bayes optimality of TKPR and again validates the effectiveness of the proposed framework.
    \item For some cases, TKPR ranks all the ground-truth labels in the top-5 list, which brings the improvement on \texttt{P@K} and \texttt{R@K}.
  \end{itemize}

\begin{table*}[ht]
  \centering
    \caption{The empirical results of state-of-the-art MLC methods and TKPR on MS-COCO, where the backbone is swin-transformer. The best and runner-up results on each metric are marked with {\color{Top1}red} and {\color{Top2}blue}, respectively. The best competitor on each measure is marked with \underline{underline}.}
    \label{tab:sota_coco_MLC_transformer}%
    \renewcommand\arraystretch{1.5}
    \tiny 
    \newcommand{\tabincell}[2]{\begin{tabular}{@{}#1@{}}#2\end{tabular}}
    \begin{tabular}{c|c|cc|cc|cc|cc|cc|cc|cc|c}
      \multicolumn{1}{c|}{\multirow{2}[4]{*}{Type}} & Metrics & \multicolumn{2}{c|}{P@K} & \multicolumn{2}{c|}{R@K} & \multicolumn{2}{c|}{mAP@K} & \multicolumn{2}{c|}{NDCG@K} & \multicolumn{2}{c|}{$\text{TKPR}^{\alpha_1}$} & \multicolumn{2}{c|}{$\text{TKPR}^{\alpha_2}$} & \multicolumn{2}{c|}{$\text{TKPR}^{\alpha_3}$} & \multicolumn{1}{c}{\multirow{2}[4]{*}{\tabincell{c}{Ranking \\ Loss}}} \\
      \cmidrule{2-16}          & K     & 3     & 5     & 3     & 5     & 3     & 5     & 3     & 5     & 3     & 5     & 3     & 5     & 3     & 5     &  \\
      \toprule
      \multicolumn{1}{c|}{\multirow{5}[2]{*}{\tabincell{c}{Loss \\ Oriented}}} & ASL\footnotemark[2] & .705 & .496 & .838 & .917 & .932 & .919 & .951 & .947 & 1.606 & 1.928 & .754 & .786 & .318 & .189 & .0096 \\
      & DB-Loss & .698 & .489 & .831 & .907 & .924 & .908 & .944 & .938 & 1.595 & 1.908 & .749 & .779 & .316 & .187 & .0125 \\
      & CCD\footnotemark[2]   & \underline{.705} & \underline{.498} & \underline{.839} & \underline{.920} & \underline{.933} & \underline{.921} & \underline{.951} & \underline{.948} & \underline{1.606} & \underline{1.930} & \underline{.754} & \underline{.787} & \underline{.318} & \underline{.189} & \color{Top2}\underline{.0094} \\
      & \multicolumn{1}{c|}{Hill\footnotemark[2]} & .696 & .493 & .832 & .915 & .918 & .909 & .940 & .940 & 1.585 & 1.909 & .745 & .780 & .314 & .187 & .0101 \\
      & \multicolumn{1}{c|}{SPLC\footnotemark[2]} & .676 & .487 & .815 & .909 & .883 & .883 & .914 & .922 & 1.532 & 1.862 & .724 & .765 & .305 & .183 & .0116 \\
      \toprule
      \multicolumn{1}{c|}{\multirow{3}[2]{*}{\tabincell{c}{TKPR \\ (Ours)}}} & $\alpha_1$ & \color{Top1}.712 & \color{Top1}.503 & \color{Top1}.847 & \color{Top1}.928 & \color{Top1}.940 & \color{Top1}.929 & \color{Top1}.956 & \color{Top1}.953 & \color{Top1}1.616 & \color{Top1}1.953 & \color{Top1}.759 & \color{Top1}.798 & \color{Top1}.321 & \color{Top1}.192 & .0096 \\
      & $\alpha_2$ & \color{Top2}.710 & \color{Top2}.501 & \color{Top2}.844 & \color{Top2}.925 & \color{Top2}.937 & \color{Top2}.926 & \color{Top2}.955 & \color{Top2}.952 & \color{Top2}1.612 & \color{Top2}1.941 & \color{Top2}.756 & \color{Top2}.790 & \color{Top2}.319 & \color{Top2}.190 & \color{Top1}.0092 \\
      & $\alpha_3$ & .707 & .499 & .842 & .922 & .935 & .923 & .953 & .950 & 1.609 & 1.935 & .755 & .788 & .319 & .189 & .0096 \\
    \end{tabular}%
\end{table*}%

\begin{table*}[ht]
  \centering
    \caption{The empirical results of state-of-the-art MLC methods and TKPR on NUS-WIDE, where the backbone is swin-transformer. The best and runner-up results on each metric are marked with {\color{Top1}red} and {\color{Top2}blue}, respectively. The best competitor on each measure is marked with \underline{underline}.}
    \label{tab:sota_nus_MLC_transformer}%
    \renewcommand\arraystretch{1.5}
    \tiny 
    \newcommand{\tabincell}[2]{\begin{tabular}{@{}#1@{}}#2\end{tabular}}
    \begin{tabular}{c|c|cc|cc|cc|cc|cc|cc|cc|c}
      \multicolumn{1}{c|}{\multirow{2}[4]{*}{Type}} & Metrics & \multicolumn{2}{c|}{P@K} & \multicolumn{2}{c|}{R@K} & \multicolumn{2}{c|}{mAP@K} & \multicolumn{2}{c|}{NDCG@K} & \multicolumn{2}{c|}{$\text{TKPR}^{\alpha_1}$} & \multicolumn{2}{c|}{$\text{TKPR}^{\alpha_2}$} & \multicolumn{2}{c|}{$\text{TKPR}^{\alpha_3}$} & \multicolumn{1}{c}{\multirow{2}[4]{*}{\tabincell{c}{Ranking \\ Loss}}} \\
      \cmidrule{2-16}          & K     & 3     & 5     & 3     & 5     & 3     & 5     & 3     & 5     & 3     & 5     & 3     & 5     & 3     & 5     &  \\
      \toprule
      \multicolumn{1}{c|}{\multirow{5}[2]{*}{\tabincell{c}{Loss \\ Oriented}}} & ASL\footnotemark[2] & .573 & .420 & .799 & .906 & \underline{.817} & .830 & \underline{.855} & .876 & \underline{1.318} & 1.601 & \underline{.708} & .757 & \underline{.284} & .175 & .0132 \\
      & DB-Loss & .572 & .419 & .798 & .904 & .815 & .828 & .853 & .875 & 1.316 & 1.597 & .707 & .756 & .283 & .174 & .0149 \\
      & CCD\footnotemark[2]   & \underline{.574} & \underline{.422} & \underline{.801} & \underline{.910} & .816 & \underline{.831} & .854 & \underline{.877} & 1.317 & \underline{1.602} & .707 & \underline{.758} & .284 & \underline{.175} & \underline{.0127} \\
      & \multicolumn{1}{c|}{Hill\footnotemark[2]} & .570 & .420 & .798 & .908 & .803 & .820 & .844 & .869 & 1.299 & 1.589 & .699 & .753 & .280 & .174 & .0132 \\
      & \multicolumn{1}{c|}{SPLC\footnotemark[2]} & .564 & .420 & .793 & .908 & .789 & .810 & .833 & .862 & 1.278 & 1.574 & .689 & .747 & .276 & .172 & .0137 \\
      \toprule
      \multicolumn{1}{c|}{\multirow{3}[2]{*}{\tabincell{c}{TKPR \\ (Ours)}}} & $\alpha_1$ & \color{Top2}.579 & \color{Top2}.424 & \color{Top2}.807 & \color{Top2}.914 & \color{Top2}.824 & \color{Top2}.838 & \color{Top2}.861 & \color{Top2}.883 & \color{Top2}1.328 & \color{Top2}1.613 & \color{Top2}.713 & \color{Top2}.763 & \color{Top2}.286 & \color{Top2}.176 & \color{Top2}.0126 \\
      & $\alpha_2$ & \color{Top1}.581 & \color{Top1}.425 & \color{Top1}.809 & \color{Top1}.916 & \color{Top1}.827 & \color{Top1}.840 & \color{Top1}.864 & \color{Top1}.885 & \color{Top1}1.336 & \color{Top1}1.617 & \color{Top1}.716 & \color{Top1}.767 & \color{Top1}.287 & \color{Top1}.177 & \color{Top1}.0122 \\
      & $\alpha_3$ & .578 & .423 & .806 & .912 & .824 & .837 & .861 & .882 & 1.328 & 1.612 & .713 & .762 & .286 & .176 & .0128 \\
  \end{tabular}%
\end{table*}%

\begin{figure}[htbp]
    \centering
    \includegraphics[width=\linewidth]{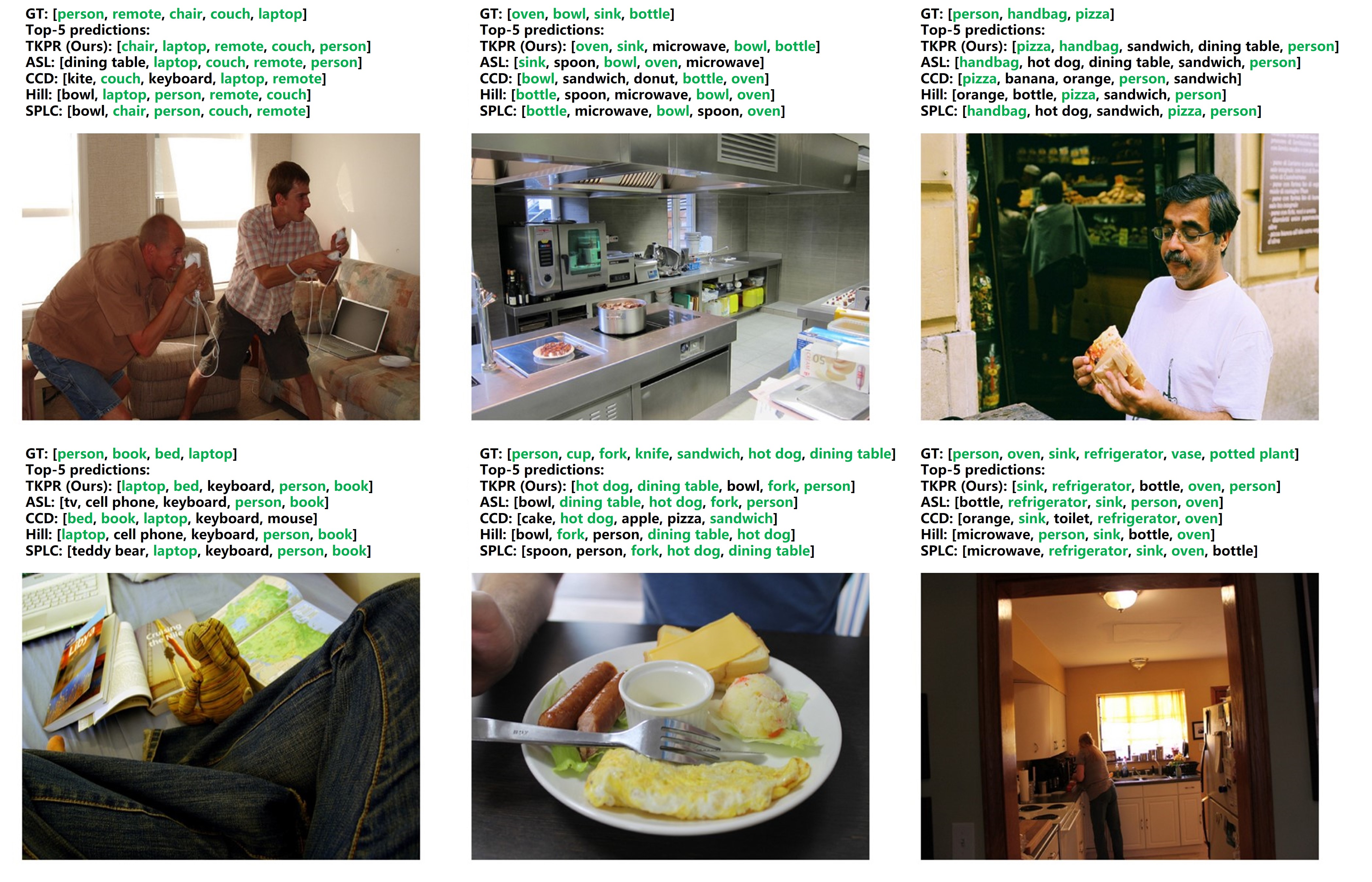}
    \caption{Case study on the COCO dataset with the swin-transformer backbone, where the proposed methods rank the relevant labels higher than the competitors do.}
    \label{fig:case_coco}
\end{figure}

\begin{figure}[htbp]
    \centering
    \includegraphics[width=\linewidth]{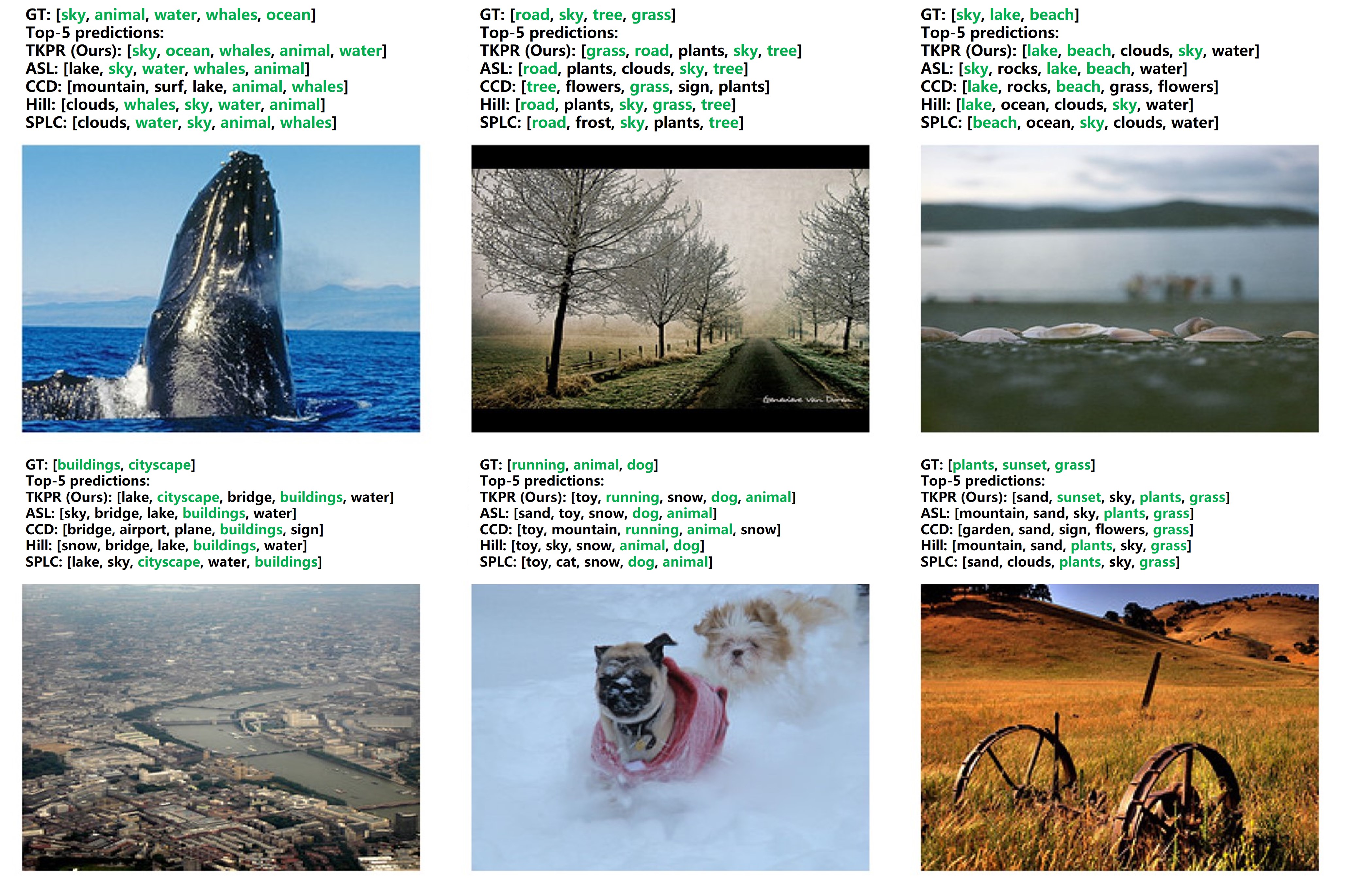}
    \caption{Case study on the NUS-WIDE dataset with the swin-transformer backbone, where the proposed methods rank the relevant labels higher than the competitors do.}
    \label{fig:case_nus}
\end{figure}

\clearpage
\subsection{More MLML experiments on Pascal VOC 2012}
\label{app:pascal12}

  In Tab.\ref{tab:ranking_loss_voc_MLML} and Tab.\ref{tab:sota_voc_MLML}, we present the results on Pascal VOC 2012 in the MLML setting. From these results, we have the following observations:
  \begin{itemize}
    \item Compared with the ranking-based losses, the proposed methods also demonstrate significant improvements on \texttt{mAP@K}, \texttt{NDCG@K}, the ranking loss, and the TKPR measures. Note that for the competitors, the performance degeneration induced by missing labels seems insignificant. This phenomenon is not surprising since the relevant labels are inherently sparse in Pascal VOC 2012.
    \item Compared with state-of-the-art methods, the proposed methods achieve consistent improvements on all the measures. Note that these performance gains are not so impressive than those in MS-COCO, \textit{i.e.}, Tab.\ref{tab:sota_coco_MLML}, which might again comes from the sparse property of the dataset.
    \item The best competitors differ on different measures, which again validates the necessity of the proposed framework.
  \end{itemize}

\begin{table*}[h]
    \centering
    \caption{The empirical results of the ranking-based losses and TKPR on Pascal VOC 2012-MLML, where the backbone is ResNet50. The best and runner-up results on each metric are marked with {\color{Top1}red} and {\color{Top2}blue}, respectively. The best competitor on each measure is marked with \underline{underline}.}
    \renewcommand\arraystretch{1.5}
    \tiny 
    \newcommand{\tabincell}[2]{\begin{tabular}{@{}#1@{}}#2\end{tabular}}
    \begin{tabular}{c|c|cc|cc|cc|cc|cc|cc|cc|c}
      \multicolumn{1}{c|}{\multirow{2}[4]{*}{Type}} & Metrics & \multicolumn{2}{c|}{P@K} & \multicolumn{2}{c|}{R@K} & \multicolumn{2}{c|}{mAP@K} & \multicolumn{2}{c|}{NDCG@K} & \multicolumn{2}{c|}{$\text{TKPR}^{\alpha_1}$} & \multicolumn{2}{c|}{$\text{TKPR}^{\alpha_2}$} & \multicolumn{2}{c|}{$\text{TKPR}^{\alpha_3}$} & \multicolumn{1}{c}{\multirow{2}[4]{*}{\tabincell{c}{Ranking \\ Loss}}} \\
      \cmidrule{2-16}          & K     & 3     & 5     & 3     & 5     & 3     & 5     & 3     & 5     & 3     & 5     & 3     & 5     & 3     & 5     &  \\
      \toprule
      \multicolumn{1}{c|}{\multirow{7}[2]{*}{\tabincell{c}{Ranking \\ loss}}} & $L_{\text{rank}}$ & .438 & .277 & .943 & .976 & .893 & .906 & .917 & .932 & 1.131 & 1.227 & .850 & .897 & .304 & .187 & .0199 \\
      & $L_{u1}$ & .444 & .279 & .955 & .983 & \underline{.926} & \underline{.938} & \underline{.944} & \underline{.957} & \underline{1.166} & \underline{1.253} & \underline{.877} & \underline{.916} & \underline{.313} & \underline{.191} & \underline{.0136} \\
      & $L_{u2}$ & .444 & .279 & .955 & .983 & .926 & .937 & .944 & .956 & 1.166 & 1.253 & .877 & .916 & .313 & .191 & .0137 \\
      & $L_{u3}$ & \underline{.445} & \underline{.280}& \underline{.956} & \underline{.984} & .919 & .930 & .939 & .951 & 1.160 & 1.251 & .871 & .913 & .311 & .191 & .0139 \\
      & $L_{u4}$ & .444 & .279 & .955 & .983 & .926 & .937 & .944 & .956 & 1.166 & 1.253 & .877 & .916 & .313 & .191 & .0137 \\
      & $L_{\text{LSEP}}$ & .444 & .280 & .954 & .983 & .915 & .927 & .936 & .949 & 1.157 & 1.249 & .869 & .911 & .310 & .190 & .0148 \\
      & $L_{\text{TKML}}$ & .443 & .280 & .951 & .983 & .910 & .923 & .931 & .946 & 1.149 & 1.244 & .863 & .908 & .308 & .189 & .0158 \\
      \toprule
      \multicolumn{1}{c|}{\multirow{3}[2]{*}{\tabincell{c}{TKPR \\ (Ours)}}} & $\alpha_1$ & \color{Top1}.450 & \color{Top1}.282 & \color{Top1}.963 & \color{Top1}.987 & \color{Top1}.939 & \color{Top2}.947 & \color{Top1}.954 & \color{Top2}.963 & \color{Top1}1.181 & \color{Top1}1.266 & \color{Top2}.886 & \color{Top1}.924 & \color{Top1}.317 & \color{Top2}.193 & \color{Top1}.0112 \\
      & $\alpha_2$ & \color{Top2}.448 & \color{Top2}.280 & \color{Top2}.961 & \color{Top2}.985 & \color{Top2}.938 & \color{Top1}.948 & \color{Top2}.953 & \color{Top1}.964 & \color{Top2}1.180 & \color{Top2}1.264 & \color{Top1}.887 & \color{Top2}.923 & \color{Top2}.317 & \color{Top1}.193 & \color{Top2}.0116 \\
      & $\alpha_3$ & .447 & .280 & .961 & .985 & .938 & .947 & .953 & .963 & 1.180 & 1.264 & .886 & .923 & .316 & .192 & .0117 \\
    \label{tab:ranking_loss_voc_MLML}%
    \end{tabular}%
\end{table*}%

\begin{table*}[h]
  \centering
    \caption{The empirical results of state-of-the-art MLML methods and TKPR on Pascal VOC 2012-MLML, where the backbone is ResNet50. The best and runner-up results on each metric are marked with {\color{Top1}red} and {\color{Top2}blue}, respectively. The best competitor on each measure is marked with \underline{underline}.}
    \renewcommand\arraystretch{1.5}
    \tiny 
    \newcommand{\tabincell}[2]{\begin{tabular}{@{}#1@{}}#2\end{tabular}}
    \begin{tabular}{c|c|cc|cc|cc|cc|cc|cc|cc|c}
      \multicolumn{1}{c|}{\multirow{2}[4]{*}{Type}} & Metrics & \multicolumn{2}{c|}{P@K} & \multicolumn{2}{c|}{R@K} & \multicolumn{2}{c|}{mAP@K} & \multicolumn{2}{c|}{NDCG@K} & \multicolumn{2}{c|}{$\text{TKPR}^{\alpha_1}$} & \multicolumn{2}{c|}{$\text{TKPR}^{\alpha_2}$} & \multicolumn{2}{c|}{$\text{TKPR}^{\alpha_3}$} & \multicolumn{1}{c}{\multirow{2}[4]{*}{\tabincell{c}{Ranking \\ Loss}}} \\
      \cmidrule{2-16}          & K     & 3     & 5     & 3     & 5     & 3     & 5     & 3     & 5     & 3     & 5     & 3     & 5     & 3     & 5     &  \\
      \toprule
      \multicolumn{1}{c|}{\multirow{7}[2]{*}{\tabincell{c}{Loss \\ Oriented}}} & ROLE\footnotemark[2] & .443 & .276 & .955 & .978 & .935 & .944 & .951 & .961 & 1.175 & 1.254 & .885 & .920 & .316 & .192 & .0163 \\
      & \multicolumn{1}{c|}{EM+APL\footnotemark[2]} & .446 & .277 & .959 & .979 & \underline{.937} & \underline{.944} & \underline{.952} & \underline{.961} & \underline{1.179} & \underline{1.259} & \underline{.886} & .920 & \underline{.316} & .192 & .0157\\
      & \multicolumn{1}{c|}{Hill\footnotemark[2]} & .444 & .279 & .956 & .984 & .924 & .935 & .942 & .955 & 1.165 & 1.253 & .876 & .916 & .313 & .191 & .0137 \\
      & \multicolumn{1}{c|}{SPLC\footnotemark[2]} & \underline{.446} & \underline{.280} & .958 & .984 & .929 & .939 & .947 & .958 & 1.170 & 1.257 & .879 & .918 & .314 & \underline{.192} & .0137 \\ 
      & \multicolumn{1}{c|}{LL-R\footnotemark[2]} & .427 & .266 & .962 & .986 & .924 & .934 & .942 & .954 & 1.132 & 1.207 & .883 & .922 & .311 & .191 & .0131 \\
      & LL-Ct\footnotemark[2] & .427 & .266 & .962 & \underline{\color{Top2}.987} & .925 & .937 & .943 & .956 & 1.134 & 1.209 & .884 & \underline{.924} & .312 & .191 & \underline{.0128} \\
      & LL-Cp\footnotemark[2] & .428 & .265 & \underline{.963} & .986 & .926 & .936 & .944 & .955 & 1.137 & 1.210 & .885 & .924 & .312 & .191 & .0131 \\
      \toprule
      \multicolumn{1}{c|}{\multirow{3}[2]{*}{\tabincell{c}{TKPR \\ (Ours)}}} & $\alpha_1$ & \color{Top1}.451 & \color{Top1}.281 & \color{Top2}.965 & \color{Top1}.987 & \color{Top1}.942 & \color{Top1}.950 & \color{Top1}.957 & \color{Top1}.966 & \color{Top1}1.183 & \color{Top1}1.268 & \color{Top2}.887 & \color{Top1}.925 & \color{Top1}.318 & \color{Top1}.193 & \color{Top1}.0112 \\
      & $\alpha_2$ & \color{Top2}.450 & \color{Top2}.280 & \color{Top1}.966 & .986 & \color{Top2}.941 & \color{Top2}.949 & \color{Top2}.956 & \color{Top2}.965 & \color{Top2}1.183 & 1.267 & \color{Top1}.888 & \color{Top2}.925 & \color{Top2}.317 & \color{Top2}.193 & \color{Top2}.0114 \\
      & $\alpha_3$ & .450 & .280 & .965 & .986 & .941 & .949 & .956 & .965 & 1.183 & \color{Top2}1.268 & .887 & .924 & .317 & .192 & .0114 \\
      \label{tab:sota_voc_MLML}%
    \end{tabular}%
\end{table*}%
\end{appendices}


\clearpage
\bibliography{sn-bibliography}

\end{document}